\documentclass{article}

\usepackage[nonatbib, final]{neurips_2021}
\usepackage[T1]{fontenc}    
\usepackage{url}            
\usepackage{booktabs}
\usepackage{caption}
\usepackage{nicefrac}       
\usepackage{bm}
\usepackage{amsthm,amsmath,amsfonts,amssymb}
\usepackage{mathtools}
\usepackage[noend]{algorithmic}
\usepackage[numbers,sort&compress]{natbib}
\usepackage{enumitem}
\usepackage{algorithm}
\usepackage[table]{xcolor}
\usepackage{graphicx}
\usepackage{subcaption}
\usepackage{multirow}
\usepackage{multicol}
\usepackage{tikz}
\usepackage{hyperref}
\usepackage{url}
\usepackage{mathrsfs, fancyhdr}
\usepackage{xr}
\def\bw{\mathbf{w}}
\def\bz{z}
\def\bwS{\bw^\ast_S}
\def\O{\mathcal{O}}
\def\EX{\mathbb{E}}

\def\Z{\mathcal{Z}}
\def\R{\mathbb{R}}

\def\barw{\bar{\mathbf{w}}}
\def\tw{\widetilde{\mathbf{w}}}
\def\X{\mathcal{X}}
\def\Y{\mathcal{Y}}
\def\Z{\mathcal{Z}}
\def\W{\mathcal{W}}

\def\mbI{\mathbb{I}}
\def\Proj{{\Pi}_\W}
\usepackage{macros}

\theoremstyle{definition}

\def\begeqn{\begin{equation}}
\def\endeqn{\end{equation}}
\def\begth{\begin{theorem}}
\def\endth{\end{theorem}}
\def\begprop{\begin{proposition}}
\def\endprop{\end{proposition}}
\def\begcor{\begin{corollary}}
\def\endcor{\end{corollary}}
\def\begdef{\begin{definition}}
\def\enddef{\end{definition}}
\def\beglemm{\begin{lemma}}
\def\endlemm{\end{lemma}}
\def\begexm{\begin{example}}
\def\endexm{\end{example}}
\def\begrem{\begin{remark}}
\def\endrem{\end{remark}}

\def\begdef{\begin{definition}}
\def\enddef{\end{definition}}

\newcommand{\zhenhuan}[1]{{\color{blue}{\noindent{zhenhuan: \bfseries [}{ \sffamily #1}{\rm\bfseries ]~}}}}

\title{Simple Stochastic and Online Gradient Descent Algorithms for Pairwise Learning}

\author{%
Zhenhuan Yang$^{1*}$ \quad Yunwen Lei$^{2*}$ \quad Puyu Wang$^3 $\quad  Tianbao Yang$^4$\quad Yiming Ying$^1$\\ 
$^1$University at Albany, SUNY, Albany, NY\quad $^2$University of Birmingham, Birmingham\quad\\
$^3$City University of Hong Kong, Hong Kong\quad $^4$University of Iowa City, IA\\
zyang6@albany.edu, y.lei@hbam.ac.uk, puyuwang@cityu.edu.hk,\\
tianbao-yang@uiowa.edu, yying@albany.edu
}

\begin{document}

\maketitle

\begin{abstract}
Pairwise learning refers to learning tasks where  the loss function depends on a pair of  instances. It instantiates many important machine learning tasks such as bipartite ranking and metric learning. A popular approach to handle streaming data in pairwise learning is an online gradient descent (OGD) algorithm, where one needs to pair the current instance with a buffering set of previous instances with a sufficiently large size and therefore suffers from a scalability issue. In this paper, we propose simple stochastic and online gradient descent methods for pairwise learning. A notable difference from the existing studies  
is that we only pair the current instance with the previous one in building a gradient direction, which is efficient in both the storage and computational complexity. We develop novel stability results, optimization, and generalization error bounds for both convex and nonconvex as well as both smooth and nonsmooth problems. We introduce novel techniques to decouple the dependency of models and the previous instance in both the optimization and generalization analysis. Our study resolves an open question on developing meaningful generalization bounds for OGD using a buffering set with a very small fixed size. We also extend our algorithms and stability analysis to develop differentially private SGD algorithms for pairwise learning which significantly improves the existing results.
\let\thefootnote\relax\footnote{$^{*}$ Equal contributions. The corresponding author is Yiming Ying. }
\end{abstract}

\vspace*{-5mm}

\section{Introduction}
Many important learning tasks involve pairwise loss functions which are often referred to as pairwise learning.  Such notable learning tasks include AUC maximization \cite{gao2013one,Kar,liu2018fast,wang2012generalization,ying2016stochastic,zhao2011online}, metric learning \cite{bellet2013survey,Kulis,Weinberger,Xing,YL}, and a minimum error entropy principle \cite{hu2013learning}.  For instance, AUC maximization  aims to rank positive instances above negative ones which involves a loss $f(\bw; (x,y), (x',y')) = ( 1-\bw^\top (x-x') )_+\,\mbI_{[y=1\wedge y'=-1]}$ with $x, x'\in \X\subseteq \mathbb{R}^d$ and $y,y'\in \Y = \{\pm 1\}.$ The aim of metric learning is to learn a distance function $h_\bw(x,x') = (x-x')^\top \bw (x-x')$ where $\bw$ is positive semi-definite matrix in $\mathbb{R}^{d\times d}$. A typical pairwise loss can be  $f(\bw; (x,y), (x',y')) =(1+\tau(y,y')h_\bw(x,x'))_+$
where $\tau(y,y')=1$ if $y=y'$ and $-1$ otherwise.   Given a training data $S = \{z_i = (x_i, y_i) \in \X \times \Y:  i\in  [n]\}$ where $[n] = \{1,2,\ldots, n\}$,   the ERM formulation for pairwise learning is defined as $\min_{\bw\in \W} \frac{1}{n(n-1)}\sum_{i,j\in [n], i\neq j}^n f(\wbf; \bz_i, \bz_j)$ where $\W$ denotes the parameter space.  This scheme has been well studied theoretically using algorithmic stability \cite{agarwal2009generalization} and  U-statistics tools \cite{clemenccon2008ranking}. At the same time,  there are considerable interests on developing and studying online gradient descent (OGD) or stochastic gradient descent (SGD) algorithms for pairwise learning due to their scalability in practice.    

The critical issue for designing such stochastic  algorithms is to construct intersecting pairs of instances for updating the model parameter upon receiving individual instances.  For the offline (finite-sum) setting where the prescribed training data $S = \{\bz_1,\ldots, \bz_n\}$ is available, one natural approach is to, at each time $t$,  randomly select a pair of instances $(\bz_{i_t}, \bz_{j_t})$ from $n(n-1)/2$ pairs and update $\bw$ based on the gradient of the local error $f(\bw_{t-1};\bz_{i_t},\bz_{j_t})$. The excess generalization and stability was nicely established in \cite{lei2020sharper}. One popular approach \cite{Kar,wang2013online,zhao2011online,YZ} is to consider the online setting where the data is continuously arriving. This approach pairs the current datum $\bz_t =(x_t,y_t)$, which is received at time $t$, with all previous instances $S_{t-1} = \{\bz_1,\ldots, \bz_{t-1}\}$ and then performs the update based on the gradient of the local error  $f(\wbf_{t-1}; S_{t-1}) = \frac{1}{t-1}\sum_{z\in S_{t-1}} f(\bw_{t-1}; \bz_t, \bz)$. It requires a high gradient complexity $\O(t)$ (i.e. the number of computing gradients) which is expensive when $t$ becomes large. To mitigate this potential limitation, \cite{Kar,zhao2011online,wang2012generalization} proposed to use a buffering set $B_{t-1} \subseteq S_{t-1}$ of size $s$ and the local error $f(\wbf_{t-1}; B_{t-1}) = \frac{1}{s}\sum_{\bz\in B_{t-1}} f(\bw_{t-1}; \bz_t, \bz)$ which reduces the gradient complexity to $\O(s)$. The excess generalization bound $\O(\frac{1}{\sqrt{s}}+ \frac{1}{\sqrt{n}})$ was established in \cite{Kar} using the online-to-batch conversion method \cite{cesa2004generalization} which is only meaningful for a very large $s$.  In particular, this bound tends to zero only when $s=s(n)$ tends to infinity as $n$ tends to infinity. 
It was mentioned in \cite{Kar} (see the discussion at the end of Section 7 there) as an open question on how to get a meaningful bound for a fixed constant $s$.   

In this paper,  we show that optimal generalization bounds can be achieved for simple SGD and OGD algorithms for pairwise learning where, at time $t$, the current instance $\bz_{t}$ is only paired with the previous instance $\bz_{t-1}$. This is equivalent to the First-In-First-Out (FIFO) buffering strategy \cite{wang2012generalization,Kar} while keeping the size $s$ of the buffering set $B_{t-1}$ to be $s=1$, where, in this FIFO policy, the data $z_t$ arriving at
time $t >1$ is included into the buffer by removing $\{z_1,\ldots, z_{t-2}\}$ from the buffer.  In particular, our main contributions are summarized as follows. 

\setlist[itemize]{leftmargin=5mm}
\begin{itemize}
\item We propose simple SGD and OGD algorithms for pairwise learning where the $t$-th update of the model parameter is based on the interacting of the current instance and the previous one which has a constant gradient complexity $\O(1)$.

\item We establish the stability results of the proposed SGD algorithms for pairwise learning and apply them to derive optimal excess generalization bounds $\O(1/{\sqrt{n}})$ for the proposed simple SGD algorithm for pairwise learning  with both convex and nonconvex as well as both smooth and nonsmooth losses in the offline (finite-sum) setting where the training data of size $n$ is given. We introduce novel techniques to decouple the dependency of the current SGD iterate with the previous instance in both the generalization and optimization error analysis, which resolves the open question in \citep{Kar} on how to develop meaningful generalization bounds when the buffering set of FIFO has a very small size.

\item  We further develop a localization version of our SGD algorithms under $(\epsilon,\delta)$-differential privacy (DP) constraints, and apply the obtained stability results to derive an optimal utility (excess generalization) bound $\tilde{\ocal}\big({1}/{\sqrt{n}} + {\sqrt{d}}/{n\epsilon}\big)$. In contrast to the existing work \cite{huai2020pairwise} which requires the loss function to be smooth and an at least quadratic gradient complexity, our proposed DP algorithms only need a linear gradient complexity $\tilde{\ocal}(n)$ for smooth convex losses to achieve the optimal utility bound and can also be applied to non-smooth convex losses. 
\end{itemize} 

The paper is organized as follows.  Section \ref{sec:related-work} reviews the related work and Section \ref{sec:alg} describes the proposed algorithms. In Section \ref{sec:gen}, we present  excess generalization bounds, stability results, and optimization errors of our algorithms. Section \ref{sec:private} is devoted to the differentially private SGD  for pairwise learning  and its utility bounds. Section \ref{sec:exp} provides experimental validation of theoretical findings. The paper is concluded in Section \ref{sec:conclusion}. All the main proofs are postponed to the Appendix.  
\vspace*{-2mm}
\section{Motivating Examples and Related Work}\label{sec:related-work}  \vspace*{-2mm}
In this section, we list examples of pairwise learning and discuss some related work. 
\subsection{Motivating Examples} 


\textbf{AUC maximization.} Area under the ROC curve (AUC) of a prediction function $h_\wbf$ is the probability that the function ranks a random positive example higher than a random negative example.  The empirical risk of AUC maximization is given by
$
F_S(\wbf) = \frac{1}{n(n-1)} \sum_{i,j\in [n], i\neq j} \ell(h_\wbf(\xbf_i) - h_\wbf(\xbf_j)) \ibb_{[y_i=1]}   \ibb_{[y_j=-1]},$
where the loss $\ell(\cdot)$ can be the least square loss $\ell(t) = (1-t)^2 $ or the hinge loss $\ell(t) = (1-t)_+.$

\textbf{Minimum error entropy principle.} Minimum error entropy (MEE) is a principle of information theoretical learning~\citep{hu2013learning,hu2015regularization}, which aims to find a predictor $h:\xcal\mapsto\ycal$  by minimizing the information entropy of the variable $E=Y-h(X)$. The R\'enyi's entropy of order $2$ for $E$ is defined as
$H(E)=-\log\int p_E^2(e)de$. Here $p_E$ is probability density function and can be approximated by Parzen windowing $\hat{p}_E(e)=\frac{1}{n\gamma}\sum_{i=1}^{n}G\big(\frac{(e-e_i)^2}{2\gamma^2}\big)$, where $e_i=y_i-h(x_i)$, and $\gamma>0$ is an MEE scaling parameter and $G:\rbb\mapsto\rbb^+$ is a windowing function. The approximation of R\'enyi's entropy is given by its empirical version
$
\hat{H}=-\log\frac{1}{n^2\gamma}\sum_{i,j\in[n]}G\big(\frac{(e_i-e_j)^2}{2\gamma^2}\big).
$
The maximization of $\hat{H}$ then leads to a pairwise learning problem since each loss function involves a pair of training examples~\citep{hu2013learning}.

\textbf{Metric learning.} Metric learning aims to fine a distance metric $h_{\bw}:\xcal\times\xcal\mapsto\rbb_+$ consistent with some supervised information, e.g., examples within the same class are close while examples from different classes are apart from each other under the learnt metric. If $\ycal=\{\pm1\}$, the performance of $h$ on an example pair $z,z'$ can be quantified by a loss function of the form $f(\bw;z,z')=\ell(yy'(1-h_{\bw};x,x'))$, where $\ell:\rbb\mapsto\rbb_+$ is a decreasing function for which some typical choices include the hinge loss $\ell(t) = (1-t)_+$ and the exponential function $\ell(t)=\log(1+\exp(-t))$. Then one can minimize the training error $F_S(\bw)=\frac{1}{n(n-1)}\sum_{i,j\in[n]:i\neq j}f(\bw;z_i,z_j)$ to learn a distance metric. Another closely related learning task to metric learning  is constrastive learning \cite{sohn2016improved, wu2018unsupervised, oord2018representation, chen2020simple} which has become very popular recently for learning visual representations without supervision.

\vspace*{-2mm}
\subsection{Related Work}  
The work \cite{Kar,wang2012generalization,wang2013online,zhao2011online} assumed the online learning setting in which a stream of i.i.d. data $\{\bz_1,\bz_2, \ldots,\bz_t, \ldots\}$ is continuously arriving. Upon receiving $\bz_t$ at time $t$, it is paired with all previous instances and then the model parameter is updated based on the local error $F_t(\wbf_{t-1}) = \frac{1}{t-1}\sum_{j=1}^{t-1} f(\wbf_{t-1}; \zbf_t, \zbf_j).$ In particular,  the work \cite{wang2012generalization,wang2013online} provided the first excess generalization bound for online learning methods by obtaining online-to-batch conversion bounds \citep{cesa2004generalization} using covering
numbers of function classes.   \citet{Kar} significantly improved the results   using the so-called symmetrization
of expectations which reduce excess risk estimates to Rademacher complexities.  To  further reduce the expensive gradient complexity $\O(t)$ at a large time $t$, the work \cite{Kar,wang2013online,zhao2011online} proposed to use a buffering set $B_{t-1}$ with size $s$ instead of all previous instances. It was shown that such OGD algorithms have an excess generalization bound $\O(\frac{1}{\sqrt{s}} + \frac{1}{\sqrt{n}})$ for convex and Lipschitz-continuous  losses.  
In the offline learning setting where the training data $S = \{\bz_1,\bz_2, \ldots,\bz_n\}$ of size $n$ is fixed, the work \cite{yang2021stability} considered the stochastic version of the algorithm in \cite{wang2012generalization,Kar} where, at time $t$, a random instance $\bz_{i_t}$ with random index $i_t\in \{1,\ldots,n\}$ is paired with all previous instances $\{\bz_{i_1},\ldots, \bz_{i_{t-1}}\}.$ They derived stability and generalization results of such algorithms in expectation.

\citet{lei2020sharper} considered the offline learning setting  and, at time $t$, the algorithm there randomly picks a pair of instances $(\zbf_{i_t}, \zbf_{j_t})$  from all $\binom{n}{2}$ pairs of instances. An excess risk bound $\tilde{\ocal}\big(1/{\sqrt{n}}\big)$ with high probability was derived for convex, Lipschitz and strongly smooth losses. Here the notation $\tilde{\ocal}(\cdot)$ means $\ocal(\cdot)$ up to some logarithmic terms.  In the particular case of AUC maximization with the least square loss,   \cite{ying2016stochastic} considered the online learning setting and reformulated the problem as a stochastic saddle point (min-max) problem which decouples the pairwise structure. From this reformulation, efficient SGD-type algorithms \cite{YL,liu2019stochastic} have been developed.

Recently, differentially private pairwise learning  has been studied where the construction of pairs of instances follows \cite{Kar,wang2012generalization,zhao2011online} or  all pairs are used at each iteration.  In particular,  \citet{huai2020pairwise}   considered both online and offline learning settings  and the pairs of instances at time $t$ follow \cite{Kar,wang2012generalization,zhao2011online}. They provided a utility bound   $\widetilde{\ocal}(\sqrt{d}/(\sqrt{n}\epsilon))$ for convex and smooth loss functions.  The study \cite{yang2021stability} also paired the current instance with all previous ones and showed that SGD with output perturbation for pairwise learning has a utility bound  $\widetilde{\ocal}(\sqrt{d}/(\sqrt{n}\epsilon))$ for nonsmooth convex losses.  The work  \citep{xue2021differentially} showed that private gradient descent  using all possible pairs can achieve a utility bound   $\widetilde{\ocal}(1/\sqrt{n} + \sqrt{d}/(n\epsilon))$ for strongly smooth and convex losses.

\section{Proposed Algorithms for Pairwise Learning}\label{sec:alg}
In this section,  we describe the proposed algorithms in two common learning settings for pairwise learning: offline and online learning settings. Let $\rho$ be a probability measure defined on $\Z:=\X\times\Y$, where $\X$ is an input space and $\Y$ is an output space.   In pairwise learning, the performance of ${\bw}$ is measured on a pair of instances $(\bz,\bz')$ by a nonnegative loss function $f(\bw;\bz,\bz')$. 
Denote by $[n]: = \{1,2,\ldots, n\}$ for any $n\in\mathbb{N}$.

\vspace*{-2mm}
\begin{minipage}[t]{0.5\textwidth}
\begin{algorithm}[H]
\begin{algorithmic}[1]
\caption{SGD for Pairwise Learning}\label{alg:markov}
\STATE {\bf Inputs:} $S\!= \!\{\bz_i\!:  i\in [n]\}$ and step sizes $\{\eta_t\}$\!
\STATE {\bf Initialize:}  $\wbf_0\in\W$, let $\bw_{-1} =\bw_0$ and randomly select $i_{0}\in [n]$
\FOR{$t=1,2,\ldots,T$}
\STATE \!\!\!\! {Randomly select ${i_{t}} \in [n]$ }
\STATE \!\!\!{$\!\wbf_{t} \!=\! \Proj\bigl(\wbf_{t-1} \!-\! {\eta_t}\nabla f(\wbf_{t-1}; \bz_{i_{t}}, \bz_{i_{t-1}})\bigr)$} \\ \ENDFOR
\STATE {\bf Outputs:} $\bar{\wbf}_T = {\sum_{j=1}^T\eta_j \bw_{j-2}}/\sum_{j=1}^T \eta_j$
\end{algorithmic}
\end{algorithm}
\end{minipage}
\hfill
\begin{minipage}[t]{0.5\textwidth}
\begin{algorithm}[H]
\begin{algorithmic}[1]
\caption{OGD for Pairwise Learning}\label{alg:OGD}
\STATE {\bf Inputs:} learning rates $\{\eta_t\}$ \\
\STATE {\bf Initialize:}  $\wbf_0\in\W$, let $\bw_{-1} =\bw_0$ and receiving datum $z_0$.  \\
\FOR{$t=1,2, \ldots, T$}
\STATE {Receive a data point $z_t$}
\STATE {$\wbf_{t} = \Proj\bigl(\wbf_{t-1} - {\eta_t}\nabla f(\wbf_{t-1}; \bz_{t}, \bz_{t-1})\bigr)$} \\  
\ENDFOR
\STATE {\bf Outputs:} $\tw_T = {\sum_{j=1}^T\eta_j \bw_{j-2}}/\sum_{j=1}^T \eta_j$
\end{algorithmic}
\end{algorithm}
\end{minipage}

\noindent{\bf Offline Learning (Finite-Sum) Setting.} The first is the  finite-sum setting where the training data $S=\{\bz_i = (x_i,y_i)\in \Z:  i \in [n]\}$ are drawn independently according to $\rho$. 
In this context, one aims to  solve the following empirical risk minimization (ERM):
\begin{equation}\label{eq:erm}
\bwS = \argmin_{\bw\in \W} \Big[F_S(\wbf) := \frac{1}{n(n-1)}\sum_{i,j\in [n], i\neq j} f(\wbf; \bz_i, \bz_j)\Big].\end{equation}
Our proposed algorithm to solve \eqref{eq:erm} is described in Algorithm \ref{alg:markov}. The notation $\Proj(\cdot)$ there denotes the projection operator to $\W.$  In particular,  at  iteration $t$,  it randomly selects one instance $\bz_{i_t}$ from the uniform distribution over $[n]$ and pairs it only with the previous instance $\bz_{i_{t-1}}$, and then do the gradient descent based on $\nabla f(\wbf_{t-1}; \bz_{i_{t}}, \bz_{i_{t-1}}).$  This is in contrast to the classical SGD for pairwise learning in \cite{wang2012generalization,Kar,zhao2011online} where the present instance $\bz_{i_t}$ is paired with all previous instances $\{\bz_{i_1}, \bz_{i_2}, \ldots, \bz_{i_{t-1}}\}$. Note $\bw_{t-1}$ depends on $\bz_{i_{t-1}}$ and then $\nabla f(\bw_{t-1}; \bz_{i_{t}}, \bz_{i_{t-1}})$ is not an unbiased estimate of $\nabla F_S(\bw_{t-1})$. Therefore, the standard analysis of SGD does not apply. We introduce novel techniques to handle this dependency in our analysis (see more details in Section \ref{sec:gen}).

\noindent {\bf Online Learning Setting.} In the online learning setting where the data $\{\bz_0, \bz_1, \bz_2, \ldots\}$ is assumed i.i.d. from an unknown distribution $\rho$ on $\Z$, the number of iterations of an online algorithm is identical to the size of available data. In the same spirit to Algorithm \ref{alg:markov}, the pseudo code   is given in Algorithm \ref{alg:OGD}. Specifically, upon  receiving a datum $z_t$ at the current time $t$, we pair it with $\bz_{t-1}$ which was revealed at the previous time $t-1$ and then perform gradient descent based on the gradient $\nabla f(\bw_{t-1}; \bz_t, \bz_{t-1}).$ It aims to minimize the population risk which is defined as $F(\bw)=\EX_{Z,Z'}[f(\bw;Z,Z')].$ Here $\EX_{Z,Z'}$ denotes the expectation with respect to (w.r.t.) $Z,Z'\sim\rho.$

It is worth pointing out that online learning \citep{shalev2012online,hazan2016introduction,orabona2019modern} in general 
does not require the i.i.d.  assumption on the data  and study the regret bounds. In this paper, we mainly consider the statistical performance, measured by excess generalization bounds, of the output $\tw_T$ of Algorithm \ref{alg:OGD} where the streaming data $\{\bz_0, \bz_1, \bz_2, \ldots\}$ is i.i.d. from the population distribution $\rho.$

\begin{remark} As discussed above,  OGD for pairwise learning was proposed and studied in \cite{wang2012generalization,Kar} where the current instance $\bz_t$ is paired with a buffering set $B_t \subseteq \{z_1,\ldots, z_{t-1}\}.$ However, the resultant excess generalization bound is in the form of  $\O(\frac{1}{\sqrt{s}}+ \frac{1}{\sqrt{n}})$ which indicated that the buffer size $s$ needs to be large enough in order to achieve good generalization. Their analysis does not apply to our case since the buffering set $B_t = \{\bz_{t-1}\}$ with size $s =1$ for Algorithm \ref{alg:OGD}.   As we show soon in Section \ref{sec:gen}, we can prove that Algorithm \ref{alg:OGD},  which pairs the current instance $\bz_t$ with the previous instance $\bz_{t-1}$,   still enjoys optimal statistical performance $\O(1/{\sqrt{n}}).$
\end{remark}
\begin{remark}
The work \cite{lei2020sharper} studied the stability and generalization of an SGD-type algorithm for pairwise learning by randomly generating pairs of instances. Specifically, at time $t$, randomly generating a pair $(\bz_{i_t}, \bz_{j_t})$ from  a given set of training data $S=\{\bz_i:  i\in [n]\}$ and the subsequent update is given by $\bw_{t} = \bw_{t-1} - \eta_t \nabla f(\bw_{t-1}; \bz_{i_t}, \bz_{j_t}).$  In contrast, our algorithm, i.e. Algorithm \ref{alg:markov}, updates the model parameter based on the pair of the current random instance and the random one generated at the previous time $t-1$, i.e. $ (\bz_{i_t}, \bz_{i_{t-1}}).$  Furthermore,  our work here significantly differs from \cite{lei2020sharper} in the following aspects. Firstly,  the algorithm there by randomly selecting pairs of instances does not work in the online learning setting while ours can seamlessly deal with the streaming data as stated in Algorithm \ref{alg:OGD}. Secondly, regarding the technical analysis, our algorithms are more challenging to analyze than the algorithm in \cite{lei2020sharper}. Indeed, $\nabla f(\bw_{t-1}; \bz_{i_t}, \bz_{j_t})$ is not an unbiased estimate of $\nabla F_S(\bw_{t-1})$ due to the independency between $\bw_{t-1}$ and $(i_t,j_t)$. Therefore, the optimization error analysis of the algorithm in \cite{lei2020sharper} is the same as the SGD for pointwise learning. As a comparison,  $\nabla f(\bw_{t-1}; \bz_{i_t}, \bz_{i_{t-1}})$ is a biased estimate of $\nabla F_S(\bw_{t-1})$ due to the coupling between $\bw_{t-1}$ and $i_{t-1}$. We introduce novel techniques to handle this coupling for both the optimization and generalization analyses. Thirdly, we will soon see below that we provide generalization results for nonsmooth, nonconvex losses and also use Algorithm \ref{alg:markov} to develop novel differentially private pairwise learning algorithms while \cite{lei2020sharper} focused on  the smooth convex losses in the non-private setting.     
\end{remark}

\begin{remark} If we let $\xi_t = (i_t, i_{t-1})$ in  Algorithm \ref{alg:markov}, then $\{\xi_t: t\in \mathbb{N}\}$ forms a Markov Chain as $\xi_t$ only depends on $\xi_{t-1}$ but not on $\{\xi_1, \ldots, \xi_{t-2}\}.$  Hence, Algorithm \ref{alg:markov} can be regarded as a Markov Chain SGD which was studied in \cite{sun2018markov}.  Despite this similarity, our results differ from \cite{sun2018markov} in two important aspects. Firstly,   we are mainly interested in  stability and  generalization of Algorithm \ref{alg:markov} while  \cite{sun2018markov} focused on the convergence analysis of the Markov Chain SGD.  One cannot apply the results in \cite{sun2018markov} to obtain excess generalization bounds for Algorithm \ref{alg:markov} in terms of the population risk   as we will show soon  in the next section.  Secondly, directly applying  Theorem 1 in \cite{sun2018markov}  only yields a convergence rate of $\O(1/t^{1-q})$ with some $1/2< q< 1$ in the convex setting.  Our proof for the convergence analysis of Algorithm \ref{alg:markov} is much simpler and direct which can yield a faster convergence rate $\O(1/\sqrt{t})$ as shown in Section \ref{sec:optimization}.    
\end{remark}

\section{Generalization Analysis}\label{sec:gen}
The aim for the generalization analysis of Algorithm \ref{alg:markov} and Algorithm \ref{alg:OGD} is the same, i.e.  to analyze the excess generalization error $F(\bw)-F(\bw^*)$ of a model $\bw$ measuring its relative behavior w.r.t. the best model $\bw^*=\arg\min_{\bw\in\wcal}F(\bw)$.

For  Algorithm \ref{alg:markov} where the training data $S$ with $n$ datum is given beforehand, the excess generalization involves the generalization error and optimization error. Specifically,  one has the following error decomposition for~$\barw_T$
\begin{equation}\label{err-dec}
  \ebb\big[F(\barw_T)\big]-F(\bw^*)=\ebb\big[F(\barw_T)-F_S(\barw_T)\big]+\ebb\big[F_S(\barw_T)-F_S(\bw^*)\big].
\end{equation}
Here, the expectation is taken w.r.t. the randomness of Algorithm \ref{alg:markov}, i.e., $\{i_t\}$ and the randomness of data $S$ which is i.i.d. from $\rho$ on $\Z.$ We refer to the first term $\ebb\big[F(\barw_T)-F_S(\barw_T)\big]$ as the generalization error and $\ebb\big[F_S(\barw_T)-F_S(\bw^*)\big]$ as the optimization error. We will use algorithmic stability to handle its generalization error in Subsection \ref{sec:stability} for smooth and nonsmooth losses. The estimation of the optimization error is given in Subsection \ref{sec:optimization} for both convex and nonconvex losses.  For Algorithm \ref{alg:OGD}, there is no generalization error as the data $\{\bz_1, \bz_2, \ldots, \bz_T\}$ is arriving in a sequential manner with $T$ increasing all the time which does not involve the training data. The randomness of Algorithm \ref{alg:OGD} is only from the i.i.d. data. Therefore, the optimization error in this setting is exactly the excess generalization error $F(\widetilde{\bw}_T)-F(\bw^*)$  which is estimated in Subsection \ref{sec:optimization}. 

\subsection{Excess Generalization Error}
In this subsection, we present excess generalization error bounds of Algorithm \ref{alg:markov} in terms of the sample size, iteration number and step size, which shows how to tune these parameters to get a model with good generalization.  
Our analysis requires the following assumptions. 
\begin{assumption} Let $f: \W \times \Z\times \Z \to \R^+$ and let $\|\cdot\|_2$ denote the Euclidean norm.
\vspace*{-1mm}
\begin{enumerate}[label=(\textbf{A\arabic{*}}), leftmargin=*]
\setlength\itemsep{-0mm}
 \item  Assume,  for any $\bz, \bz'$ and $\bw\in \W$, that $ f(\cdot; \bz,\bz')$ is $G$-Lipschitz continuous, i.e.  $|f(\bw;\bz,\bz') - f(\bw';\bz,\bz')| \le G \|\bw-\bw'\|_2.$
   \item  Assume, for any $\bz, \bz'\in \Z$, the map $\bw \mapsto f(\bw; \bz,\bz')$ is $L$-strongly smooth, i.e. $ f(\bw; \bz,\bz') -  f(\bw'; \bz,\bz') - \langle\nabla f(\bw'; \bz,\bz'), \bw-\bw'\rangle \le \frac{L} {2} \|\bw-\bw'\|_2^2.$
  \item  Assume, for any $z, z'\in \Z$, $ f(\cdot; \bz,\bz')$ is $\alpha$-strongly convex, i.e.  $ f(\bw; \bz,\bz') -  f(\bw'; \bz,\bz') - \langle\nabla f(\bw'; \bz,\bz'), \bw-\bw'\rangle \ge \frac{\alpha}{2} \|\bw-\bw'\|_2^2.$ The case of $\alpha=0$ is identical to convexity.
  \item   Assume $F_S$ satisfies the \PL\ (PL) condition with parameter $\mu\!>\!0$, i.e., for $\wbf_S \!\in\!\arg\min_{\wbf \in \wcal} F_S(\wbf)$, there holds
  $
    2\mu\big(F_S(\bw)\!-\!F_S(\bw_S)\big)\!\leq\! \|\nabla F_S(\bw)\|_2^2
  $
  for all $\bw\in\wcal$.
 \end{enumerate}
\end{assumption}
The PL condition ({\bf A4}) means that the suboptimality in terms of function values can be bounded by gradients~\citep{karimi2016linear}. Functions under the PL condition have found various applications including neural networks, matrix factorization, generalized linear models and robust regression~ (see, e.g., \citep{karimi2016linear}). In particular, AUC maximization problem with the classifier given by a one hidden layer network satisfies the PL condition as shown in \cite{liu2019stochastic}.

We first study smooth and non-smooth problems for the convex case, and derive the excess generalization bounds of the order $\O(1/\sqrt{n})$ in both cases. We use the notation $B\asymp \tilde{B} $ if there exist constants $c_1, c_2>0$ such that $c_1\tilde{B} \le B\le c_2\tilde{B} $.   The proofs are given in Section  \ref{sec:proof-gen}.
\begin{theorem}[Nonsmooth Problems\label{thm:excess-nonsmooth}]
Let $\bw_{-1} = \bw_0$ and $\{\bw_t: t\in [T]\}$ be produced by Algorithm \ref{alg:markov} with $\eta_t = \eta > 0$. Let $\barw_T =   \sum_{t=1}^{T} \eta_t\bw_{t-2} / \sum_{t=1}^{T}\eta_t$. Let (A1) and (A3) hold true with $\alpha = 0$.  Then, we have 
\begin{equation}\label{excess-nonsmooth}
\ebb_{S,\acal}[F(\bar{\wbf}_T)]-F(\bw^*)=\O\Big(\sqrt{T}\eta+\frac{T\eta}{n}+\frac{1+T\eta^2}{T\eta}\Big).
\end{equation}
Furthermore, selecting $T\asymp n^2$ and $\eta\asymp T^{-\frac{3}{4}}$ yields that  $\ebb_{S,\acal}[F(\bar{\wbf}_T)]-F(\bw^*)=\O(1/\sqrt{n})$.
\end{theorem}

\begin{theorem}[Smooth Problems\label{thm:excess-smooth}]
Let (A1), (A2) and (A3) hold true with $\alpha = 0$.
Let $\bw_{-1} = \bw_0$ and $\{\bw_t: t\in [T]\}$ be produced by Algorithm \ref{alg:markov} with $\eta_t = \eta\leq 2/L$. Let $\barw_T =   \sum_{t=1}^{T} \eta_t\bw_{t-2} / \sum_{t=1}^{T}\eta_t$.   Then, there holds
\begin{equation}\label{excess-smooth}
\ebb_{S,\acal}[F(\bar{\wbf}_T)]-F(\bw^*)=\O\Big(\frac{T\eta}{n}+\frac{1+T\eta^2}{T\eta}\Big).
\end{equation}
Furthermore, choosing $T\asymp n$ and $\eta\asymp T^{-\frac{1}{2}}$ implies that  $\ebb_{S,\acal}[F(\bar{\wbf}_T)]-F(\bw^*)=\O(1/\sqrt{n})$.
\end{theorem}

\begin{remark} Notice that the gradient complexity (i.e. the number of computing gradients) of Algorithm \ref{alg:markov} is identical to the number of iterations $T.$ 
The above results show, to get excess generalization bounds $\O(1/\sqrt{n})$, that  Algorithm $\ref{alg:markov}$ requires a gradient complexity $\O(n^2)$  for nonsmooth problems, and $\O(n)$  for smooth problems. This matches the existing generalization analysis for pointwise learning~\citep{lei2020fine,bassily2020stability,hardt2016train}. In Appendix \ref{sec:iterative-localization}, additional results are provided where we propose Algorithm \ref{alg:iterative-localization} based on the iterative localization technique \citep{feldman2020private} in order to reduce the gradient complexity $\O(n^2)$ required in Theorem \ref{thm:excess-nonsmooth} to  $\O(n)$ for nonsmooth problems.  
\end{remark}

\begin{remark}
As stated in the introduction, Algorithm \ref{alg:markov} can be considered as a specific case of the classic pairwise learning algorithm~\citep{Kar} with a FIFO buffering set $B_{t-1}$ of size $s=1$. A key difficulty in the generalization analysis is that $\bw_{t-1}$ depends on $B_{t-1}$, which renders the standard martingale analysis not applicable. \citet{Kar} proposed to remove this coupling effect by considering $\sup_{\bw}\big[f(\bw;B_{t-1})-F_S(\bw)\big]$, which is why they only derived the excess generalization error bound $\O(1/\sqrt{s})$. We introduce novel techniques to handle the coupling in both generalization analysis and optimization error analysis. For the generalization analysis, our strategy is to write the stability as a deterministic function of several indicator functions on whether we select the different point in neighboring datasets, and then finally consider the randomness of these indicator functions. This delay of considering expectation successfully decouples the coupling between $\bw_{t-1}$ and $B_{t-1}$. For the optimization error analysis, our novelty is to observe that $f(\bw_{t-1};\bz_{i_t},\bz_{i_{t-1}})=f(\bw_{t-2};\bz_{i_t},\bz_{i_{t-1}})+\O(\eta_{t-1})$, which removes the decoupling since $\bw_{t-2}$ is now independent of both $i_t$ and $i_{t-1}$. Since the additional term $\O(\eta_{t-1})$ here is a term of smaller magnitude, the coupling effect is removed without incurring any additional cost.
\end{remark}

Finally, we study nonconvex pairwise learning under the PL condition. The proof is in Section \ref{sec:proof-gen}.
\begin{theorem}\label{thm:exg-pl}
 Let {(A1), (A2) and (A4)} hold true. Let $\alpha_0,B>0$.
Let $\bw_{-1} \!=\! \bw_0$ and $\{\bw_j\!:\! j\in [t]\}$ be produced by Algorithm \ref{alg:markov} with {$\eta_j=2/(\mu(j+1))$}. If $\ebb_{i_{j+1},i_{j+2}}\big[\|\nabla f(\bw_j;\bz_{i_{j+1}},\bz_{i_{j+2}})\|_2^2\big]\leq\alpha_0^2$ and $\sup_{\bz,\bz'}f(\bw_j;\bz,\bz')\leq B$ for any $j$, then
$$\ebb\big[F(\bw_T)]-F(\bw^*)=\O\Big(\frac{T^{\frac{2L}{2L+\mu}}}{n}\Big)+\O\big(1/(T\mu^2)\big).$$
Furthermore, choosing $T\asymp n^{\frac{2L+\mu}{4L+\mu}}\mu^{-\frac{4L+2\mu}{4L+\mu}}$ yields that $\ebb\big[F(\bw_T)]-F(\bw^*)=n^{-\frac{2L+\mu}{4L+\mu}}\mu^{-\frac{4L}{4L+\mu}}$.
\end{theorem}

\begin{remark}As pointed out before, there is no generalization error for the OGD algorithm, i.e. Algorithm \ref{alg:OGD} as the i.i.d. data is given in a streaming manner and the iteration number equals the number of the available data (i.e. $t=n$). In this setting, the optimization error is identical to the excess generalization error, which will be estimated in Subsection \ref{sec:optimization}. 
\end{remark}

\subsection{Stability and Generalization Errors}\label{sec:stability}
We study generalization errors by algorithmic stability, which measures the sensitivity of the output of an algorithm w.r.t. the perturbation of the dataset. 
Below we give the definition of uniform argument stability.
We say $S,S'$ are neighboring datasets if they differ at most by a single example. 

\begin{definition}\label{def:uniform-stability}
A (randomized) algorithm $\acal$ for pairwise learning is called $\varepsilon$-uniformly argument stable if for all neighboring datasets $S,S'\in \zcal^n$ we have
$
\ebb_\acal [\|\acal(S) - \acal(S')\|_2] \leq \varepsilon.
$
\end{definition}

It is clear $\varepsilon$-uniform argument stability implies $G\varepsilon$-uniform stability~i.e., $\sup_{\zbf,\zbf'} \ebb_\acal [f(\acal(S), \zbf, \zbf') \!-\! f(\acal(S'), \zbf, \zbf')] \leq G\varepsilon$ for Lipschitz losses~\citep{bousquet2002stability}. The connection between the uniform stability for pairwise learning and its generalization has been established in the literature \citep{agarwal2009generalization, shen2020stability}.
\begin{lemma}\label{lem:generalization-via-stability}
If an algorithm $\acal$ for pairwise learning is $\varepsilon$-uniformly stable for some $\varepsilon > 0$, then we have
$|\ebb_{S,\acal}[F_S(\acal(S)) - F(\acal(S))]| \leq 2\varepsilon.$
\end{lemma}

We develop uniform argument stability bound of Algorithm \ref{alg:markov} and apply it together with Lemma \ref{lem:generalization-via-stability} to establish the following generalization bounds. Theorem \ref{lem:est-nonsmooth} handles nonsmooth problems, while Theorem \ref{lem:est-smooth} handles smooth problems. The detailed proofs for them can be found in Section \ref{sec:proof-est}.  
\begin{theorem}\label{lem:est-nonsmooth}
Let $\bw_{-1} = \bw_0$ and $\{\bw_j: j\in [t]\}$ be produced by Algorithm \ref{alg:markov} with $\eta_j = \eta$.  Let (A1) and (A3) hold with $\alpha = 0$.  Then, Algorithm \ref{alg:markov} is $2\sqrt{e}G\eta\big(\sqrt{5t}+\frac{2t}{n}\big)$-uniformly argument stable and
$$
\ebb_{S,\acal}[F(\bar{\wbf}_t) - F_S(\bar{\wbf}_t)]\leq 4\sqrt{e}G^2\eta\Big(\sqrt{5t}+\frac{2t}{n}\Big).$$
\end{theorem}
\begin{theorem}\label{lem:est-smooth}
Let $\bw_{-1} = \bw_0$ and $\{\bw_j: j\in [t]\}$ be produced by Algorithm \ref{alg:markov} with $\eta_j = \eta\leq 2/L$.  Let (A1), (A2) and (A3) hold true with $\alpha = 0$.  Then, Algorithm \ref{alg:markov} is $\frac{4G}{n}\sum_{j=1}^{t}\eta_j$-uniformly argument stable and the generalization error satisfies
$\ebb_{S,\acal}[F(\bar{\wbf}_t) - F_S(\bar{\wbf}_t)]\leq \frac{8G^2}{n}\sum_{j=1}^{t}\eta_j.$
\end{theorem}

For Algorithm \ref{alg:OGD}, there is no generalization error as the data $\{\bz_1, \bz_2, \ldots, \bz_T \}$ is assumed to arrive in a sequential manner with $T$ increasing all the time which does not involve the training data.

\subsection{Optimization Error}\label{sec:optimization}
\newcommand{\bwa}{\mathbf{w}}
In this subsection, we establish the convergence rate, i.e., optimization error, of Algorithm \ref{alg:markov} for convex, nonconvex and strongly convex problems. We consider both bounds in expectation and with high probability. Our analysis is based on the key observation $f(\bw_{t-1};\bz_{i_t},\bz_{i_{t-1}})=f(\bw_{t-2};\bz_{i_t},\bz_{i_{t-1}})+\O(\eta_{t-1})$. The proofs of results in this subsection are given in Section~\ref{sec:proof-opt}. 

Below we only present optimization error bounds for Algorithm \ref{alg:markov} here in the offline (finite-sum) setting where the training data of size $n$, denoted by $S= \{z_1,\ldots, z_n\}$, is fixed, and the optimization error is measured by $F_S(\barw_t) - \inf_{\bw\in \W}F_S(\bw).$ We emphasize that all our optimization error bounds hold true  for Algorithm \ref{alg:OGD} in the online learning setting with exactly the same analysis where the streaming data $\{z_1,\ldots, z_t, ...\}$ is assumed to be i.i.d according to the   population distribution $\rho$, and the bounds for the optimization error in this case is given for the excess generalization error (excess population risk), i.e. $F(\widetilde{\bw}_t) - \inf_{\bw\in \W}F(\bw).$ 
\begin{theorem}\label{thm:convergence}
Let $\bw_{-1} = \bw_0$ and $\{\bw_j: j\in [t]\}$ be produced by Algorithm \ref{alg:markov}.  Let (A1) and (A3) hold true with $\alpha\ge 0$.   Then, for any $\bwa$ independent of $\acal$ we have the following convergence rates:   
\begin{enumerate}[label=(\alph*), leftmargin=*]\setlength\itemsep{-1mm}
    \item Assume $f$ is convex, i.e. $\alpha= 0.$ Then, we have 
\begin{equation}\label{eq:parta}
   \ebb_{\acal}[F_S(\barw_t)] - F_S(\bwa)  \le \frac{ \|\bw_{0} -  \bwa\|_2^2  + G^2 \sum_{j=1}^t(2\eta_j\eta_{j-1}+ \eta_j^2)}{ 2\sum_{j=1}^t \eta_j}.
\end{equation}
\item Let $f$ be $\alpha$-strongly convex with $\alpha>0$ and $\eta_j = \frac{2}{\alpha(j+1)}$. Then, there holds
$\ebb_{\acal}[F_S(\barw_t)] - F_S(\bwa) = \mathcal{O}\bigl({G^2 / (\alpha t)}\bigr).$
\end{enumerate}
\end{theorem}
\begin{remark}
  The above convergence rates match those in the pointwise learning~\citep{bottou2018optimization}. Furthermore, if $\eta_j=\eta$, then Eq. \eqref{eq:parta} becomes $\ebb_{\acal}[F_S(\barw_t)] - F_S(\bwa)=\O(1/(t\eta)+\eta)$ and one can choose $\eta\asymp 1/\sqrt{t}$ to get  $\ebb_{\acal}[F_S(\barw_t)] - F_S(\bwa)=\O(1/\sqrt{t})$.
  We can extend our convergence analysis to a more general update as $\bw_t=\Pi_{\wcal}\big(\bw_{t-1}-\frac{\eta_t}{s}\sum_{j=1}^s\nabla f(\bw_{t-1};\bz_{i_t},\bz_{i_{t-j}})\big)$ for $s\in\nbb$. Indeed, one can use the observation $f(\bw_{t-1};\bz_{i_t},\bz_{i_{t-s}})=f(\bw_{t-s-1};\bz_{i_t},\bz_{i_{t-s}})+\O\big(\sum_{j=1}^s\eta_{t-j}\big)$ to derive the convergence rate $\O(\sqrt{s}/\sqrt{t})$.
\end{remark}

Below we present high-probability bounds to understand the variation of the algorithm. We need to take conditional expectation of $f(\bw_{2j-2}; \bz_{i_{2j}}, \bz_{i_{2j- 1}})$ w.r.t. $(i_{2j},i_{2j-1})$ to get $F_S(\bw_{2j-2})$. However, there is a coupling between $(i_{2j},i_{2j-1})$ and $(i_{2j-1},i_{2j-2})$. Therefore, one can not directly apply concentration inequalities for martingales to handle $\sum_{j=1}^t\big(f(\bw_{2j-2}; \bz_{i_{2j}}, \bz_{i_{2j-1}})-F_S(\bw_{2j-2})\big)$. We introduce a novel decoupling technique to handle this coupling. Note that the high-probability bounds match the bounds in expectation up to a constant factor.
\begin{theorem}\label{thm:opt-hp}
Let $\bw_{-1} = \bw_0$ and $\{\bw_j: j\in [t]\}$ be produced by Algorithm \ref{alg:markov}.  Let (A1) and (A3) hold true with $\alpha\ge 0$ and $\sup_{\bw}f(\bw;\bz,\bz')\leq B$ for some $B>0$. Let $\delta\in(0,1)$.
\begin{enumerate}[label=(\alph*), leftmargin=*]\setlength\itemsep{-1mm}
    \item Assume $f$ is convex, i.e. $\alpha= 0$ and let $\barw_t =   \sum_{j=1}^{t} \eta_j\bw_{j-2} / \sum_{j=1}^{t}\eta_j$. Then, for any $\bw\in\wcal$,   with probability at least $1-\delta$ the following inequality holds
\[
  \!\!\!F_S(\barw_t) - F_S(\bwa)  \!\le\! \frac{1}{ \sum_{j=1}^t \eta_j}\Big(\!2B\Big(2\sum_{j=1}^{t}\eta_{j}^2\log(2/\delta)\Big)^{\frac{1}{2}}+
 \frac{1}{2}\|\bw_{1} -  \bwa\|_2^2  + G^2 \sum_{j=1}^t\big(\eta_{j-1}\eta_j+ \frac{\eta^2_j}{{2}}\big)\!\Big).
\]
\item Assume $f$ is $\alpha$-strongly convex with $\alpha>0$ and $\eta_j = \frac{2}{\alpha(j+1)}$. Let $\barw_t =   \sum_{j=1}^{t} j\bw_{j-2} / \sum_{j=1}^{t}j$. Then, with probability at least $1-\delta$, we have
$
F_S(\barw_t) - F_S(\bwa)  =  \mathcal{O}\bigl({G^2\log(1/\delta) /(\alpha t)}\bigr).
$
\end{enumerate}
\end{theorem}
 
Finally, we study the convergence of Algorithm \ref{alg:markov} associated with nonconvex functions. 
We first consider general smooth problems.
Since we cannot find a global minimum in this setting, we measure the convergence rate in terms of gradient norms~\citep{ghadimi2013stochastic}. The following theorem establishes the convergence rate $\O(1/\sqrt{t})$ for $\min_{j=1,\ldots,t}\ebb_{\acal}[\|\nabla F_S(\bw_j)\|_2^2]$.
\begin{theorem}\label{thm:opt-nc}
Let $\bw_{-1} \!=\! \bw_0$ and $\{\bw_j\!:\! j\in [t]\}$ be produced by Algorithm \ref{alg:markov} with $\eta_j\!=\!\eta\leq 1/(2\sqrt{L})$. Let (A2) hold true and $\ebb_{i_{j+1},i_{j+2}}\big[\|\nabla f(\bw_j;\bz_{i_{j+1}},\bz_{i_{j+2}})\|_2^2\big]\leq\alpha_0^2$ for some $\alpha_0$ and any $j$. Then, 
  $$\frac{1}{t}\sum_{j=1}^{t}\ebb_{\acal}\big[\|\nabla F_S(\bw_{j-2})\big\|_2^2\big] \leq \frac{F_S(\bw_{0})}{t\eta}+8L\eta\alpha_0^2. $$
  Furthermore, choosing $\eta\asymp1/\sqrt{t}$ implies that $ \frac{1}{t}\sum_{j=1}^{t}\ebb_{\acal}\big[\|\nabla F_S(\bw_{j-2})\big\|_2^2\big]=\O(1/\sqrt{t})$.
\end{theorem}
We now turn to nonconvex problems under a PL condition. Theorem \ref{thm:opt-pl} gives convergence rates of the order $\O(1/t)$, which match the existing results for standard SGD in pointwise learning~\citep{karimi2016linear}. 
\begin{theorem}\label{thm:opt-pl}
 Assume (A2) and (A4) hold true and $\ebb_{i_{j+1},i_{j+2}}\big[\|\nabla f(\bw_j;\bz_{i_{j+1}},\bz_{i_{j+2}})\|_2^2\big]\leq\alpha_0^2$ for some $\alpha_0$ and any $j$. Let $\bw_{-1} \!=\! \bw_0$ and $\{\bw_j\!:\! j\in [t]\}$ be produced by Algorithm \ref{alg:markov} with $\eta_j=2/(\mu(j+1))$. Then
  $$ 
  \ebb_{\acal}[F_S(\bw_t)-F_S(\bw_S)] \leq \frac{32L\alpha_0^2}{\mu^2}\Big(\frac{1}{t+1}+\frac{\log(et)}{\mu t(t+1)}\Big).$$
\end{theorem}

\section{Application: Differentially Private SGD for Pairwise Learning\label{sec:private}}
 
We now use Algorithm \ref{alg:markov} and our stability analysis (i.e. Theorem \ref{lem:est-smooth}) to develop a differentially private algorithm for pairwise learning. Let us start with the definition of differential privacy \citep{dwork2014algorithmic}.
\begin{definition}
A (randomized) algorithm $\acal$ is called $(\epsilon,\delta)$-differentially private (DP) if, for all neighboring datasets $S,S'$ and for all events $O$ in the output space of $\acal$, one has $\pbb[\acal(S) \in O] \leq e^\epsilon \pbb[\acal(S') \in O] + \delta.$
\end{definition}

\begin{algorithm}[t]
\caption{Differentially Private Localized SGD for Pairwise Learning\label{alg:dp-iterative-localization-2}}
\begin{algorithmic}[1]
\STATE {\bf Inputs:} Dataset $S = \{\zbf_i: i\in[n]\}$, parameters $\epsilon, \delta > 0$, and learning rate $\eta$, initial point $\wbf_0$
\STATE Set $K\!=\!\lceil\log_2 n\rceil$ and divide $S$ into $K$ disjoint subsets $\{S_1, \cdots, S_K\}$ where $|S_k| \!=\! n_k \!=\! 2^{-k}n$.
\FOR{$k=1$ to $K$}
\STATE Set $\eta_k = 4^{-k} \eta$
\STATE Compute $\bar{\wbf}_k$ by Algorithm \ref{alg:markov} based on $S_k$ and initiated at $\wbf_{k-1}$ for $\lceil n_k\log(4/\delta)\rceil$ steps. 
\STATE Set $\wbf_k = \bar{\wbf}_k + \ubf_k$ where $\ubf_k \sim \ncal(0,\sigma_k^2I_d)$ with $\sigma_k = 12G\eta_k\log(4/\delta)\sqrt{2\log(2.5/\delta)}/\epsilon$.
\ENDFOR
\STATE {\bf Outputs:} $\wbf_K$
\end{algorithmic}
\end{algorithm}
Our proposed DP algorithm for pairwise learning is described in Algorithm \ref{alg:dp-iterative-localization-2} which is inspired by the iterative localization technique \citep{feldman2020private} for pointwise learning.  The privacy and utility guarantees are given by the following theorem. Here $D$ denotes  the diameter of $\wcal$.

\begin{theorem}\label{thm:privacy-utility-2}
Let (A1), (A2), and (A3) hold true with $\sigma=0$. Let $\{\wbf_k: k \in [K]\}$ be produced by Algorithm \ref{alg:dp-iterative-localization-2} with $\eta = \frac{D}{G}\min\{\frac{\log(4/\delta)}{\sqrt{n}}, \frac{\epsilon}{12\log(4/\delta)\sqrt{2d\log(2.5/\delta)}}\} \leq \frac{2}{L}$. Then, Algorithm \ref{alg:dp-iterative-localization-2} satisfies $(\epsilon, \delta)$-DP and, with gradient complexity  $\ocal(n\log(1/\delta)),$ we have the utility bound that $$
\ebb[F(\wbf_K) - F(\wbf^*)] = \ocal\Big(GD\Big(\frac{1}{\sqrt{n}} + \frac{\sqrt{d}\log^{\frac{3}{2}}(1/\delta)}{\epsilon n}\Big)\Big).$$
\end{theorem}
The main difference from the pointwise setting in \citep{feldman2020private} is that Algorithm \ref{alg:dp-iterative-localization-2} involves the coupling dependency between $\{i_t, i_{t-1}\}$ at time $t$ in Algorithm \ref{alg:markov} and $\{i_{t-1}, i_{t-2}\}$ at time $t-1$, which renders the direct application of the standard concentration inequalities infeasible.  We propose a novel decomposition to circumvent this hurdle (see more detailed proof for  Theorem \ref{thm:privacy-utility-2} in Appendix \ref{sec:private-proof}).  

\begin{remark}
The above bound matches the lower bound given in \cite{bassily2014private} for $(\epsilon,\delta)$-differentially private pointwise learning up to a $\log(1/\delta)$ term. 
Our utility bound improves over the previous work \citep{huai2020pairwise} which has the bound $\ocal(\sqrt{d\log(1/\delta)}\log(n/\delta)/(\sqrt{n}\epsilon))$. During the preparation of this work, we notice a very recent paper \citep{xue2021differentially} also studied the private version of pairwise algorithm by using the localization technique. Their algorithm establishes the optimal rate $\ocal\big(1/\sqrt{n} +  \sqrt{d\log(1/\delta)}/(\epsilon n)\big)$ which, however, needs an expensive gradient complexity $\ocal(n^3\log(1/\delta))$. As a comparison, we achieve nearly optimal utility bound with linear gradient complexity $\ocal(n\log(1/\delta))$. 
\end{remark}

\begin{remark}
In Appendix \ref{sec:private-2}, we further remove the smoothness assumption (A2) required in Theorem \ref{thm:privacy-utility-2} and propose a private algorithm (stated as Algorithm \ref{alg:dp-iterative-localization} there) that achieves the optimal rate $\ocal\big(\big(1/\sqrt{n} +  \sqrt{d\log(1/\delta)}/(\epsilon n)\big)\big)$ with gradient complexity $\ocal(n^2\log(1/\delta))$. Such bound improves over the previous known results with nonsmooth losses  \citep{yang2021stability} where the utility bound was $\ocal(\sqrt{d\log(1/\delta)}\log(n/\delta)/(\sqrt{n}\epsilon))$.
\end{remark}

\section{Experimental Validation}\label{sec:exp}
We now report some preliminary experiments\footnote{The source codes are available at \url{https://github.com/zhenhuan-yang/simple-pairwise}.} on AUC maximization with $f(\wbf; (\xbf, y), (\xbf', y')) = \ell(\wbf^\top(\xbf - \xbf'))\ibb_{[y=1\wedge y'=-1]}$ where $\ell$ is a surrogate loss function, e.g., the hinge loss $\ell(t) = (1 - t)_+.$

The purpose of our first   experiment is to compare our algorithm, i.e. Algorithm \ref{alg:markov},  against four existing algorithms for pairwise learning in terms of generalization and CPU running time on several datasets available from the LIBSVM website \citep{CC01a}.  These algorithms are: 1) \code{OLP} \citep{Kar} uses a buffer $B_t$ updated by a variant of Reservoir sampling with replacement where the buffer size is chosen to be $200$ in order to guarantee the maximum AUC score as indicated in \cite{Kar}; 2) $\code{OAM}_{gra}$ \citep{zhao2011online} is tailored for AUC maximization with the hinge loss which uses buffers by Reservoir sampling.  The buffer size is set to be $100$ for both positive and negative buffers as suggested in that paper; 3) $\code{SGD}_{pair}$ \citep{lei2020sharper} randomly pick a pair from $\binom{n}{2}$ pairs by uniform distribution;  4) \code{SPAUC} \citep{lei2021stochastic}, where AUC maximization problem with the least square loss was reformulated as stochastic saddle point (min-max) problem.  Note that \code{SPAUC} and $\code{OAM}_{gra}$ can only apply to AUC maximization problem with the least square loss and hinge loss, respectively.

To validate the generalization ability, the surrogate loss for Algorithm \ref{alg:markov}, $\code{SGD}_{pair}$ and \code{OLP} is chosen to be the hinge loss. Average AUC scores of different algorithms are listed in Table \ref{tab:gen-hinge} where we can see that our algorithm yields competitive generalization performance with \code{OAM} and \code{OLP} using a large buffering set.    Detailed experimental setup, data statistics and more  results such as comparison with \code{OLP} and \code{OAM} with the size of the buffering set $s=1$ are listed in Appendix \ref{sec:more-exp}. 
 

\begin{table*}[t]
\centering
\setlength{\tabcolsep}{2pt}
\small
\caption{Average AUC score $\pm$ standard deviation across multiple datasets. Our best results are highlighted in bold.}
\begin{tabular}{@{\hskip1pt}c@{\hskip1pt}|c|c|c|c|c|c}
\hline
 Algorithm & \code{diabetes} & \code{german} & \code{ijcnn1} & \code{letter} & \code{mnist} & \code{usps} \\\hline\hline 
\code{Ours} & \bm{$.831 \pm .030$} &  $.793 \pm .021$ & \bm{$.934 \pm .002$} & $.810 \pm .007$  & \bm{$.932 \pm .001$} & \bm{$.926 \pm .006$}  \\\hline
$\code{SGD}_{pair}$~\citep{lei2020sharper} & $.830 \pm .028$ &  \textbf{$.794 \pm .023$} & $.934 \pm .003$ & $.811 \pm .008$  & $.932 \pm .001$ & $.925 \pm .006$ \\\hline
\code{OLP} \citep{Kar} & $.825 \pm .028$ & $.787 \pm .028$ & $.916 \pm .003$ & $.808 \pm .010$ & $.927 \pm .003$  & $.917 \pm .006$ \\\hline
$\code{OAM}_{gra}$ \citep{zhao2011online} & $.828 \pm .026$ & $.785 \pm .029$ & $.930 \pm .003$ & $.806 \pm .008$  & $.898 \pm .002$ & $.916 \pm .005$ \\\hline
\code{SPAUC} \citep{lei2021stochastic} & $.828 \pm .031$ & $.799 \pm .026$ & $.932 \pm .002$ & $.809 \pm .008$ & $.927 \pm .002$ & $.923 \pm .005$ \\\hline
\end{tabular}
\label{tab:gen-hinge}
\end{table*}

\begin{figure}[ht!]
\begin{subfigure}{.33\textwidth}
\centering
\includegraphics[width=.95\linewidth]{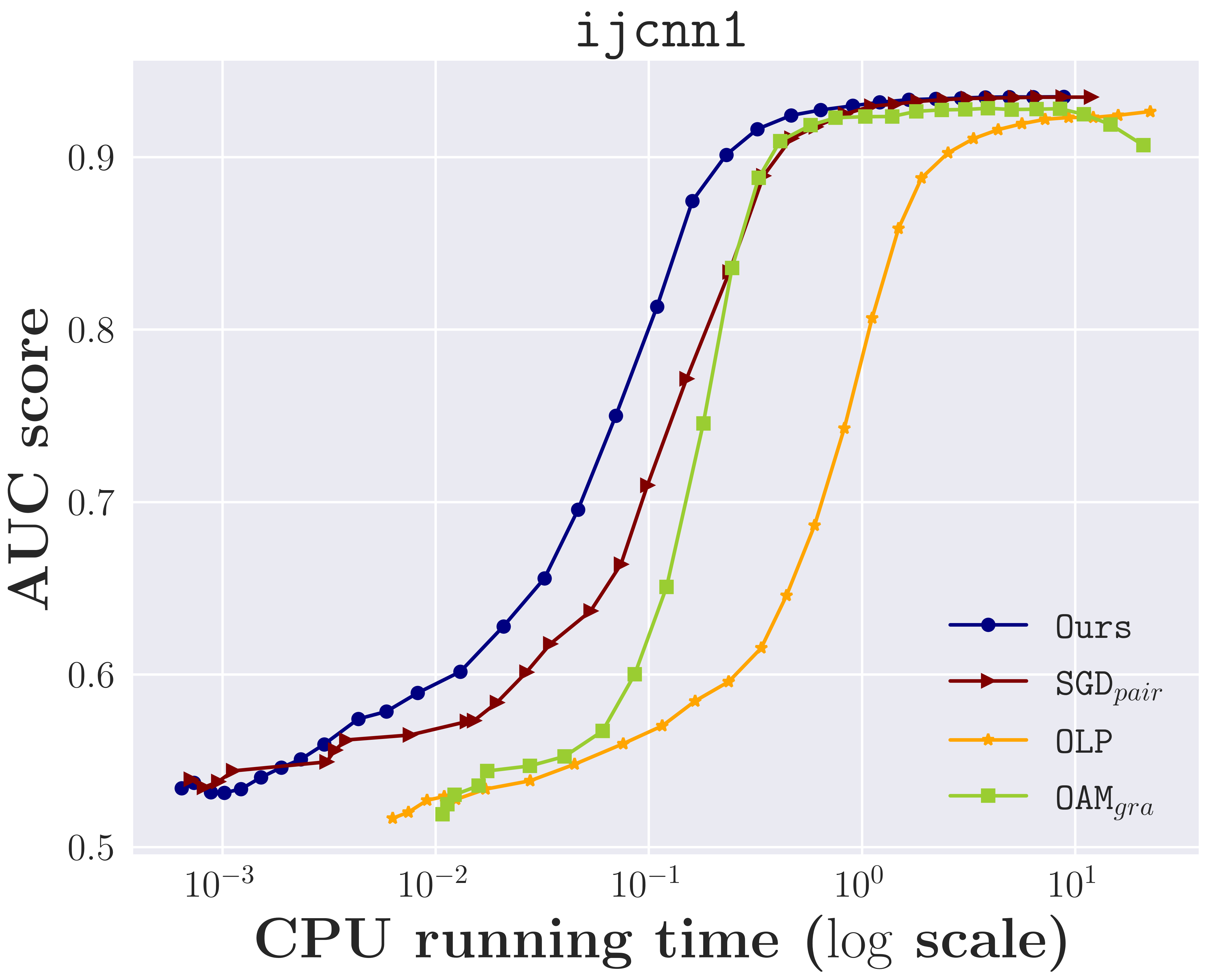}  
\end{subfigure}
\begin{subfigure}{.33\textwidth}
\centering
\includegraphics[width=.95\linewidth]{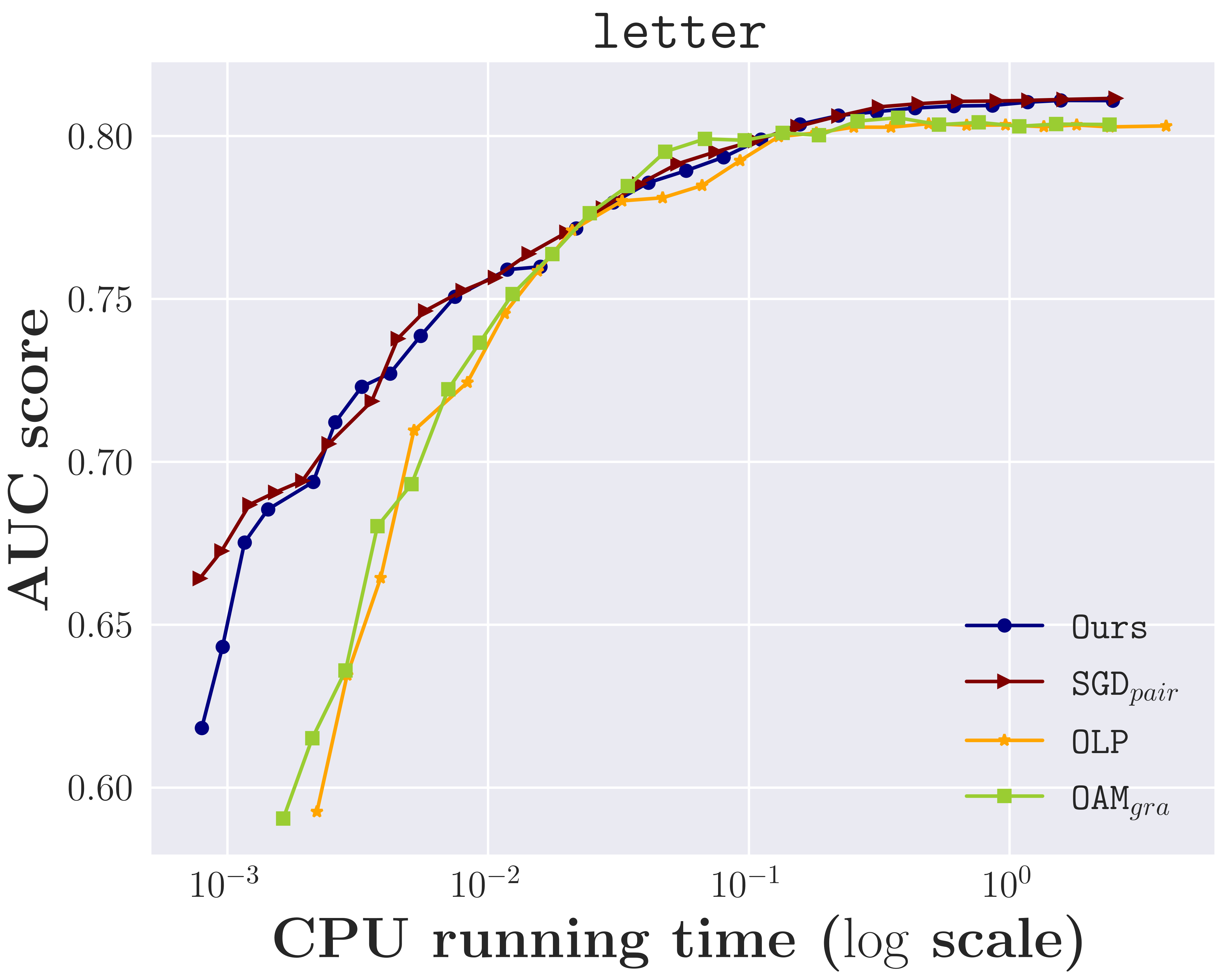}  
\end{subfigure}
\begin{subfigure}{.33\textwidth}
\centering
\includegraphics[width=.95\linewidth]{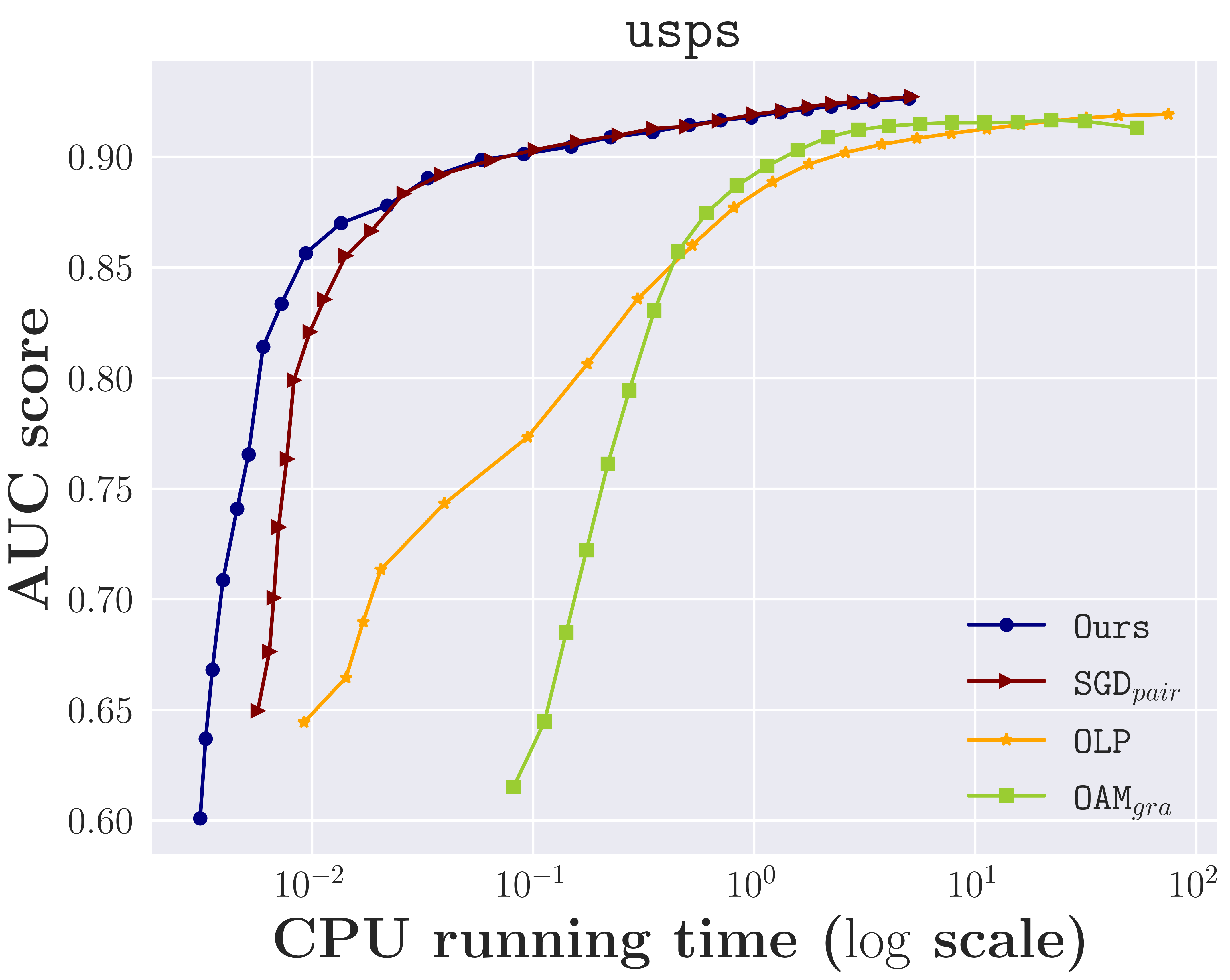}  
\end{subfigure}
\caption{CPU running time ($\log$ scale) versus the AUC score\label{fig:iter-hinge}}
\end{figure}

To fairly compare the CPU running time, we apply the following uniform setting across all algorithms: 1) $\W$ is an $\ell_2$ ball with the same diameter; 2) the step sizes $\eta_t = \eta$ which is tuned by cross validation. 
We report the results in Figure \ref{fig:iter-hinge} for the hinge loss. We can see that CPU running time for our algorithm and $\code{SGD}_{pair}$ are similar while  \code{OLP} and \code{OAM} needs more time to converge.  The possible reason behind this is that they have a high gradient complexity $\O(s)$ at each iteration while ours is $\O(1).$  More results on comparison for the least square loss and for differentially private algorithms are given in Appendix \ref{sec:more-exp}.




\section{Conclusion}\label{sec:conclusion}
In this paper, we propose simple stochastic and online  gradient descent algorithms for pairwise learning. The key idea is to build a gradient estimator by pairing the current instance with the previous instance, which enjoys favorable computation and storage complexity. We leverage the lens of algorithmic stability to study its generalization and apply tools in optimization theory to study its convergence rates for various problems including convex/nonconvex and smooth/nonsmooth settings. We also use  our algorithms and stability analysis to develop a new DP algorithm  for pairwise learning with differential privacy constraints which significantly improves the existing results. 
The main  difference from pointwise learning in the analysis is the coupling between models and previous instances, which is handled by introducing novel decoupling techniques. 

For future work, it would be interesting to see whether the analysis and results still hold true if the current example $\zbf_{i_t}$ in Algorithm \ref{alg:markov} is paired with one arbitrary previous example (e.g., $ \zbf_{i_1}$ ).  Other future work would be a systematic extension of our algorithms using other acceleration schemes such as momentum and variance reduction techniques.

\section*{Acknowledgments}\vspace*{-2mm}
The authors are grateful to the anonymous reviewers for their constructive comments and suggestions. Tianbao Yang is partially supported by NSF Career Award \#1844403 and NSF Award \#2110545.  Yiming Ying  is supported by NSF under grants DMS-2110836, IIS-1816227, IIS-2110546, and  IIS-2103450. 

\setlength{\bibsep}{0.06cm}
\bibliographystyle{abbrvnat}
\bibliography{bibfile,learning}

\newpage
\newpage
\appendix
\numberwithin{equation}{section}
\numberwithin{theorem}{section}
\numberwithin{figure}{section}
\numberwithin{table}{section}
\renewcommand{\thesection}{{\Alph{section}}}
\renewcommand{\thesubsection}{\Alph{section}.\arabic{subsection}}
\renewcommand{\thesubsubsection}{\Roman{section}.\arabic{subsection}.\arabic{subsubsection}}
\setcounter{secnumdepth}{-1}
\setcounter{secnumdepth}{3}


\section{Proofs of Generalization Error (Theorems 4 and 5)\label{sec:proof-est}}

In this section, we present the stability and generalization error bounds for SGD with convex loss functions. We first prove Theorem \ref{lem:est-nonsmooth} on smooth problems and then Theorem \ref{lem:est-smooth} on nonsmooth problems.  Recall $\ebb_{\acal}$ denotes the expectation w.r.t. the internal randomness of $\acal$. For SGD, this means the expectation w.r.t. $\{i_j\}_{j\in[t]}$.
\vspace*{-2mm}
\begin{proof}[{\bf Proof of Theorem \ref{lem:est-nonsmooth}}]
We first investigate the uniform stability of Algorithm \ref{alg:markov}. 
Let $S'=\{z_1,\ldots,\bz_{n-1},\bz_{n}'\}$, where $z_n'$ is independently drawn from $\rho$, and  $\{\bw_t'\}$ be produced by Algorithm \ref{alg:markov} w.r.t. data $S'$.
We consider two cases:  i.e. the case of  $\{ i_t\neq n \text{ and } i_{t-1}\neq n \}$ and the case of $\{ i_t= n \text{ or }  i_{t-1}= n\}.$

If $i_t\neq n$ and $i_{t-1}\neq n$, then
\begin{align*}
\big\|\bw_t-\bw_t'\big\|_2^2 & \leq \big\|\bw_{t-1}-\eta_t\nabla f(\bw_{t-1};\bz_{i_t},\bz_{i_{t-1}})-\bw'_{t-1}+\eta_t\nabla f(\bw'_{t-1};z'_{i_t},z'_{i_{t-1}})\big\|_2^2 \\
& = \big\|\bw_{t-1}-\eta_t\nabla f(\bw_{t-1};\bz_{i_t},\bz_{i_{t-1}})-\bw'_{t-1}+\eta_t\nabla f(\bw'_{t-1};\bz_{i_t},\bz_{i_{t-1}})\big\|_2^2 \\
& = \|\bw_{t-1}-\bw'_{t-1}\|_2^2 + \eta_t^2\big\|\nabla f(\bw_{t-1};\bz_{i_t},\bz_{i_{t-1}})-\nabla f(\bw'_{t-1};\bz_{i_t},\bz_{i_{t-1}})\big\|_2^2\\
& -2\eta_t\langle\bw_{t-1}-\bw_{t-1}',\nabla f(\bw_{t-1};\bz_{i_t},\bz_{i_{t-1}})-\nabla f(\bw'_{t-1};\bz_{i_t},\bz_{i_{t-1}})\rangle\\
& \leq \|\bw_{t-1}-\bw'_{t-1}\|_2^2 + \eta_t^2\big\|\nabla f(\bw_{t-1};\bz_{i_t},\bz_{i_{t-1}})-\nabla f(\bw'_{t-1};\bz_{i_t},\bz_{i_{t-1}})\big\|_2^2\\
& \leq \|\bw_{t-1}-\bw'_{t-1}\|_2^2 + 4\eta_t^2G^2,
\end{align*}
where the last second inequality follows from the inequality
$\langle\bw_{t-1}-\bw_{t-1}',\nabla f(\bw_{t-1};\bz_{i_t},\bz_{i_{t-1}})-\nabla f(\bw'_{t-1};\bz_{i_t},\bz_{i_{t-1}})\rangle\geq0$ due to the convexity of $f$ and the last inequality follows from the Lipschitz continuity of $f$. If $i_t=n$ or $i_{t-1}=n$, it follows from the elementary inequality $(a+b)^2\leq(1+p)a^2+(1+1/p)b^2$ and the Lipschitz condition that
\begin{align*}
  \big\|\bw_t-\bw_t'\big\|_2^2 & \leq (1+p)\|\bw_{t-1}-\bw'_{t-1}\|_2^2 \\ & +(1+1/p)\eta_t^2\big\|\nabla f(\bw_{t-1};\bz_{i_t},\bz_{i_{t-1}})-\nabla f(\bw'_{t-1};z'_{i_t},z'_{i_{t-1}})\big\|_2^2\\
  & \leq (1+p)\|\bw_{t-1}-\bw'_{t-1}\|_2^2+4(1+1/p)\eta_t^2G^2.
\end{align*}
We can combine the above two cases together and derive
\begin{align*}
\big\|\bw_t-\bw_t'\big\|_2^2
& \leq \big(\|\bw_{t-1}-\bw'_{t-1}\|_2^2 + 4\eta_t^2G^2\big)\ibb_{[i_t\neq  n\text{ and }i_{t-1}= n]} \\
& + \big((1+p)\|\bw_{t-1}-\bw'_{t-1}\|_2^2+4(1+1/p)\eta_t^2G^2\big)\ibb_{[i_t= n\text{ or }i_{t-1}= n]} \\
& \leq \big(1+p\ibb_{[i_t= n\text{ or }i_{t-1}= n]}\big)\|\bw_{t-1}-\bw'_{t-1}\|_2^2+4\eta_t^2G^2\big(1+\ibb_{[i_t= n\text{ or }i_{t-1}= n]}/p\big)\\
& = \big(1+p\big)^{\ibb_{[i_t= n\text{ or }i_{t-1}= n]}}\|\bw_{t-1}-\bw'_{t-1}\|_2^2+4\eta_t^2G^2\big(1+\ibb_{[i_t= n\text{ or }i_{t-1}= n]}/p\big),
\end{align*}
where $\ibb_{[\cdot]}$ is the indicator function.
We can apply the above inequality recursively and get
\begin{align*}
\big\|\bw_t-\bw_t'\big\|_2^2 &\leq 4G^2\sum_{k=1}^{t}\eta_k^2\big(1+\ibb_{[i_k= n\text{ or }i_{k-1}= n]}/p\big)\prod_{j=k+1}^{t}\big(1+p\big)^{\ibb_{[i_j= n\text{ or }i_{j-1}= n]}}\\
& \leq 4G^2\prod_{j=1}^{t}\big(1+p\big)^{\ibb_{[i_j= n\text{ or }i_{j-1}= n]}}\sum_{k=1}^{t}\eta_k^2\big(1+\ibb_{[i_k= n\text{ or }i_{k-1}= n]}/p\big)\\
& = 4G^2\eta^2\big(1+p\big)^{\sum_{j=1}^{t}\ibb_{[i_j= n\text{ or }i_{j-1}= n]}}\big(t+\sum_{k=1}^{t}\ibb_{[i_k= n\text{ or }i_{k-1}= n]}/p\big),
\end{align*}
where the last inequality follows from $\eta_j=\eta$. 

Now, we choose $p=1/\big(\sum_{j=1}^{t}\ibb_{[i_j= n\text{ or }i_{j-1}= n]}\big)$ and use the inequality $(1+x)^{1/x}\leq e$ to derive the following inequality
\[
\big\|\bw_t-\bw_t'\big\|_2^2\leq 4eG^2\eta^2\Big(t+\big(\sum_{k=1}^{t}\ibb_{[i_k= n\text{ or }i_{k-1}= n]}\big)^2\Big).
\]
By the inequality $\ibb_{[i_k= n\text{ or }i_{k-1}= n]}\leq \ibb_{[i_k= n]}+\ibb_{[i_{k-1}= n]}$ and $(a+b)^2\leq 2(a^2+b^2)$ we know
\begin{align*}
\ebb_{\acal}\Big[\big(\sum_{k=1}^{t}\ibb_{[i_k= n\text{ or }i_{k-1}= n]}\big)^2\Big]
& \leq 2\ebb_{\acal}\Big[\big(\sum_{k=1}^{t}\ibb_{[i_k= n}\big)^2\Big]+2\ebb_{\acal}\Big[\big(\sum_{k=1}^{t}\ibb_{[i_{k-1}= n}\big)^2\Big]\\
& = 4\ebb\Big[\big(\sum_{k=1}^{t}\ibb_{[i_k= n]}\big)^2\Big]
\leq 4t+4\sum_{j,k\in[t]:j\neq k}\ebb\big[\ibb_{[i_j= n]}\ibb_{[i_k= n]}\big]\\
& = 4t+4\sum_{j,k\in[t]:j\neq k}\frac{1}{n^2} \leq 4t+4t^2/n^2.
\end{align*}
We can combine the above two inequalities together and derive
\[
\ebb_\acal [\|\wbf_t - \wbf_t'\|_2^2] \leq 4eG^2\eta^2 \Big(5t+\frac{4t^2}{n^2}\Big)
\]
and by the convexity of $\|\cdot\|_2^2$ it follows
\[
\ebb_\acal [\|\bar{\wbf}_t - \bar{\wbf}_t'\|_2^2] \leq \frac{1}{t}\sum_{j=1}^t\ebb_\acal [\|\wbf_j - \wbf_j'\|_2^2] \leq 4eG^2\eta^2 \Big(5t+\frac{4t^2}{n^2}\Big).
\]
This establishes the uniform stability of Algorithm \ref{alg:markov}.
Furthermore, for any $\zbf, \zbf'$, we have
\begin{align*}
\ebb_\acal [f(\bar{\wbf}_t, \zbf, \zbf') - f(\bar{\wbf}_t', \zbf, \zbf')] \leq & G\ebb_\acal [\|\bar{\wbf}_t - \bar{\wbf}_t'\|_2] = G\ebb_\acal \Big[\sqrt{\|\bar{\wbf}_t - \bar{\wbf}_t'\|_2^2}\Big]\\
\leq & G\sqrt{\ebb_\acal [\|\bar{\wbf}_t - \bar{\wbf}_t'\|_2^2]} \leq 2\sqrt{e}G^2\eta\Big(\sqrt{5t}+\frac{2t}{n}\Big)
\end{align*}
where the first inequality we used the $G$-Lipschitz continuity of $f$ and the second inequality we used the Jensen's inequality. Therefore, Algorithm \ref{alg:markov} is $2\sqrt{e}G^2\eta\Big(\sqrt{5t}+\frac{2t}{n}\Big)$-uniformly stable. By Lemma \ref{lem:generalization-via-stability} it follows
\[
\ebb_\acal[F(\bar{\wbf}_t) - F_S(\bar{\wbf}_t)] \leq 4\sqrt{e}G^2\eta\Big(\sqrt{5t}+\frac{2t}{n}\Big),
\]
which gives us the desired result. 
\end{proof}

To prove Theorem \ref{lem:est-smooth} we require the following lemma on the nonexpansiveness of gradient map $\bw\mapsto\bw-\eta \nabla f(\bw;z,z')$.
\begin{lemma}[\citealt{hardt2016train}\label{lem:nonexpansive}]
Assume for all $z\in\zcal$, the function $\bw\mapsto f(\bw;z,z')$ is convex and $L$-smooth. Then for all $\eta\leq2/L$ and $z,z'\in\zcal$ there holds
\[
\|\bw-\eta \nabla f(\bw;z,z')-\bw'+\eta \nabla f(\bw';z,z')\|_2\leq \|\bw-\bw'\|_2.
\]
\end{lemma}
\begin{proof}[{\bf Proof of Theorem \ref{lem:est-smooth}}]
Let $S'=\{z_1,\ldots,\bz_{n-1},\bz_{n}'\}$, where $z_n'$ is independently drawn from $\rho$.
Let $\{\bw_t'\}$ be produced by Algorithm \ref{alg:markov} w.r.t. $S'$.
We consider two cases. If $i_t\neq n$ and $i_{t-1}\neq n$, then it follows from Lemma \ref{lem:nonexpansive} that
\begin{align*}
\big\|\bw_t-\bw_t'\big\|_2 & \leq \big\|\bw_{t-1}-\eta_t\nabla f(\bw_{t-1};\bz_{i_t},\bz_{i_{t-1}})-\bw'_{t-1}+\eta_t\nabla f(\bw'_{t-1};z'_{i_t},z'_{i_{t-1}})\big\|_2 \\
& = \big\|\bw_{t-1}-\eta_t\nabla f(\bw_{t-1};\bz_{i_t},\bz_{i_{t-1}})-\bw'_{t-1}+\eta_t\nabla f(\bw'_{t-1};\bz_{i_t},\bz_{i_{t-1}})\big\|_2 \\
& \leq \|\bw_{t-1}-\bw'_{t-1}\|_2.
\end{align*}
Otherwise, we know
\begin{align*}
\big\|\bw_t-\bw_t'\big\|_2 & \leq \|\bw_{t-1}-\bw'_{t-1}\|_2+\eta_t\big\|\nabla f(\bw_{t-1};\bz_{i_t},\bz_{i_{t-1}})-\nabla f(\bw'_{t-1};z'_{i_t},z'_{i_{t-1}})\big\|_2 \\
& \leq \|\bw_{t-1}-\bw'_{t-1}\|_2+2\eta_tG.
\end{align*}
We can combine the above two cases together and derive the following inequality
\begin{align*}
\big\|\bw_t-\bw_t'\big\|_2 & \leq \|\bw_{t-1}-\bw'_{t-1}\|_2\ibb_{[i_t\neq n\text{ and }i_{t-1}\neq n]} + \big(\|\bw_{t-1}-\bw'_{t-1}\|_2+2\eta_tG\big)\ibb_{[i_t= n\text{ or }i_{t-1}= n]}\\
& = \|\bw_{t-1}-\bw'_{t-1}\|_2 + 2\eta_tG\ibb_{[i_t= n\text{ or }i_{t-1}= n]}.
\end{align*}
We can apply the above inequality recursively and get
\[
\big\|\bw_t-\bw_t'\big\|_2\leq2G\sum_{j=1}^{t}\eta_j\ibb_{[i_j= n\text{ or }i_{j-1}= n]}\leq 2G\sum_{j=1}^{t}\eta_j\big(\ibb_{[i_j= n]}+\ibb_{[i_{j-1}= n]}).
\]
Taking expectations over both sides gives
$
\ebb_{\acal}\big[\|\bw_t-\bw_t'\big\|_2\big]\leq \frac{4G}{n}\sum_{j=1}^{t}\eta_j.
$
It then follows from the convexity of $\|\cdot\|_2$ that
\[
\ebb_{\acal}\big[\|\bar{\wbf}_t-\bar{\wbf}_t'\big\|_2\big]\leq \frac{4G}{n}\sum_{j=1}^{t}\eta_j.
\]
This establishes the uniform argument stability of Algorithm \ref{alg:markov}.
Furthermore, it follows the Lipschitz condition that
\[
\sup_{z,z'}\ebb_{\acal}\big[f(\bar{\wbf}_t;z,z')-f(\bar{\wbf}_t';z,z')\big]\leq \frac{4G^2}{n}\sum_{j=1}^{t}\eta_j.
\]
The desired result then follows from Lemma \ref{lem:generalization-via-stability}. The proof for Theorem \ref{lem:est-smooth} is completed.
\end{proof}

Finally, we consider the generalization analysis for nonconvex problems under the PL condition.
To prove Theorem \ref{thm:exg-pl}, we first introduce a lemma motivated by the arguments in \citep{hardt2016train}.
\begin{lemma}\label{lem:stab-cond}
  Let $S=\{z_i\}_{i\in[n]}$ and $S'=\{z_i'\}_{i\in[n]}$ be neighboring datasets differing by a single example. 
  Let $\{\bw_t\}_t$ and $\{\bw_t'\}_t$ be produced by Algorithm \ref{alg:markov} w.r.t. $S$ and $S'$, respectively.
  Let Assumption (A1) hold and $\sup_{z,z'}f(\bw_i,z,z')\leq B$. Let $\triangle_t=\|\bw_t-\bw_t'\|_2$. Then { for every $z, z' \in \Z $ and every $t_0 \in [n]$, there holds}
  \[
  \ebb\big[|f(\bw_T;z,z')-f(\bw_T';z,z')|\big] \leq G\ebb\big[\triangle_T |\triangle_{t_0}=0\big]+\frac{Bt_0}{n}.
  \]
\end{lemma}
\begin{proof}
  Without loss of generality, we assume that $S$ and $S'$ differ by the last example. Let $\ecal$ denote the event $\triangle_{t_0}=0$.
  Then we have
  \begin{multline*}
    \ebb\big[|f(\bw_T;z,z')-f(\bw_T';z,z')|\big] = \ebb\big[|f(\bw_T;z,z')-f(\bw_T';z,z')| |\ecal\big]\mbox{Pr}\{\ecal\}\\
    +\ebb\big[|f(\bw_T;z,z')-f(\bw_T';z,z')| |\ecal^c\big]\mbox{Pr}\{\ecal^c\},
  \end{multline*}
  where $\ecal^c$ denotes the complement of $\ecal$.
  Furthermore, we know
  \[
  \mbox{Pr}\{\ecal^c\}\leq \sum_{t=1}^{t_0}\mbox{Pr}\{i_t=n\}=\frac{t_0}{n}.
  \]
  We can combine the above two inequalities and the Lipschitz continuity of $f$ to derive the stated bound, which completes the proof. 
\end{proof}

\section{Proofs of Optimization Error (Theorems 6-9) \label{sec:proof-opt}}

In this section, we prove optimization error bounds for SGD. We first consider convex cases, and prove convergence rates in expectation (Theorem \ref{thm:convergence}) and with high probability (Theorem \ref{thm:opt-hp}). Then, we establish convergence rates for SGD with nonconvex loss functions (Theorem \ref{thm:opt-nc} and Theorem \ref{thm:opt-pl}).
\begin{proof}[{\bf Proof of Theorem \ref{thm:convergence}}]
Consider $j\ge 1$. Note that $f(\cdot; z,z')$ is $\alpha$-strongly convex and $G$-Lipschitz continuous, we have
\begin{align*}
&\|\bw_j - \bwa\|_2^2
\leq \|\bw_{j-1}-\eta_j \nabla f(\bw_{j-1}; \bz_{i_j}, \bz_{i_{j-1}}) - \bwa  \|_2^2\\
& = \|\bw_{j-1} - \bwa\|_2^2  - 2\eta_j \langle \nabla f(\bw_{j-1}; \bz_{i_j}, \bz_{i_{j-1}}),  \bw_{j-1}-\bwa\rangle  +  \eta_j^2\|\nabla f(\bw_{j-1}; \bz_{i_j}, \bz_{i_{j-1}})\|_2^2 \\
& \le  (1-\eta_j \alpha)\|\bw_{j-1} - \bwa\|_2^2  - 2\eta_j [f(\bw_{j-1}; \bz_{i_j}, \bz_{i_{j-1}}) - f(\bwa; \bz_{i_j}, \bz_{i_{j-1}}) ]  +  G^2\eta_j^2 \\
& =  (1-\eta_j \alpha)\|\bw_{j-1} - \bwa\|_2^2  - 2\eta_j [f(\bw_{j-2}; \bz_{i_j}, \bz_{i_{j-1}}) - f(\bwa; \bz_{i_j}, \bz_{i_{j-1}}) ]\\ & + 2\eta_j  [f(\bw_{j-2}; \bz_{i_j}, \bz_{i_{j-1}}) - f(\bw_{j-1}; \bz_{i_j}, \bz_{i_{j-1}})] +  G^2\eta_j^2 \\
&\le  (1-\eta_j \alpha)\|\bw_{j-1} - \bwa\|_2^2  - 2\eta_j [f(\bw_{j-2}; \bz_{i_j}, \bz_{i_{j-1}}) - f(\bwa; \bz_{i_j}, \bz_{i_{j-1}}) ]\\ &  + 2\eta_j G  \|\bw_{j-1} - \bw_{j-2}\|_2+  G^2\eta_j^2 \\
& \le   (1-\eta_j \alpha)\|\bw_{j-1} - \bwa\|_2^2  - 2\eta_j [f(\bw_{j-2}; \bz_{i_j}, \bz_{i_{j-1}}) - f(\bwa; \bz_{i_j}, \bz_{i_{j-1}}) ] \\ & + 2G^2 \eta_j  \eta_{j-1} +  G^2\eta_j^2, \numberthis \label{eq:equ1}
\end{align*}
where the last inequality used the fact that $\|\bw_{j} - \bw_{j-1}\|_2 = \eta_j \|\nabla f(\bw_j; \bz_{i_j}, \bz_{i_{j-1}})\|_2 \le G \eta_j.$

For the convex case, i.e. $\alpha=0$, we know from \eqref{eq:equ1} that
\begin{align}\label{opt-convex}
   & \sum_{j=1}^t \eta_j [f(\bw_{j-2}; \bz_{i_j}, \bz_{i_{j-1}}) - f(\bwa; \bz_{i_j}, \bz_{i_{j-1}})]\notag\\
   & \le \frac{1}{2}\sum_{j=1}^t  [\|\bw_{j-1} - \bwa\|_2^2 - \|\bw_{j} - \bwa\|_2^2] + \frac{G^2}{2} \sum_{j=1}^t(2\eta_{j-1}\eta_j+ \eta^2_j) \notag\\
   & \le  \frac{1}{2}\|\bw_{0} -  \bwa\|_2^2  + \frac{G^2}{2} \sum_{j=1}^t(2\eta_{j-1}\eta_j+ \eta^2_j).
\end{align}
Taking the expectation on both sides of the above inequality and observing that $f(\cdot; z,z')$ is convex, we get the desired estimation \eqref{eq:parta}.

For the strongly-convex case, i.e. $\alpha>0$, we obtain from \eqref{eq:equ1} that
\[ f(\bw_{j-2}; \bz_{i_j}, \bz_{i_{j-1}}) - f(\bwa; \bz_{i_j}, \bz_{i_{j-1}}) \le \frac{\eta_j^{-1} - \alpha}{2}\|\bw_{j-1} - \bwa\|_2^2 -\frac{\eta_{j}^{-1}}{2}\|\bw_{j} -\bwa\|_2^2 + G^2 \eta_{j-1} +  \frac{G^2\eta_j}{2}.\]
Now, we choose  $\eta_j = \frac{2}{\alpha(j+1)}$ for any $j$,  which implies that 
\begin{align*}
    &j[f(\bw_{j-2}; \bz_{i_j}, \bz_{i_{j-1}}) - f(\bwa; \bz_{i_j}, \bz_{i_{j-1}}) ] \\
    &\le \frac{j(j-1)\alpha}{4}\|\bw_{j-1} -\bwa\|_2^2- \frac{ j(j+1)\alpha }{4}\|\bw_j - \bwa\|_2^2 + \frac{2G^2}{\alpha} +  \frac{G^2j}{\alpha (j+1)}\\
    &\le \frac{\alpha}{4} [ {j(j-1)}\|\bw_{j-1} -\bwa\|_2^2- { j(j+1)}\|\bw_j - \bwa\|_2^2] + \frac{3G^2}{\alpha}.
\end{align*}
Taking the summation over $j$ implies that
\begin{align}
 &\sum_{j=1}^t j[f(\bw_{j-2}; \bz_{i_j}, \bz_{i_{j-1}}) - f(\bwa; \bz_{i_j}, \bz_{i_{j-1}}) ]\nonumber \\
 &\le  \frac{3G^2 t}{\alpha} + \frac{\alpha}{4}\sum_{j=1}^t  [ {j(j-1)}\|\bw_{j-1} -\bwa\|_2^2- { j(j+1)}\|\bw_j - \bwa\|_2^2]\nonumber\\
 &\le \frac{3G^2 t}{\alpha}+ \frac{\alpha}{4}  [ 0 - { t(t+1)}\|\bw_t - \bwa\|_2^2] \le  \frac{3G^2 t}{\alpha}.\label{opt-sc}
\end{align}
Dividing both sides of the above inequality by $\sum_{j=1}^t j $ yields the desired estimation in part (b). 
\end{proof}

To prove high-probability bounds, we require the following lemma on concentration inequalities of martingales~\citep{boucheron2013concentration,zhang2005data}.
\begin{lemma}\label{lem:martingale}
  Let $\tilde{z}_1,\ldots,\tilde{z}_n$ be a sequence of random variables such that $\tilde{z}_k$ may depend on the previous variables $\tilde{z}_1,\ldots,\tilde{z}_{k-1}$ for all $k=1,\ldots,n$. Consider a sequence of functionals $\xi_k(\tilde{z}_1,\ldots,\tilde{z}_k),k=1,\ldots,n$.
  Let $\alpha_n^2=\sum_{k=1}^{n}\ebb_{\tilde{z}_k}\big[\big(\xi_k-\ebb_{\tilde{z}_k}[\xi_k]\big)^2\big]$ be the conditional variance.
  \begin{enumerate}[label=(\arabic{*}), leftmargin=*]
    \item Assume $|\xi_k-\ebb_{\tilde{z}_k}[\xi_k]|\leq b_k$ for each $k$. Let $\delta\in(0,1)$. With probability at least $1-\delta$
    \begin{equation}\label{hoeffding}
      \sum_{k=1}^{n}\ebb_{\tilde{z}_k}[\xi_k]-\sum_{k=1}^{n}\xi_k\leq \Big(2\sum_{k=1}^{n}b_k^2\log\frac{1}{\delta}\Big)^{\frac{1}{2}}.
    \end{equation}
    \item Assume that $\xi_k-\ebb_{\tilde{z}_k}[\xi_k]\leq b$ for each $k$. Let $\rho\in(0,1]$ and $\delta\in(0,1)$. With probability at least $1-\delta$ we have
    \begin{equation}\label{bernstein}
      \sum_{k=1}^{n}\ebb_{\tilde{z}_k}[\xi_k]-\sum_{k=1}^{n}\xi_k\leq \frac{\rho\alpha_n^2}{b}+\frac{b\log\frac{1}{\delta}}{\rho}.
    \end{equation}
  \end{enumerate}
\end{lemma}
\begin{proof}[{\bf Proof of Theorem \ref{thm:opt-hp}}]
For simplicity, we assume $t$ is an even number. We first consider the convex case.
Let
\[
\xi_j=\eta_{2j}\big(f(\bw_{2j-2}; \bz_{i_{2j}}, \bz_{i_{2j-1}})-f(\bwa; \bz_{i_{2j}}, \bz_{i_{2j-1}})\big),\quad j\in[t/2].
\]
It is obvious that $|\xi_j-\ebb_{i_{2j}, i_{2j-1}}[\xi_j]|\leq 2B\eta_{2j}$.
Let $\tilde{z}_j=(i_{2j}, i_{2j-1})$. It is clear that $\tilde{z}_j,j\in[t/2]$ are i.i.d. random variables.
Therefore, one can apply Part (a) of Lemma \ref{lem:martingale} to derive the following inequality with probability at least $1-\delta/2$
\[
\sum_{j=1}^{t/2}\ebb_{\tilde{z}_j}[\xi_j] -\sum_{j=1}^{t/2}\xi_j\leq 2B\Big(2\sum_{j=1}^{t/2}\eta_{2j}^2\log(2/\delta)\Big)^{\frac{1}{2}}.
\]
It is clear that $\ebb_{\tilde{z}_j}[\xi_j]=\eta_{2j}\big(F_S(\bw_{2j-2})
- {F_S(\bw)} \big)$. Therefore, the following inequality holds with probability at least $1-\delta/2$
\begin{multline*}
\sum_{j=1}^{t/2}\eta_{2j}\big(F_S(\bw_{2j-2})-F_S(\bwa)-f(\bw_{2j-2}; \bz_{i_{2j}}, \bz_{i_{2j-1}})+f(\bwa; \bz_{i_{2j}}, \bz_{i_{2j-1}})\big)
\leq 2B\Big(2\sum_{j=1}^{t/2}\eta_{2j}^2\log(2/\delta)\Big)^{\frac{1}{2}}.
\end{multline*}
In a similar way, one can derive the following inequality with probability at least $1-\delta/2$
\begin{multline*}
\sum_{j=1}^{t/2}\eta_{2j-1}\big(F_S(\bw_{2j-3})-F_S(\bwa)-f(\bw_{2j-3}; \bz_{i_{2j-1}}, \bz_{i_{2j-2}})\\+f(\bwa; \bz_{i_{2j-1}}, \bz_{i_{2j-2}})\big)
\leq 2B\Big(2\sum_{j=1}^{t/2}\eta_{2j-1}^2\log(2/\delta)\Big)^{\frac{1}{2}}.
\end{multline*}
We can combine the above two inequalities together and derive the following inequality with probability $1-\delta$
\begin{multline*}
\sum_{j=1}^{t}\eta_j\big(F_S(\bw_{j-2})-F_S(\bwa)-f(\bw_{j-2}; \bz_{i_{j}}, \bz_{i_{j-1}})+f(\bwa; \bz_{i_{j}}, \bz_{i_{j-1}})\big)
\leq 2B\Big(2\sum_{j=1}^{t}\eta_{j}^2\log(2/\delta)\Big)^{\frac{1}{2}}.
\end{multline*}
We can combine the above inequality and Eq. \eqref{opt-convex} to derive the following inequality with probability at least $1-\delta$
\[
\sum_{j=1}^{t}\eta_j\big(F_S(\bw_{j-2})-F_S(\bwa)\big)\leq 2B\Big(2\sum_{j=1}^{t}\eta_{j}^2\log(2/\delta)\Big)^{\frac{1}{2}}+
 \frac{1}{2}\|\bw_{{0}} -  \bwa\|_2^2  + G^2 \sum_{j=1}^t\big(\eta_{j-1}\eta_j+ \frac{\eta^2_j}{{2}}\big).
\]
The stated bound then follows from the convexity of $F_S$.

We now turn to the strongly convex case. Let
\[
\xi_j=2j(f(\bw_{2j-2}; \bz_{i_{2j}}, \bz_{i_{2j-1}}) - f(\bw_S; \bz_{i_{2j}}, \bz_{i_{2j-1}}) ),\quad j\in[t/2].
\]
It is clear that $|\xi_j-\ebb_{i_{2j}, i_{2j-1}}[\xi_j]|\leq 4jB\leq 2tB$ for $j\in[t/2]$. Furthermore, the conditional variance satisfies
\begin{align*}
\ebb_{\tilde{z}_j}\big[\big(\xi_j-\ebb_{\tilde{z}_j}[\xi_j]\big)^2\big]&\leq \ebb_{\tilde{z}_j}[\xi_j^2]\leq 4j^2G^2\|\bw_{2j-2}-\bw_S\|_2^2\\
& \leq 8\alpha^{-1}j^2G^2\big(F_S(\bw_{2j-2})-F_S(\bw_S)\big),
\end{align*}
where the first inequality follows from $f$ is $G$-Lipschitz continuous, and the second inequality used the fact $\nabla F_S(\bw_S)=0$ and $f$ is $\alpha$-strongly convex. 

Note $\tilde{z}_j$ are independent random variables and
\[
\ebb_{\tilde{z}_j}[\xi_j]=2j\big(F_S(\bw_{2j-2})-F_S(\bw_S)\big).
\]
Therefore, we can apply Part (b) of Lemma \ref{lem:martingale} to derive the following inequality with probability at least $1-\delta/2$
\begin{multline*}
 2\sum_{j=1}^{t/2}j\Big(F_S(\bw_{2j-2})-F_S(\bw_S)-f(\bw_{2j-2}; \bz_{i_{2j}}, \bz_{i_{2j-1}}) + f(\bw_S; \bz_{i_{2j}}, \bz_{i_{2j-1}})\Big) \\
 \leq \frac{8G^2\rho\sum_{j=1}^{t/2}j^2\big(F_S(\bw_{2j-2})-F_S(\bw_S)\big)}{2tB\alpha}+\frac{2tB\log(2/\delta)}{\rho}.\end{multline*}
In a similar way, one can derive the following inequality with probability at least $1-\delta/2$
\begin{multline*}
 \sum_{j=1}^{t/2}(2j-1)\Big(F_S(\bw_{2j-3})-F_S(\bw_S)-f(\bw_{2j-3}; \bz_{i_{2j-1}}, \bz_{i_{2j-2}}) + f(\bw_S; \bz_{i_{2j-1}}, \bz_{i_{2j-2}})\Big) \\
 \leq \frac{2G^2\rho\sum_{j=1}^{t/2}(2j-1)^2\big(F_S(\bw_{2j-3})-F_S(\bw_S)\big)}{2tB\alpha}+\frac{2tB\log(2/\delta)}{\rho}.
\end{multline*}
We can combine the above two inequalities together and derive the following inequality with probability at least $1-\delta$
\begin{multline*}
 \sum_{j=1}^{t}j\Big(F_S(\bw_{j-2})-F_S(\bw_S)-f(\bw_{j-2}; \bz_{i_{j}}, \bz_{i_{j-1}}) + f(\bw_S; \bz_{i_{j}}, \bz_{i_{j-1}})\Big) \\
 \leq \frac{2G^2\rho\sum_{j=1}^{t}j^2\big(F_S(\bw_{j-2})-F_S(\bw_S)\big)}{2tB\alpha}+\frac{2tB\log(2/\delta)}{\rho}.
\end{multline*}
We can combine the above inequality and Eq. \eqref{opt-sc} together and derive the following inequality with probability $1-\delta$
\[
 \sum_{j=1}^{t}j\big(F_S(\bw_{j-2})-F_S(\bw_S)\big) \leq \frac{3G^2t}{\alpha}+
  \frac{G^2\rho\sum_{j=1}^{t}j\big(F_S(\bw_{j-2})-F_S(\bw_S)\big)}{B\alpha}+\frac{2tB\log(2/\delta)}{\rho}.
\]
Now, we take $\rho=\min\big\{1,B\alpha/(2G^2)\big\}$ and get the following inequality with probability at least $1-\delta$
\[
\sum_{j=1}^{t}j\big(F_S(\bw_{j-2})-F_S(\bw_S)\big) \leq \frac{3G^2t}{\alpha} + \frac{1}{2}\sum_{j=1}^{t}j\big(F_S(\bw_{j-2})-F_S(\bw_S)\big) + 2t\log(2/\delta)\max\big\{B,2G^2/\alpha\big\}
\]
and therefore
\[
\sum_{j=1}^{t}j\big(F_S(\bw_{j-2})-F_S(\bw_S)\big)\leq \frac{14G^2t\log(2/\delta)}{\alpha}+4Bt\log(2/\delta).
\]
The stated bound then follows from the convexity of $F_S$. The proof is completed.
\end{proof}

We now turn to nonconvex problems.
\begin{proof}[{\bf Proof of Theorem \ref{thm:opt-nc}}]
  It is clear that $F_S$ is $L$-smooth and therefore
  \begin{align*}
  F_S(\bw_j) & \leq F_S(\bw_{j-1})+\langle\bw_j-\bw_{j-1},\nabla F_S(\bw_{j-1})\rangle+\frac{L}{2}\|\bw_j-\bw_{j-1}\|_2^2\\
  & = F_S(\bw_{j-1})-\eta_j\langle\nabla f(\bw_{j-1};\bz_{i_j},\bz_{i_{j-1}}),\nabla F_S(\bw_{j-1})\rangle+\frac{L\eta_j^2}{2}\|\nabla f(\bw_{j-1};\bz_{i_j},\bz_{i_{j-1}})\|_2^2
  \end{align*}
  Taking expectations over both sides gives
  \begin{multline}\label{opt-ns-1}
  \ebb_{\acal}[F_S(\bw_j)] \leq \ebb_{\acal}[F_S(\bw_{j-1})]-\eta_j\ebb_{\acal}\big[\langle\nabla f(\bw_{j-1};\bz_{i_j},\bz_{i_{j-1}}),\nabla F_S(\bw_{j-1})\rangle\big]+\\
  \frac{L\eta_j^2}{2}\ebb_{\acal}\big[\|\nabla f(\bw_{j-1};\bz_{i_j},\bz_{i_{j-1}})\|_2^2\big].
  \end{multline}
  According to the elementary inequality $(a+b)^2\leq2(a^2+b^2)$ we know
  \begin{align*}
    & \ebb_{\acal}\big[\|\nabla f(\bw_{j-1};\bz_{i_j},\bz_{i_{j-1}})\|_2^2\big] \\
    & \leq 2\ebb_{\acal}\big[\|\nabla f(\bw_{j-1};\bz_{i_j},\bz_{i_{j-1}})-\nabla f(\bw_{j-2};\bz_{i_j},\bz_{i_{j-1}})\|_2^2\big]+2\ebb_{\acal}\big[\|\nabla f(\bw_{j-2};\bz_{i_j},\bz_{i_{j-1}})\|_2^2\big]\\
    & \leq 2L\ebb_{\acal}\big[\|\bw_{j-1}-\bw_{j-2}\|_2^2\big]+2\alpha_0^2
    = 2L\eta_{j-1}^2\ebb_{\acal}\big[\|\nabla f(\bw_{j-2};\bz_{i_{j-1}},\bz_{i_{j-2}})\|_2^2\big]+2\alpha_0^2\\
    & \leq \frac{1}{2}\ebb_{\acal}\big[\|\nabla f(\bw_{j-2};\bz_{i_{j-1}},\bz_{i_{j-2}})\|_2^2\big]+2\alpha_0^2,
  \end{align*}
  where we have used the $L$-smoothness, the assumption $\ebb_{i_j,i_{j-1}}\big[\|\nabla f(\bw_{j-2};\bz_{i_j},\bz_{i_{j-1}})\|_2^2\big]\leq\alpha_0^2$ and $4L\eta_{j-1}^2\leq1$.
  It is clear that $\ebb_{\acal}\big[\|\nabla f(\bw_{0};\bz_{i_1},\bz_{i_{0}})\|_2^2\big]\leq\alpha_0^2$. It is easy to use an induction and the above inequality to show that
  \begin{equation}\label{opt-ns-2}
    \ebb_{\acal}\big[\|\nabla f(\bw_{j-1};\bz_{i_j},\bz_{i_{j-1}})\|_2^2\big] \leq 4\alpha_0^2,\quad\forall j.
  \end{equation}
  Furthermore, the smoothness assumption implies that
  \begin{align*}
    & \big\langle\nabla f(\bw_{j-1};\bz_{i_j},\bz_{i_{j-1}}), \nabla F_S(\bw_{j-1})\big\rangle \\
    & =
    \big\langle\nabla f(\bw_{j-2};\bz_{i_j},\bz_{i_{j-1}}), \nabla F_S(\bw_{j-1})\big\rangle + \big\langle\nabla f(\bw_{j-1};\bz_{i_j},\bz_{i_{j-1}})-\nabla f(\bw_{j-2};\bz_{i_j},\bz_{i_{j-1}}), \nabla F_S(\bw_{j-1})\big\rangle\\
    & = \big\langle\nabla f(\bw_{j-2};\bz_{i_j},\bz_{i_{j-1}}), \nabla F_S(\bw_{j-2})\big\rangle + \big\langle\nabla f(\bw_{j-2};\bz_{i_j},\bz_{i_{j-1}}), \nabla F_S(\bw_{j-1})-\nabla F_S(\bw_{j-2})\big\rangle + \\
    & + \big\langle\nabla f(\bw_{j-1};\bz_{i_j},\bz_{i_{j-1}})-\nabla f(\bw_{j-2};\bz_{i_j},\bz_{i_{j-1}}), \nabla F_S(\bw_{j-1})\big\rangle\\
    & \geq \big\langle\nabla f(\bw_{j-2};\bz_{i_j},\bz_{i_{j-1}}), \nabla F_S(\bw_{j-2})\big\rangle - L\|\bw_{j-1}-\bw_{j-2}\|_2\big(\|\nabla f(\bw_{j-2};\bz_{i_j},\bz_{i_{j-1}})\|_2 + \|\nabla F_S(\bw_{j-1})\|_2\big).
  \end{align*}
  According to Schwartz inequality, the variance assumption and Eq. \eqref{opt-ns-2}, we know
  \begin{align*}
    & \ebb_{\acal}\Big[\|\bw_{j-1}-\bw_{j-2}\|_2\big(\|\nabla f(\bw_{j-2};\bz_{i_j},\bz_{i_{j-1}})\|_2 + \|\nabla F_S(\bw_{j-1})\|_2\big)\Big]\\
    & \leq \frac{1}{2\eta_{j-1}}\ebb_{\acal}\big[\|\bw_{j-1}-\bw_{j-2}\|_2^2\big] + \eta_{j-1}\ebb_{\acal}\big[\|\nabla f(\bw_{j-2};\bz_{i_j},\bz_{i_{j-1}})\|_2^2+\|\nabla F_S(\bw_{j-1})\|_2^2\big]\\
    & = \frac{\eta_{j-1}}{2}\ebb_{\acal}\big[\|\nabla f(\bw_{j-2};\bz_{i_{j-1}},\bz_{i_{j-2}})\|_2^2\big]
    + \eta_{j-1}\ebb_{\acal}\big[\|\nabla f(\bw_{j-2};\bz_{i_j},\bz_{i_{j-1}})\|_2^2+\|\nabla F_S(\bw_{j-1})\|_2^2\big]\\
    & \leq 2\alpha_0^2\eta_{j-1}+2\alpha_0^2\eta_{j-1}.
  \end{align*}
  We can combine the above two inequalities together and get
  \[
  \ebb_{\acal}\big[\big\langle\nabla f(\bw_{j-1};\bz_{i_j},\bz_{i_{j-1}}), \nabla F_S(\bw_{j-1})\big\rangle\big]
  \geq \ebb_{\acal}\big[\big\langle\nabla f(\bw_{j-2};\bz_{i_j},\bz_{i_{j-1}}), \nabla F_S(\bw_{j-2})\big\rangle\big]-4L\alpha_0^2\eta_{j-1}.
  \]
  We can combine \eqref{opt-ns-1}, \eqref{opt-ns-2} and the above inequality, and get
  \begin{align*}
    \ebb_{\acal}[F_S(\bw_j)] &\leq \ebb_{\acal}[F_S(\bw_{j-1})]-\eta_j\ebb_{\acal}\big[\big\langle\nabla f(\bw_{j-2};\bz_{i_j},\bz_{i_{j-1}}), \nabla F_S(\bw_{j-2})\big\rangle\big]+4L\alpha_0^2\eta_j\eta_{j-1}+2L\eta_j^2\alpha_0^2\\
    & \le
    \ebb_{\acal}[F_S(\bw_{j-1})]-\eta_j\ebb_{\acal}\big[\|\nabla F_S(\bw_{j-2})\|_2^2\big]+4L\alpha_0^2\big(\eta_j\eta_{j-1}+\eta_j^2\big),
  \end{align*}
  where the last inequality holds since $\bw_{j-2}$ is independent of $i_{j}$ and $i_{j-1}$.
  The above inequality can be reformulated as
  \begin{equation}\label{opt-ns-3}
  \eta_j\ebb_{\acal}\big[\|\nabla F_S(\bw_{j-2})\big\|_2^2\big]\leq \ebb_{\acal}[F_S(\bw_{j-1})]-\ebb_{\acal}[F_S(\bw_j)]+4L\alpha_0^2\big(\eta_j\eta_{j-1}+\eta_j^2\big).
  \end{equation}
  We can take a summation of the above inequality and get
  \[
  \sum_{j=1}^{t}\eta_j\ebb_{\acal}\big[\|\nabla F_S(\bw_{j-2})\|_2^2\big] \leq F_S(\bw_{0}) + 4L\alpha_0^2\sum_{j=1}^{t}\big(\eta_j\eta_{j-1}+\eta_j^2\big).
  \]
  Since $\eta_j=\eta$, we further get
  \[
  \sum_{j=1}^{t}\ebb_{\acal}\big[\|\nabla F_S(\bw_{j-2})\|_2^2\big] \leq \eta^{-1}F_S(\bw_{0}) + 8L\alpha_0^2t\eta.
  \]
  The proof is completed.
\end{proof}

\begin{proof}[{\bf Proof of Theorem \ref{thm:opt-pl}}]
  According to the elementary inequality $\frac{1}{2}(a+b)^2\leq a^2+b^2$ we know
  \begin{align*}
    \ebb_{\acal}\big[\|\nabla F_S(\bw_{j-2})\big\|_2^2\big] & \geq - \ebb_{\acal}\big[\|\nabla F_S(\bw_{j-2})-\nabla F_S(\bw_{j-1})\|_2^2\big] + 2^{-1}\ebb_{\acal}\big[\|\nabla F_S(\bw_{j-1})\|_2^2\big]\\
    & \geq -L \ebb_{\acal}\big[\|\bw_{j-2}-\bw_{j-1}\|_2^2\big] + 2^{-1}\ebb_{\acal}\big[\|\nabla F_S(\bw_{j-1})\|_2^2\big]\\
    & = -L\eta_{j-1}^2 \ebb_{\acal}\big[\|\nabla f(\bw_{j-2};\bz_{i_{j-1}},\bz_{i_{j-2}})\|_2^2\big] + 2^{-1}\ebb_{\acal}\big[\|\nabla F_S(\bw_{j-1})\|_2^2\big]\\
    & \geq -4L\eta_{j-1}^2\alpha_0^2+ 2^{-1}\ebb_{\acal}\big[\|\nabla F_S(\bw_{j-1})\|_2^2\big],
  \end{align*}
  where we have used \eqref{opt-ns-2}. This together with \eqref{opt-ns-3} gives
  \[
  2^{-1}\eta_j\ebb_{\acal}\big[\|\nabla F_S(\bw_{j-1})\big\|_2^2\big]\leq 4L\eta_j\eta_{j-1}^2\alpha_0^2+\ebb_{\acal}[F_S(\bw_{j-1})]-\ebb_{\acal}[F_S(\bw_j)]+4L\alpha_0^2\big(\eta_j\eta_{j-1}+\eta_j^2\big).
  \]
  It then follows from the PL condition that
  \[
  \mu\eta_j\ebb_{\acal}\big[F_S(\bw_{j-1})-F_S(\bw_S)\big]\leq \ebb_{\acal}[F_S(\bw_{j-1})]-\ebb_{\acal}[F_S(\bw_j)]+4L\alpha_0^2\big(\eta_j\eta_{j-1}+\eta_j^2+\eta_j\eta_{j-1}^2\big).
  \]
  We can reformulate the above inequality as
  \[
  \ebb_{\acal}[F_S(\bw_j)-F_S(\bw_S)] \leq (1-\mu\eta_j)\ebb_{\acal}[F_S(\bw_{j-1})-F_S(\bw_S)]+4L\alpha_0^2\big(\eta_j\eta_{j-1}+\eta_j^2+\eta_j\eta_{j-1}^2\big).
  \]
  Now, taking $\eta_j=2/(\mu(j+1))$, we get
  \[
  \ebb_{\acal}[F_S(\bw_j)-F_S(\bw_S)] \leq \frac{j-1}{j+1}\ebb_{\acal}[F_S(\bw_{j-1})-F_S(\bw_S)]+\frac{4L\alpha_0^2}{\mu^2}\Big(\frac{8}{j(j+1)}+\frac{8}{j^2(j+1)\mu}\big).
  \]
  We can multiple both sides by $j(j+1)$ and get
  \[
  j(j+1)\ebb_{\acal}[F_S(\bw_j)-F_S(\bw_S)] \leq (j-1)j\ebb_{\acal}[F_S(\bw_{j-1})-F_S(\bw_S)]+\frac{4L\alpha_0^2}{\mu^2}\big(8+8j^{-1}\mu^{-1}\big).
  \]
  Taking a summation of the above inequality gives
  \[
  t(t+1)\ebb_{\acal}[F_S(\bw_t)-F_S(\bw_S)] \leq \frac{32L\alpha_0^2}{\mu^2}\sum_{j=1}^{t}\big(1+j^{-1}\mu^{-1}\big).
  \]
  The stated bound then follows. The proof is completed.
\end{proof}

\section{Proofs of Excess Generalization Error (Theorems 1-3)\label{sec:proof-gen}}
In this section, we prove Theorem \ref{thm:excess-nonsmooth},  Theorem \ref{thm:excess-smooth} and Theorem \ref{thm:exg-pl} on excess generalization error bounds.
\begin{proof}[{\bf Proof of Theorem \ref{thm:excess-nonsmooth}}]
According to Theorem \ref{lem:est-nonsmooth}, we know
\[
\ebb_{S,\acal}[F(\bar{\wbf}_T) - F_S(\bar{\wbf}_T)]= \O\Big(\sqrt{T}\eta+\frac{T\eta}{n}\Big).
\]
Furthermore, according to Part (a) of Theorem \ref{thm:convergence} with $\bw=\bw^*$, we know
\begin{equation}\label{opt-nonsmooth}
\ebb_{S,\acal}[F_S(\bar{\wbf}_T)-F_S(\bw^*)]=\O\Big(\frac{1+T\eta^2}{T\eta}\Big).
\end{equation}
We can plug the above generalization error bound and optimization error bound back into the error decomposition \eqref{err-dec}, and get \eqref{excess-nonsmooth}.
Taking $T\asymp n^2$ and $\eta\asymp T^{-\frac{3}{4}}$ in Eq. \eqref{excess-nonsmooth}, we immediately get $\ebb_{S,\acal}[F(\bar{\wbf}_T)]-F(\bw^*)=\O(1/\sqrt{n})$. The desired result is proved.
\end{proof}
\begin{proof}[{\bf Proof of Theorem \ref{thm:excess-smooth}}]
According to Theorem \ref{lem:est-smooth}, we know
\[
\ebb_{S,\acal}[F(\bar{\wbf}_T) - F_S(\bar{\wbf}_T)]=\O\Big(\frac{T\eta}{n}\Big).
\]
We can plug the above generalization error bound and the optimization error bound \eqref{opt-nonsmooth} back into the error decomposition \eqref{err-dec}, and get \eqref{excess-smooth}. Taking $T\asymp n^2$ and $\eta\asymp T^{-\frac{3}{4}}$ in Eq. \eqref{excess-smooth}, we immediately get $\ebb_{S,\acal}[F(\bar{\wbf}_T)]-F(\bw^*)=\O(1/\sqrt{n})$. The proof is completed.
\end{proof}

\begin{proof}[{\bf Proof of Theorem \ref{thm:exg-pl}}]
Let $S'=\{z_1,\ldots,\bz_{n-1},\bz_{n}'\}$, where $z_n'$ is independently drawn from $\rho$.
Let $\{\bw_t'\}$ be produced by Algorithm \ref{alg:markov} w.r.t. $S'$.
If $i_t\neq n$ and $i_{t-1}\neq n$, then
\begin{align*}
  \|\bw_t-\bw_t'\|_2 & \leq \big\|\bw_{t-1}-\eta_t\nabla f(\bw_{t-1};\bz_{i_t},\bz_{i_{t-1}})-\bw_{t-1}'-\eta_t\nabla f(\bw_{t-1}';\bz_{i_t},\bz_{i_{t-1}})\big\|_2 \\
  & \leq \|\bw_{t-1}-\bw_{t-1}'\|_2 + \eta_t\|\nabla f(\bw_{t-1};\bz_{i_t},\bz_{i_{t-1}})-\nabla f(\bw_{t-1}';\bz_{i_t},\bz_{i_{t-1}})\|_2\\
  & \leq (1+L\eta_t)\|\bw_{t-1}-\bw_{t-1}'\|_2,
\end{align*}
where in the last inequality we used the smoothness of $f$. 

  Otherwise, it follows from the Lipschitz condition that
$
  \|\bw_t-\bw_t'\|_2\leq \|\bw_{t-1}-\bw_{t-1}'\|_2+2G\eta_t.$
  Consequently, it  follows that
  \begin{align*}
  \|\bw_t-\bw_t'\|_2 &\leq (1+L\eta_t)\|\bw_{t-1}-\bw_{t-1}'\|_2\ibb_{[i_t\neq n\text{ and }i_{t-1}\neq n]} \\ & + \big(\|\bw_{t-1}-\bw_{t-1}'\|_2+2G\eta_t\big)\ibb_{[i_t= n\text{ or }i_{t-1}= n]}\\
  & \leq (1+L\eta_t)\|\bw_{t-1}-\bw_{t-1}'\|_2 + 2G\eta_t\ibb_{[i_t= n\text{ or }i_{t-1}= n]}.
  \end{align*}
  We can apply the above inequality recursively and get
  \[
  \triangle_t \leq 2G\sum_{k=t_0+1}^{t}\eta_k\ibb_{[i_k= n\text{ or }i_{k-1}= n]}\prod_{k'=k+1}^{t}(1+L\eta_{k'}) + \triangle_{t_0}\prod_{k=t_0+1}^{t}(1+L\eta_k).
  \]
  Since $\triangle_{t_0}=0$ implies $i_{t_0}\neq n$, we have
  \begin{align*}
    \ebb[\triangle_t|\triangle_{t_0}=0] & \leq 2G\sum_{k=t_0+1}^{t}\eta_k\ebb\big[\ibb_{[i_k= n\text{ or }i_{k-1}= n]}|\triangle_{t_0}=0\big]\prod_{k'=k+1}^{t}(1+L\eta_{k'}) \\
     & = 2G\sum_{k=t_0+2}^{t}\eta_k\ebb\big[\ibb_{[i_k= n\text{ or }i_{k-1}= n]}|\triangle_{t_0}=0\big]\prod_{k'=k+1}^{t}(1+L\eta_{k'}) \\
     & = 2G\sum_{k=t_0+2}^{t}\eta_k\ebb\big[\ibb_{[i_k= n\text{ or }i_{k-1}= n]}\big]\prod_{k'=k+1}^{t}(1+L\eta_{k'}),
  \end{align*}
  where we have used the independency between $\triangle_{t_0}$ and $i_t$ for $t>t_0$.
  It then follows that
  \begin{align*}
    \ebb[\triangle_t|\triangle_{t_0}=0] & \leq 2G\sum_{k=t_0+2}^{t}\eta_k\ebb\big[\ibb_{[i_k= n]}+\ibb_{[i_{k-1}= n]}\big]\prod_{k'=k+1}^{t}(1+L\eta_{k'}) \\
     & \leq \frac{4G}{n}\sum_{k=t_0+2}^{t}\eta_k\prod_{k'=k+1}^{t}\exp(L\eta_{k'})\le \frac{8G}{\mu n}\sum_{k=t_0+2}^{t}\frac{1}{k+1}\exp\Big(2L\mu^{-1}\sum_{k'=k+1}^{t}\frac{1}{k'+1}\Big) \\
     & \leq \frac{8G}{\mu n}\sum_{k=t_0+2}^{t}\frac{1}{k+1}\exp\Big(2L\mu^{-1}\log(t/k)\Big) 
     = \frac{8G}{\mu n}\sum_{k=t_0+2}^{t}\frac{1}{k+1}\Big(\frac{t}{k}\Big)^{2L\mu^{-1}}\\
     & \le  \frac{8Gt^{2L\mu^{-1}}}{\mu n}\sum_{k=t_0+2}^{t}k^{-1-2L\mu^{-1}} \leq \frac{8G}{\mu n(2L\mu^{-1})}\Big(\frac{t}{t_0}\Big)^{2L\mu^{-1}}.
  \end{align*}
  Here we use $1+x\le \exp(x)$ and $\eta_j=\frac{2}{\mu (j+1)}$. 
  We can plug the above inequality back into Lemma \ref{lem:stab-cond} and derive
  \[
  \ebb\big[|f(\bw_T;z,z')-f(\bw_T';z,z')|\big] \leq \frac{4G^2}{nL}\Big(\frac{T}{t_0}\Big)^{2L\mu^{-1}}+\frac{Bt_0}{n}.
  \]
  We can choose $t_0\asymp T^{\frac{2L}{2L+\mu}}$ and get the following generalization error bounds
\[
\ebb\big[|f(\bw_T;z,z')-f(\bw_T';z,z')|\big]=\O\Big(\frac{T^{\frac{2L}{2L+\mu}}}{n}\Big).
\]
Lemma \ref{lem:generalization-via-stability} then implies $\ebb\big[F(\bw_T)-F_S(\bw_T)\big]=\O\Big(\frac{T^{\frac{2L}{2L+\mu}}}{n}\Big)$.
Furthermore, according to Theorem \ref{thm:opt-pl} we have the following optimization error bounds
\[
\ebb_A[F_S(\bw_T)]-\inf_{\bw}[F_S(\bw)]=\O\big(1/(T\mu^2)\big).
\]
The desired result follows by combining the above two inequalities together and using the fact $\ebb[\inf_{\bw}[F_S(\bw)]\leq \ebb[F_S(\bw^*)]=F(\bw^*)$.
\end{proof}

\section{Proof of Privacy and Utility Guarantees of Algorithm \ref{alg:dp-iterative-localization-2} (Theorem \ref{thm:privacy-utility-2})\label{sec:private-proof}}

In this section, we present the proof of Theorem \ref{thm:privacy-utility-2} on the  privacy guarantee and excess generalization bound  of Algorithm \ref{alg:dp-iterative-localization-2}.
 
To this end, we need the definition of $\ell_2$-sensitivity in terms of high probability and some lemmas. The $\ell_2$-sensitivity definition given below corresponds to the high probability version of uniform argument stability stated in Definition \ref{def:uniform-stability}.

\begin{definition}
For any $\gamma \in (0,1)$, a (randomized) algorithm $\acal$ has $\ell_2$-sensitivity of $\Delta$ with probability at least $1 - \gamma$ if for any neighboring datasets $S, S' \in \zcal^n$, one has $\|\acal(S) - \acal(S')\|_2 \leq \Delta$.
\end{definition}

The next lemma demonstrates that Gaussian mechanism ensures the privacy of an algorithm with high probability $\ell_2$-sensitivity.  

\begin{lemma}\label{lem:high-probability-gaussian-privacy}
Let $\acal:\zcal^n \rightarrow \rbb^d$ be a randomized algorithm with $\ell_2$-sensitivity of $\Delta$ with probability at least $1-\delta/2$. Then the Gaussian mechanism $\mcal(S) = \acal(S) + \ubf$ where $\ubf \sim \ncal(0, (2\Delta^2\log(2.5/\delta)/\epsilon^2) I_d)$ satisfies $(\epsilon, \delta)$-DP.
\end{lemma}
\begin{proof}
Let $S$ and $S'$ be two neighboring datasets. 
Denote $E$ as the set when $\acal$ satisfies $\ell_2$-sensitivity of $\Delta$, i.e. $E = \{\|\acal(S) - \acal(S')\|_2 \leq \Delta\}$. Then we know $\pbb[E] \geq 1 - \delta/2$. In favor of $E$, by classical results for Gaussian mechanism, we know $\mcal$ satisfies $(\epsilon, \delta/2)$-DP with $\sigma = \Delta\sqrt{2\log(2.5/\delta)}/\epsilon$. Therefore, for any $\epsilon > 0$ and any event $O$ in the output space of $\mcal$, we have
\begin{align*}
\pbb[\mcal(S) \in O] = & \pbb[\mcal(S) \in O|E]\pbb[E] + \pbb[\mcal(S) \in O|E^c]\pbb[E^c]\\
\leq & \Big(e^\epsilon \pbb[\mcal(S') \in O|E] + \frac{\delta}{2}\Big)\pbb[E] + \frac{\delta}{2}\\
\leq & e^\epsilon \pbb[\mcal(S') \in O \cap E] + \frac{\delta}{2} + \frac{\delta}{2}\\
\leq & e^\epsilon \pbb[\mcal(S') \in O] + \delta
\end{align*}
where the first inequality is because $\mcal$ satisfies $(\epsilon, \delta/2)$-DP when $E$ occurs and $\pbb[E^c] \leq \delta/2$, the second and third inequalities are by basic properties of probability. The proof is completed.
\end{proof}





We need the following Chernoff’s bound for a summation of independent Bernoulli random variables \citep{yang2021stability}.

\begin{lemma}[Chernoff bound for Bernoulli vector]\label{lem:chernoff-bernoulli}
Let $X_1,\ldots,X_t$ be independent random variables taking values in $\{0,1\}$. Let $X=\sum_{j=1}^{t}X_j$ and $\mu=\EX[X]$. Then for any $\tilde{\gamma} > 0$, with probability at least $1-\exp\big(-\mu\tilde{\gamma}^2/(2+\tilde{\gamma})\big)$ we have $X\leq (1+\tilde{\gamma})\mu$.
\end{lemma} 

In order to prove the privacy guarantee and excess generalization bound  for Algorithm \ref{alg:dp-iterative-localization-2}, we also need the following high probability $\ell_2$-sensitivity of the output of Algorithm \ref{alg:markov}. 

\begin{lemma}\label{lem:high-prob-sensitivity}
Let $\{\bar{\wbf}_t\}$ and $\{\bar{\wbf}'_t\}$ be produced by Algorithm \ref{alg:markov} based on the neighboring datasets $S$ and $S'$, respectively. If $f$ is convex and $L$-smooth and $\eta_t = \eta \leq 2/L$, then with probability at least $1 - \gamma$ we have
\[
\|\bar{\wbf}_t-\bar{\wbf}_t'\big\|_2\leq 4G\eta\Big(\frac{t}{n}+ \log(2/\gamma) + \sqrt{\frac{t\log(2/\gamma)}{n}}\Big).
\]
\end{lemma}

\begin{proof}
Without loss of generality, we assume the different example between $S$ and $S'$ is the $n$-th item. By the proof of Theorem \ref{lem:est-smooth}, we know
\[
\big\|\bw_t-\bw_t'\big\|_2\leq2G\sum_{j=1}^{t}\eta_j\ibb_{[i_j= n\text{ or }i_{j-1}= n]}\leq 2G\sum_{j=1}^{t}\eta_j\big(\ibb_{[i_j= n]}+\ibb_{[i_{j-1}= n]}).
\]
Applying Lemma~\ref{lem:chernoff-bernoulli}, with probability at least $1-\gamma$ there holds
\[
\sum_{j=1}^{t}(\ibb_{[i_j=n]}+\ibb_{[i_{j-1}= n]}) \leq \frac{2t}{n}+ 2\log(2/\gamma) + 2\sqrt{\frac{t\log(2/\gamma)}{n}}.
\]
It then follows from the convexity of $\|\cdot\|_2$ that
\[
\|\bar{\wbf}_t-\bar{\wbf}_t'\big\|_2\leq 4G\eta\Big(\frac{t}{n}+ \log(2/\gamma) + \sqrt{\frac{t\log(2/\gamma)}{n}}\Big),
\]
which implies the desired result. 
\end{proof}
With the above preparations, we are now ready to prove Theorem \ref{thm:privacy-utility-2}.  

\begin{proof}[{\bf Proof of Theorem \ref{thm:privacy-utility-2}}]
We first consider the privacy guarantee of Algorithm \ref{alg:dp-iterative-localization-2}. Since we run Algorithm \ref{alg:markov} for $\lceil n_k\log(4/\delta) \rceil$ steps for each $k$, by Lemma \ref{lem:high-prob-sensitivity}, we know with probability $1 - \delta/2$
\[
\|\bar{\wbf}_k-\bar{\wbf}_k'\big\|_2\leq 12G\eta_k\log(4/\delta) .
\]
Therefore, by Lemma \ref{lem:high-probability-gaussian-privacy}, each iteration $k$ of Algorithm \ref{alg:dp-iterative-localization-2} is $(\epsilon, \delta)$-DP. Since the partition of the dataset $S$ is disjoint, and each iteration $k$ of Algorithm \ref{alg:dp-iterative-localization-2} we only use one subset, thus the whole process satisfies  $(\epsilon,\delta)$-DP.

Next we investigate the utility bound of Algorithm \ref{alg:dp-iterative-localization-2}. Let $\bar{\wbf}_0 = \wbf^*$ and $\ubf_0= \wbf_0- \wbf^*$, then 
\begin{align*}\label{eq:dp-sco-decomposition-2}
\ebb[F(\wbf_K) - F(\wbf^*)] = \sum_{k=1}^K \ebb[F(\bar{\wbf}_k) - F(\bar{\wbf}_{k-1})] + \ebb[F(\wbf_K) - F(\bar{\wbf}_K)] \numberthis
\end{align*}
Denote $F_{S_k}$ be the empirical objective based on sample $S_k$. For the first term on the RHS of \eqref{eq:dp-sco-decomposition-2}, we have 
\begin{align*}
& \ebb[F(\bar{\wbf}_k) - F(\bar{\wbf}_{k-1})]\\
&=  \ebb[F(\bar{\wbf}_k) - F_{S_k}(\bar{\wbf}_k)] + \ebb[F_{S_k}(\bar{\wbf}_k) - F_{S_k}(\bar{\wbf}_{k-1})] + \ebb[F_{S_k}(\bar{\wbf}_{k-1}) - F(\bar{\wbf}_{k-1})]\\
&=  \ebb[F(\bar{\wbf}_k) - F_{S_k}(\bar{\wbf}_k)] + \ebb[F_{S_k}(\bar{\wbf}_k) - F_{S_k}(\bar{\wbf}_{k-1})]\\
&\leq   8G^2\log(4/\delta)\eta_k + \Big(\frac{\ebb[\|\ubf_{k-1}\|_2^2]}{2\eta_k n_k} + \frac{3G^2\eta_k}{2}\Big) \le  \frac{\ebb[\|\ubf_{k-1}\|_2^2]}{2\eta_k n_k} +  18\log(4/\delta)  G^2\eta_k , 
\end{align*} 
where the second identity is because $\bar{\wbf}_{k-1}$ is independent of $S_k$ and the inequality follows from Theorem \ref{thm:convergence} Part (a) and Theorem \ref{lem:est-smooth}. Recall that by definition $\eta \leq \frac{D\epsilon}{12G\log(4/\delta)\sqrt{2d\log(2.5/\delta)}}$, so that for all $k \geq 0$, 
\[
\ebb[\|\ubf_k\|_2^2] = d\sigma_k^2 = d\Big(\frac{4^{-k}G\eta}{\epsilon}\Big)^2 \leq 16^{-k}D^2.
\]
Plugging the above estimate into \eqref{eq:dp-sco-decomposition-2} it follows
\begin{align*}
\ebb[F(\wbf_K) - F(\wbf^*)] \leq & 
\sum_{k=1}^K  \frac{8 \cdot 16^{-k} D^2}{ 4^{-k}  2^{-k} \eta  n   } + 18 \log(4/\delta) 4^{-k}  G^2\eta + 4^{-K} GD \\
\leq & \sum_{k=1}^K 2^{-k}\Big(\frac{8D^2}{\eta n} + 18 \log(4/\delta) G^2\eta\Big) + \frac{GD}{n^2}\\
= & \O\Big(GD \Big( \frac{1}{\sqrt{n}} + \frac{\sqrt{d} \log^{\frac{3}{2}}(1/\delta) }{\epsilon n} \Big)  \Big),
\end{align*}
where in the second inequality used $K=\lceil \log_2 n \rceil$, and the last inequality is due to $\eta = \frac{D}{G}\min\{\frac{\log(4/\delta)}{\sqrt{n}}, \frac{\epsilon}{12\log(4/\delta)\sqrt{2d\log(2.5/\delta)}}\}  $.  
The desired excess generalization error bound is proved.

Finally, we investigate the gradient complexity argument. Since we run Algorithm \ref{alg:markov} for $n_k$ at iteration $k$. Therefore, the total gradient complexity is  $\ocal\Big(\sum_{k=1}^K n_k\Big) = \ocal(n \log(1/\delta)).$  The proof is completed.
\end{proof}

\section{Additional Results: Localization-Based Algorithm to Improve Theorem 
\ref{thm:excess-nonsmooth}\label{sec:iterative-localization}}

In this section, we provide additional results on how to reduce the gradient complexity $\O(n^2)$ required in Theorem \ref{thm:excess-nonsmooth} to  $\O(n)$ for nonsmooth problems.  This improvement is attained by  Algorithm \ref{alg:iterative-localization} which is motivated by the iterative localization technique \citep{feldman2020private}.

\begin{algorithm}[ht!]
\caption{Localized SGD for Pairwise Learning\label{alg:iterative-localization}}
\begin{algorithmic}[1]
\STATE {\bf Inputs:} Dataset $S = \{\zbf_i: i=1,\ldots, n\}$, parameter $\zeta$, initial point $\wbf_0$
\STATE Set $K=\lceil\log_2 n\rceil$ and divide $S$ into $K$ disjoint subsets $\{S_1, \cdots, S_K\}$ such that $|S_k| = n_k =  2^{-k}n$
\FOR{$k=1$ to $K$}
\STATE Set $\zeta_k = 2^{-k}\zeta$
\STATE Compute $\bar{\wbf}_k \in \wcal$ by Algorithm \ref{alg:markov} with step sizes $\eta_j = \frac{\zeta_kn_k}{j+1}, j \in [T_k]$ and  $T_k \asymp n_k$ iterations based on the objective $F_k$ where \[F_k(\wbf; S_k) = \frac{1}{n_k(n_k-1)} \sum_{z,z'\in S_k:z\neq z'} f(\wbf; \zbf, \zbf') + \frac{1}{\zeta_kn_k}\|\wbf - \bar{\wbf}_{k-1}\|_2^2\] 
\ENDFOR
\STATE {\bf Outputs:} $\bar{\wbf}_K$
\end{algorithmic}
\end{algorithm}

The next theorem shows that the empirical risk minimization can imply models with good excess generalization error by Algorithm \ref{alg:iterative-localization}.

\begin{theorem}\label{thm:erm-to-sco}
Let (A1) and (A3) hold true with $\alpha=0$ and let $D$ be the diameter of   $\wcal$. Let $\{\bar{\wbf}_k: k \in [K]\}$ be produced by Algorithm \ref{alg:iterative-localization} with $\zeta = \frac{D}{G\sqrt{n}}$. Then we have the following excess generalization error bounds
\[
\ebb[F(\bar{\wbf}_K) - F(\wbf^*)] = \ocal\Big(\frac{GD}{\sqrt{n}}\Big)
\]
with gradient complexity $\ocal(n)$.
\end{theorem}


We provide two technical lemmas before we present the proof of Theorem \ref{thm:erm-to-sco}.
\begin{lemma}\label{lem:parameter-distance}
Let (A1) and (A3) hold true with $\alpha = 0$ and let $\hat{\wbf}_k = \arg\min_{\wbf} F_k(\wbf;S_k)$, then
\[
\ebb[\|\bar{\wbf}_k - \hat{\wbf}_k\|_2^2] = \ocal\big(G^2\zeta_k^2n_k\big).
\] 
\end{lemma}

\begin{proof}
Note that $F_k$ is $\alpha_k = \frac{2}{\zeta_k n_k}$-strongly convex, by the convergence of Algorithm \ref{alg:markov} in Theorem \ref{thm:convergence} Part (b), we know that
\[
\frac{\alpha_k}{2}\ebb[\|\bar{\wbf}_k - \hat{\wbf}_k\|_2^2] \leq \ebb[F_k(\bar{\wbf}_k; S_k) - F_k(\hat{\wbf}_k; S_k)] = \ocal\Big(\frac{G^2}{\alpha_k n_k}\Big)
\]
which implies
\[
\ebb[\|\bar{\wbf}_k - \hat{\wbf}_k\|_2^2] = \ocal\big(G^2\zeta_k^2n_k\big).
\]
The proof is completed.
\end{proof}

\begin{lemma}\label{lem:function-distante}
Let (A1) and (A3) hold true with $\alpha = 0$. For any $\wbf \in \wcal$, we know that
\[
\ebb[F(\hat{\wbf}_k) - F(\wbf)] \leq \frac{\ebb[\|\bar{\wbf}_{k-1} - \wbf\|_2^2]}{\zeta_k n_k} + 2G^2\zeta_k.
\]
\end{lemma}

\begin{proof}
Let $r(\wbf; \zbf, \zbf') = f(\wbf, \zbf, \zbf') + \frac{1}{\zeta_k n_k}\|\wbf - \bar{\wbf}_{k-1}\|_2^2, R(\wbf) = \ebb_{\zbf,\zbf'} [r(\wbf; \zbf, \zbf')]$ and $\wbf_R^* = \arg\min_{\wbf\in \wcal} R(\wbf)$. 
By the proof of Theorem 6  in \citet{shalev2009stochastic} , one has that 
\begin{align*}
&\ebb[F(\hat{\wbf}_k) + \frac{1}{\zeta_k n_k}\|\hat{\wbf}_k - \bar{\wbf}_{k-1}\|_2^2 - F(\wbf) - \frac{1}{\zeta_k n_k}\|\wbf - \bar{\wbf}_{k-1}\|_2^2] \\
&= \ebb[R(\hat{\wbf}_k) - R(\wbf)]  
\leq   \ebb[R(\hat{\wbf}_k) - R(\wbf_R^*)] \leq 2G^2\zeta_k,
\end{align*}
which implies that
\begin{align*}
\ebb[F(\hat{\wbf}_k) - F(\wbf)] \leq & 2G^2\zeta_k - \frac{1}{\zeta_k n_k}\ebb[\|\hat{\wbf}_k - \bar{\wbf}_{k-1}\|_2^2] + \frac{1}{\zeta_k n_k}\ebb[\|\wbf - \bar{\wbf}_{k-1}\|_2^2]\\
\leq & 2G^2\zeta_k + \frac{1}{\zeta_k n_k}\ebb[\|\wbf - \bar{\wbf}_{k-1}\|_2^2].
\end{align*}
The proof is completed.
\end{proof}

\begin{proof}[{\bf Proof of Theorem \ref{thm:erm-to-sco}}]
Let $\hat{\wbf}_0 = \wbf^*$, we have
\begin{align*}\label{eq:sco-decomposition}
\ebb[F(\bar{\wbf}_K)] - F(\wbf^*) = \sum_{k=1}^K \ebb[F(\hat{\wbf}_k) - F(\hat{\wbf}_{k-1})] + \ebb[F(\bar{\wbf}_K) - F(\hat{\wbf}_K)]. \numberthis
\end{align*}
For the first term we have
\begin{align*}\label{eq:sco-decomposition-1}
\sum_{k=1}^K \ebb[F(\hat{\wbf}_k) - F(\hat{\wbf}_{k-1})] \leq & \sum_{k=1}^K    \Big(\frac{\ebb[\|\bar{\wbf}_{k-1} - \hat{\wbf}_{k-1}\|_2^2]}{\zeta_k n_k} + 2G^2\zeta_k\Big)\\ 
= & \ocal\Big(\frac{D^2}{\zeta n}  + \sum_{k=2}^K G^2\zeta_k  + \sum_{k=1}^K 2^{-k} G^2 \zeta \Big)\\
= & \ocal\Big(\frac{D^2}{\zeta n} + G^2 \zeta\Big) \numberthis
\end{align*}
where the first inequality is by Lemma \ref{lem:function-distante}, the second inequality is by Lemma \ref{lem:parameter-distance} and $\zeta = \frac{D}{G\sqrt{n}}$. For the second term we have
\begin{align*}\label{eq:sco-decomposition-2}
\ebb[F(\bar{\wbf}_K) - F(\hat{\wbf}_K)] \leq &  G\ebb[\|\bar{\wbf}_K - \hat{\wbf}_K\|_2] \leq G\sqrt{\ebb[\|\bar{\wbf}_K - \hat{\wbf}_K\|_2^2]}   = \ocal\big(G^2\zeta_K\sqrt{n_K}\big)\\
= &  \ocal\Big(2^{-2K}G^2\zeta\sqrt{n}\Big) = \ocal\Big(G^2\zeta \Big) \numberthis
\end{align*}
where the first inequality is by $G$-Lipschitz continuity of $F$, the second inequality is by Jensen's inequality, the first identity is by Lemma \ref{lem:parameter-distance} and the second identity is by $n_k=2^{-k}n$.

Now putting \eqref{eq:sco-decomposition-1} and \eqref{eq:sco-decomposition-2} back to \eqref{eq:sco-decomposition} and using $\zeta = \frac{D}{G\sqrt{n}}$, we derive
\[
\ebb[F(\bar{\wbf}_K)] - F(\wbf^*) = \ocal\Big(\frac{GD}{\sqrt{n}}\Big).
\]

Finally we investigate the gradient complexity. Since $F_k$ is $\frac{2}{\zeta_kn_k}$-strongly convex, by Theorem \ref{thm:convergence} Part (b), we need to choose $T_k \asymp n_k$ so that Lemma \ref{lem:parameter-distance} holds. Therefore, in total, we require $
\ocal\Big(\sum_{k=1}^K n_k\Big) = \ocal(n)$ gradient complexity,  which yields the desired result. 
\end{proof}

\section{Additional Results: Differentially Private SGD for Pairwise Learning with Non-Smooth Losses\label{sec:private-2}}

In this section, we propose a differentially private algorithm based on iterative localization \citep{feldman2020private} for nonsmooth pairwise learning problems. The algorithm is presented as follows.

\begin{algorithm}[ht!]
\caption{Differentially Private Localized SGD for Pairwise Learning\label{alg:dp-iterative-localization}}
\begin{algorithmic}[1]
\STATE {\bf Inputs:} Dataset $S = \{\zbf_i: i\in[n]\}$, parameters $\epsilon, \delta$, and $\zeta$, initial points $\wbf_0$
\STATE Set $K\!=\!\lceil\log_2 n\rceil$ and divide $S$ into $K$ disjoint subsets $\{S_1, \cdots, S_K\}$ where $|S_k| \!=\! n_k \!=\! 2^{-k}n$.
\FOR{$k=1$ to $K$}
\STATE Set $\zeta_k = 4^{-k}\zeta$
\STATE Compute $\bar{\wbf}_k \in \wcal$ by Algorithm \ref{alg:markov} with step sizes $\eta_j = \frac{\zeta_kn_k}{j+1}$ on objective $F_k$ such that  with prob $1-\delta$, \[F_k(\bar{\wbf}_k; S_k) - \min_{\wbf \in \wcal}F_k(\wbf; S_k) \leq G^2\zeta_k/n_k\] where $F_k(\wbf; S_k) = \frac{1}{n_k(n_k-1)} \sum_{z,z'\in S_k:z\neq z'} f(\wbf; z,z') + \frac{1}{\zeta_kn_k}\|\wbf - \wbf_{k-1}\|_2^2$ 
\STATE Set $\wbf_k = \bar{\wbf}_k + \ubf_k$ where $\ubf_k \sim \ncal(0,\alpha_k^2I_d)$ with $\sigma_k = 4G\zeta_k\sqrt{\log(2.5/\delta)}/\epsilon$.
\ENDFOR
\STATE {\bf Outputs:} $\wbf_K$
\end{algorithmic}
\end{algorithm}

We are now ready to present the privacy guarantee and utility bound of Algorithm \ref{alg:dp-iterative-localization} in the following theorem. The proof differs from the iterative localization algorithm in pointwise learning \citep{feldman2020private} since we employ our high probability convergence results for non-smooth losses in pairwise learning.

\begin{theorem}\label{thm:privacy-utility}
Let (A1) and (A3) hold true with $\alpha=0$ and let $D$ be the diameter of $\wcal$. Let $\{\wbf_k: k \in [K]\}$ be produced by Algorithm \ref{alg:dp-iterative-localization} with $\zeta = \frac{D}{G}\min\{\frac{4}{\sqrt{n}}, \frac{\epsilon}{4\sqrt{d\log(1/\delta)}}\}$. Then Algorithm \ref{alg:dp-iterative-localization} satisfies $(\epsilon, \delta)$-DP. Furthermore we have the following excess generalization error bounds
\[
\ebb[F(\wbf_K) - F(\wbf^*)] = \ocal\Big(GD\Big(\frac{1}{\sqrt{n}} + \frac{\sqrt{d\log(1/\delta)}}{\epsilon n}\Big)\Big)
\]
with no more than $\ocal(n^2\log(1/\delta))$ stochastic gradient computations.
\end{theorem}

\begin{proof}[{\bf Proof of Theorem \ref{thm:privacy-utility}}]
We first consider the privacy guarantee of Algorithm \ref{alg:dp-iterative-localization}. 
For any neighboring datasets $S=\{ S_1,\ldots, S_K \}$ and $S'=\{ S'_1,\ldots, S'_K \}$ differing by one example, where $S'$ follows the same partition as $S$, and $S_i \cap S_j =  \emptyset$ if $i \ne j$. 
Let $\hat{\wbf}_k = \arg\min_{\wbf} F_k(\wbf; S_k)$ and $\hat{\wbf}'_k = \arg\min_{\wbf} F_k(\wbf; S'_k)$. We first investigate the $\ell_2$-sensitivity of $\hat{\wbf}_k$. Since $F_k$ is $\alpha_k = \frac{2}{\zeta_k n_k}$-strongly convex, by Theorem 6 in \citet{shalev2009stochastic} we have
\[
\|\hat{\wbf}_k - \hat{\wbf}'_k\|_2 \leq \frac{4G}{\alpha_k n_k} = 2G\zeta_k,
\]
where $\bar{\wbf}'_k$ is the return from Line 5 in Algorithm \ref{alg:dp-iterative-localization} based on $F_k(\wbf;S_k')$. By the strong convexity of $F_k$ again, we have with probability at least $1 - \delta$
\[
\frac{\alpha_k}{2}\|\bar{\wbf}_k - \hat{\wbf}_k\|_2^2 \leq F_k(\bar{\wbf}_k; S_k) - F_k(\hat{\wbf}_k; S_k) \leq \frac{G^2\zeta_k}{n_k}
\]
which implies $\|\bar{\wbf}_k - \hat{\wbf}_k\|_2 \leq G\zeta_k$. This further implies $\bar{\wbf}_k$ has $\ell_2$-sensitivity of $4G\zeta_k$ with probability $1 - \delta$. Therefore, by Lemma \ref{lem:high-probability-gaussian-privacy}, each iteration $k$ of Algorithm \ref{alg:dp-iterative-localization} is $(\epsilon, \delta)$-DP. Since the partition of the dataset $S$ is disjoint, and each iteration $k$ of Algorithm \ref{alg:dp-iterative-localization} we only use one subset, thus the whole process will still be $(\epsilon,\delta)$-DP.

Next we investigate the utility bound of Algorithm \ref{alg:dp-iterative-localization}. Firstly, for any fixed $\wbf$, 
\begin{align*}
\ebb[F(\bar{\wbf}_k) - F(\wbf)] = & \ebb[F(\hat{\wbf}_k) - F(\wbf)] + \ebb[F(\bar{\wbf}_k) - F(\hat{\wbf}_k)]\\
\leq & \frac{\ebb[\|\wbf_{k-1} - \wbf\|_2^2]}{\zeta_k n_k} + 3G^2\zeta_k
\end{align*}
where we used Lemma \ref{lem:function-distante} and $\|\bar{\wbf}_k - \hat{\wbf}_k\|_2 \leq G\zeta_k$. Denote $\bar{\wbf}_0 = \wbf^*$ and $\ubf_0 = \wbf_0 - \wbf^*$, we have
\begin{align*}\label{eq:dp-sco-decomposition}
\ebb[F(\wbf_K) - F(\wbf^*)] = & \sum_{k=1}^K \ebb[F(\bar{\wbf}_k) - F(\bar{\wbf}_{k-1})] + \ebb[F(\wbf_K) - F(\bar{\wbf}_K)]\\
\leq & \sum_{k=1}^K \Big(\frac{\ebb[\|\ubf_{k-1}\|_2^2]}{\zeta_k n_k} + 3G^2\zeta_k\Big) + G\ebb[\|\ubf_K\|_2]. \numberthis 
\end{align*}
Recall that by definition $\zeta \leq \frac{D\epsilon}{4G\sqrt{d\log(2.5/\delta)}}$, so that for all $k \geq 0$, there holds 
\[
\ebb[\|\ubf_k\|_2^2] = d\sigma_k^2 = d\Big(\frac{4^{-k}G\zeta}{\epsilon}\Big)^2 \leq 16^{-k}D^2.
\]
Plugging the above estimate into \eqref{eq:dp-sco-decomposition} it follows
\begin{align*}
\ebb[F(\wbf_K) - F(\wbf^*)] \leq & \sum_{k=1}^K 2^{-k}\Big(\frac{8D^2}{\zeta n} + 3G^2\zeta\Big) + 4^{-K}GD\\
\leq & \sum_{k=1}^K 2^{-k}GD\Big(\frac{8}{n}\max\Big\{\sqrt{n}, \frac{\sqrt{d\log(1/\delta)}}{\epsilon}\Big\} + \frac{1}{2\sqrt{n}}\Big) + \frac{GD}{n^2}\\
\leq & 9GD\Big(\frac{1}{\sqrt{n}} + \frac{\sqrt{d\log(1/\delta)}}{n\epsilon}\Big) + \frac{GD}{n^2}.
\end{align*}
This yields the desired utility bound.  

Finally, we investigate the gradient complexity argument. Since $F_k$ is $\frac{2}{\zeta_kn_k}$-strongly convex. We know from Theorem \ref{thm:opt-hp} Part (b), after $T_k \asymp n_k^2\log(1/\delta)$ iterations, we have with probability $1-\delta$
\[
F_k(\bar{\wbf}_k; S_k) - \min_\wbf F_k(\wbf; S_k) = \ocal\Big(\frac{G^2\zeta_kn_k\log(1/\delta)}{n_k^2\log(1/\delta)}\Big) = \ocal\Big(\frac{G^2\zeta_k}{n_k}\Big)
\]
which satisfies the requirement at Line 5 of Algorithm \ref{alg:dp-iterative-localization}. Therefore, in total the gradient complexity is of the form 
$\ocal\Big(\sum_{k=1}^K n_k^2\log(1/\delta)\Big) = \ocal\big(n^2\log(1/\delta)\big).
$ The proof is completed.
\end{proof}

\section{Additional Experiments\label{sec:more-exp}}

In this section, we provide the experimental details and additional experiments to support our theoretical findings. The datasets we used are from LIBSVM website \citep{CC01a}. The statistics of the data is included in Table \ref{tab:datasts}. For data with multiple classes, we convert the first half of class numbers to be the positive class and the second half of class numbers to be the negative class.

\begin{table*}[!ht]
\centering
\small
\caption{Data Statistics. $n$ is the number of samples and $d$ is the number of features.}
\begin{tabular}{c|c|c|c|c|c|c}
\hline
& \code{diabtes} & \code{german} & \code{ijcnn1}  & \code{letter}  & \code{mnist} & \code{usps} \\\hline\hline 
n & 768 & 1,000  & 49,990 & 15,000& 60,000 & 7,291 \\\hline
d & 8 & 24 & 22 & 161 & 780 & 256 \\\hline
\end{tabular}
\label{tab:datasts}
\vspace*{-3mm}
\end{table*}

\begin{table*}[!htb] 
\setlength{\tabcolsep}{2pt}
\centering
\small
\caption{Average AUC score $\pm$ standard deviation across multiple datasets. Our best results are highlighted in bold.}
\begin{tabular}{@{\hskip1pt}c@{\hskip1pt}|c|c|c|c|c|c}
\hline
 Algorithm & \code{diabetes} & \code{german} & \code{ijcnn1} & \code{letter} & \code{mnist} & \code{usps} \\\hline\hline 
\code{Our} & \bm{$.831 \pm .030$} &  $.793 \pm .021$ & \bm{$.934 \pm .002$} & $.810 \pm .007$  & \bm{$.932 \pm .001$} & \bm{$.926 \pm .006$} \\\hline
$\code{SGD}_{pair}$~\citep{lei2020sharper} & $.830 \pm .028$ &  $.794 \pm .023$ & $.934 \pm .003$ & $.811 \pm .008$  & $.932 \pm .001$ & $.925 \pm .006$ \\\hline
\code{OLP} \citep{Kar} & $.825 \pm .028$ & $.787 \pm .028$ & $.916 \pm .003$ & $.808 \pm .010$ & $.927 \pm .003$  & $.917 \pm .006$ \\\hline
$\code{OAM}_{gra}$ \citep{zhao2011online} & $.828 \pm .026$ & $.785 \pm .029$ & $.930 \pm .003$ & $.806 \pm .008$  & $.898 \pm .002$ & $.916 \pm .005$ \\\hline
\code{OLP-RS1} & $.736 \pm .074$ & $.630 \pm .065$ & $.668 \pm .026$ & $.683 \pm .033$ & $.749 \pm .045$ & $.737 \pm .056$ \\\hline
\code{OAM-RS1} & $.737 \pm .069$ & $.640 \pm .058$ & $.677 \pm .014$ & $.675 \pm .050$ & $.685 \pm .042$ & $.691 \pm .059$ \\\hline
\code{SPAUC} \citep{lei2021stochastic} & $.828 \pm .031$ & $.799 \pm .026$ & $.932 \pm .002$ & $.809 \pm .008$ & $.927 \pm .002$ & $.923 \pm .005$ \\\hline
\end{tabular}
\label{tab:gen-hinge-extension}
\vspace*{-2mm}
\end{table*}

For each dataset, we have used $80\%$ of the data for training and the remaining $20\%$ for testing. The results are based on 25 runs of random shuffling. The generalization performance is reported using the average AUC score and standard deviation on the test data. To determine proper hyper parameters, we conduct 5-fold cross
validation on the training sets: 1) for Algorithm \ref{alg:markov} and $\code{SGD}_{pair}$, we select step sizes $\eta_t = \eta \in 10^{[-3:3]}$ and $\wcal$ diameter $D \in 10^{[-3:3]}$; 2) for \code{OLP} we select step sizes $\eta_t = \eta / \sqrt{t}$ where $\eta \in 10^{[-3:3]}$ and $\wcal$ diameter $D \in 10^{[-3:3]}$; 3) for $\code{OAM}_{gra}$ we select learning rate parameter $C \in 10^{[-3:3]}$; 4) for \code{SPAUC} we select  step sizes $\eta_t = \eta / \sqrt{t}$ where $\eta \in 10^{[-3:3]}$. 

Firstly, Table \ref{tab:gen-hinge-extension} summarizes the generalization performance of different algorithms which contains more comparison results than Table \ref{tab:gen-hinge}.  In particular, two additional results are added for comparison, i.e.  \code{OLP-RS1} and  \code{OAM-RS1} which denote the \code{OLP} \citep{Kar} and $\code{OAM}_{gra}$ \citep{zhao2011online} with Reservoir sampling and the buffering set size $s=1$, respectively. We can see that \code{OLP-RS1} and \code{OAM-RS1} are inferior to other algorithms. This inferior performance  for \code{OLP} and \code{OAM} with a small buffering set was also observed in the experiments of \cite{Kar, zhao2011online}. 

\begin{figure}[ht!]
\begin{subfigure}{.33\textwidth}
\centering
\includegraphics[width=.95\linewidth]{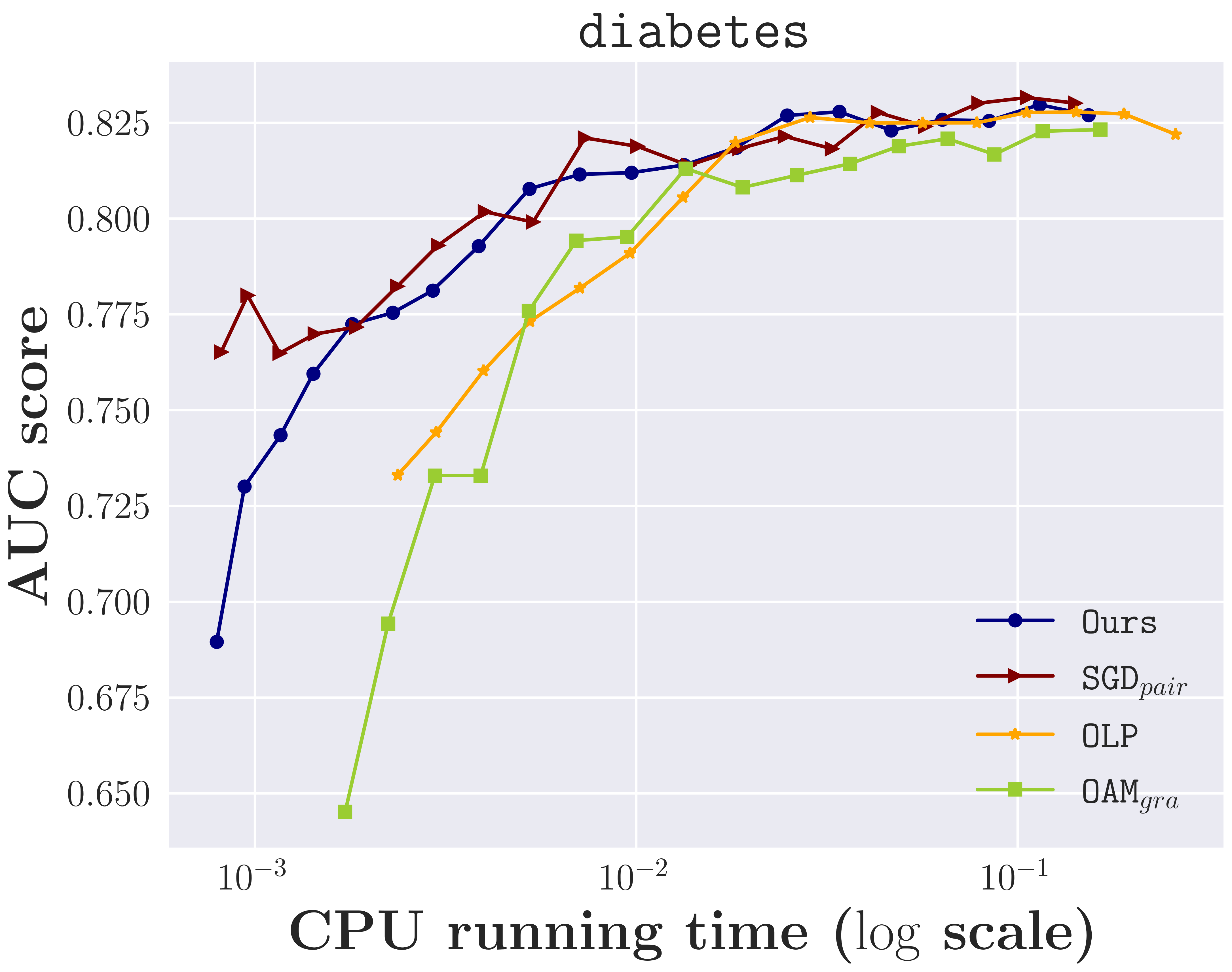}  
\end{subfigure}
\begin{subfigure}{.33\textwidth}
\centering
\includegraphics[width=.95\linewidth]{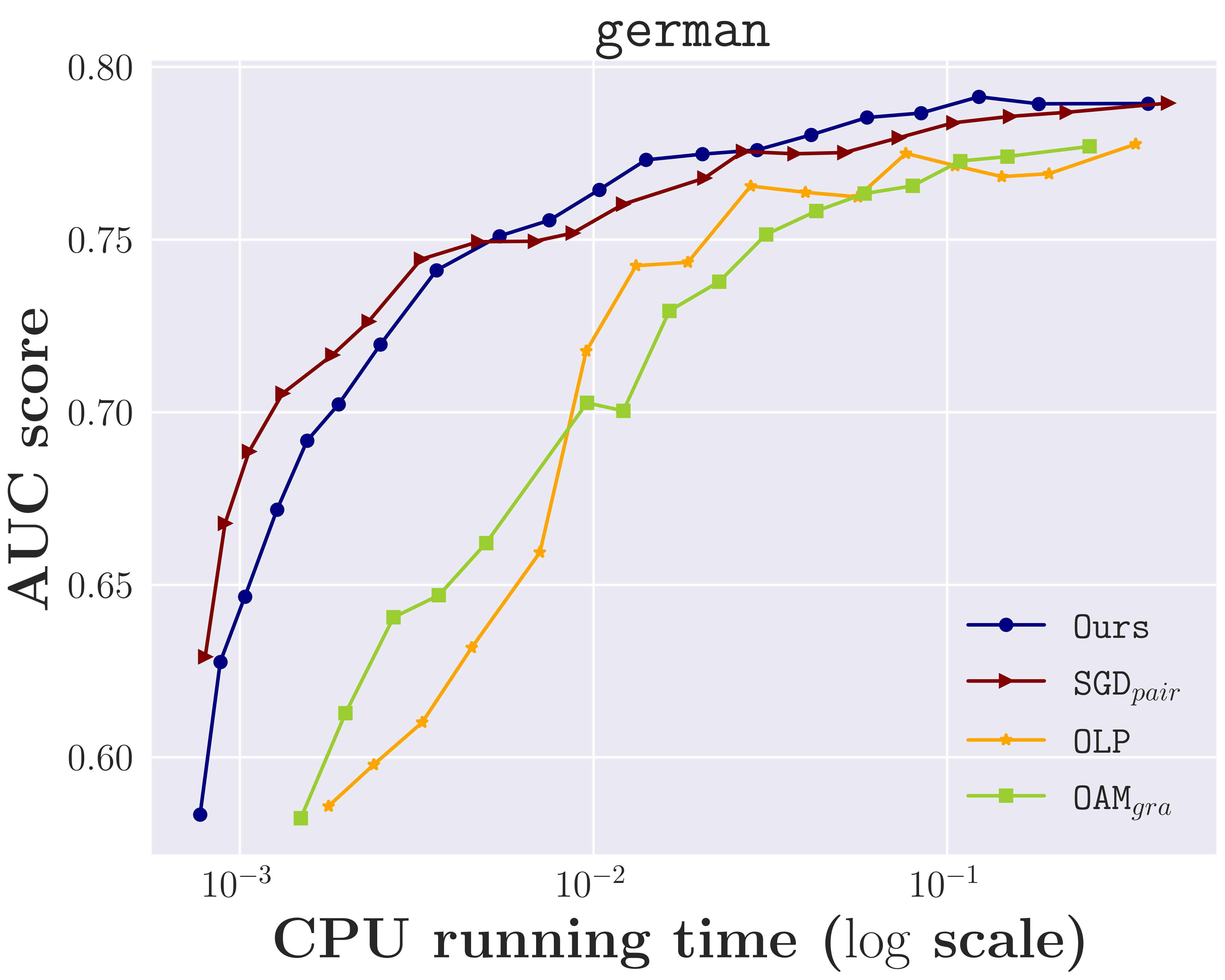}  
\end{subfigure}
\begin{subfigure}{.33\textwidth}
\centering
\includegraphics[width=.95\linewidth]{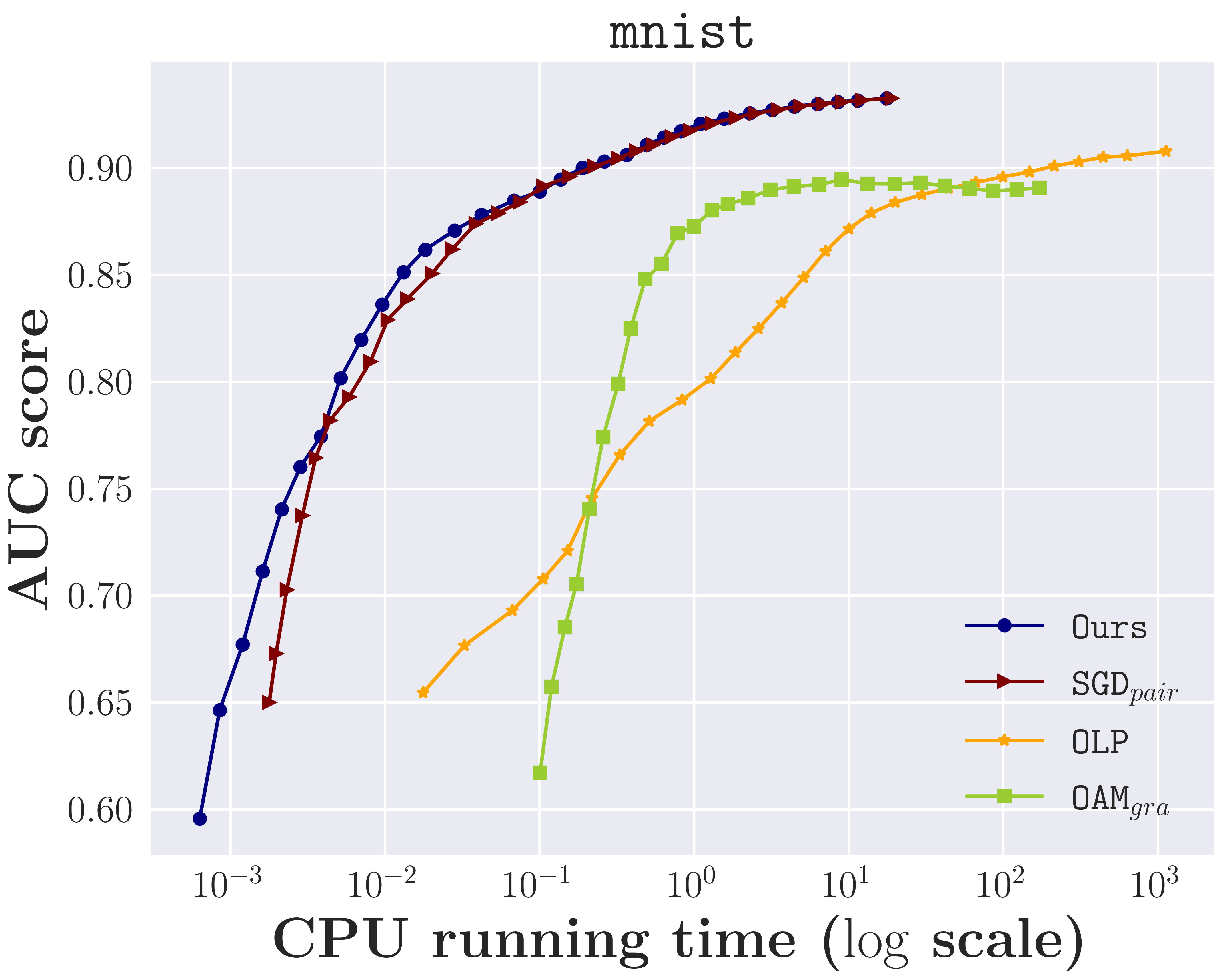}  
\end{subfigure}
\caption{More CPU running time against AUC score for the hinge loss\label{fig:iter-hinge-more}}
\vspace*{-5mm}
\end{figure}

Secondly, we also report more plots on CPU running time against the AUC score on different datasets. Figure \ref{fig:iter-hinge-more} contains more convergence plots for the hinge loss. For a fair comparison of Algorithm \ref{alg:markov} with \code{SPAUC}, the loss function is chosen as the least square loss for Algorithm \ref{alg:markov}, $\code{SGD}_{pair}$ and $\code{OLP}$. The results are shown in Figure \ref{fig:iter-square}.  We can see there that \code{SPAUC} performs very well among most of the datasets. However, this algorithm was designed very specifically for the AUC maximization problem with the least square loss while our algorithm can handle any loss functions and any pairwise learning problems.  We can also observe that 
our algorithm  and $\code{SGD}_{pair}$ converge in a similar CPU running time. In fact, Algorithm \ref{alg:markov} is slightly faster than $\code{SGD}_{pair}$ when the number of samples  gets larger. This is partly  due to different sampling schemes in Algorithm \ref{alg:markov} and  $\code{SGD}_{pair}$. Indeed, at each iteration $\code{SGD}_{pair}$ picks a random pair of examples from  $\binom{n}{2}$ pairs, while  Algorithm \ref{alg:markov} only needs to randomly pick one example from $n$ individual ones. Figure \ref{fig:sampling} depicts the CPU times of these two sampling schemes versus the the number of examples $n$. We can see that, when the sample size $n$ increases, the sampling scheme used in $\code{SGD}_{pair}$ needs significantly more time than our algorithm.

\begin{figure}[!htb]
\begin{subfigure}{.33\textwidth}
\centering
\includegraphics[width=.95\linewidth]{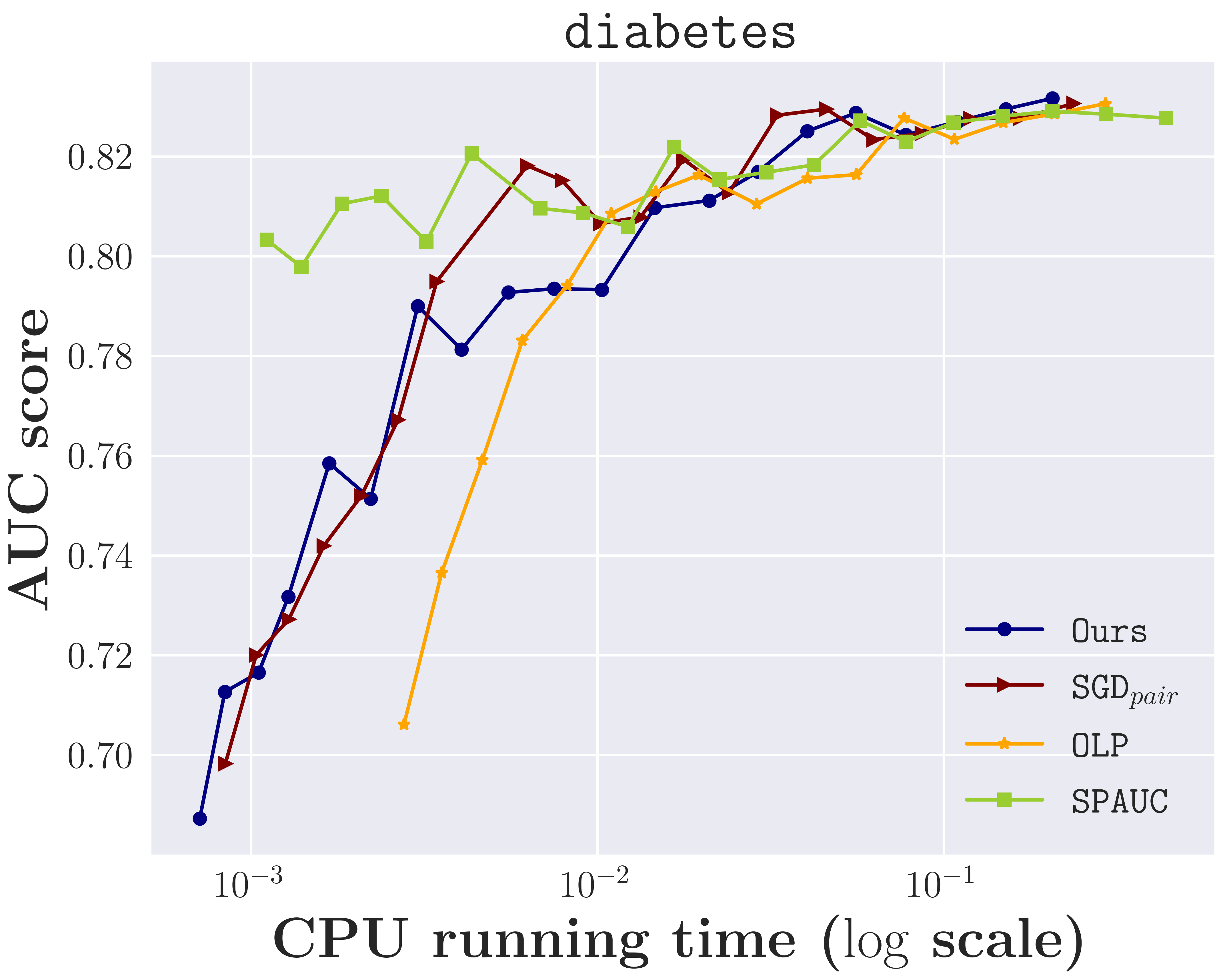}  
\end{subfigure}
\begin{subfigure}{.33\textwidth}
\centering
\includegraphics[width=.95\linewidth]{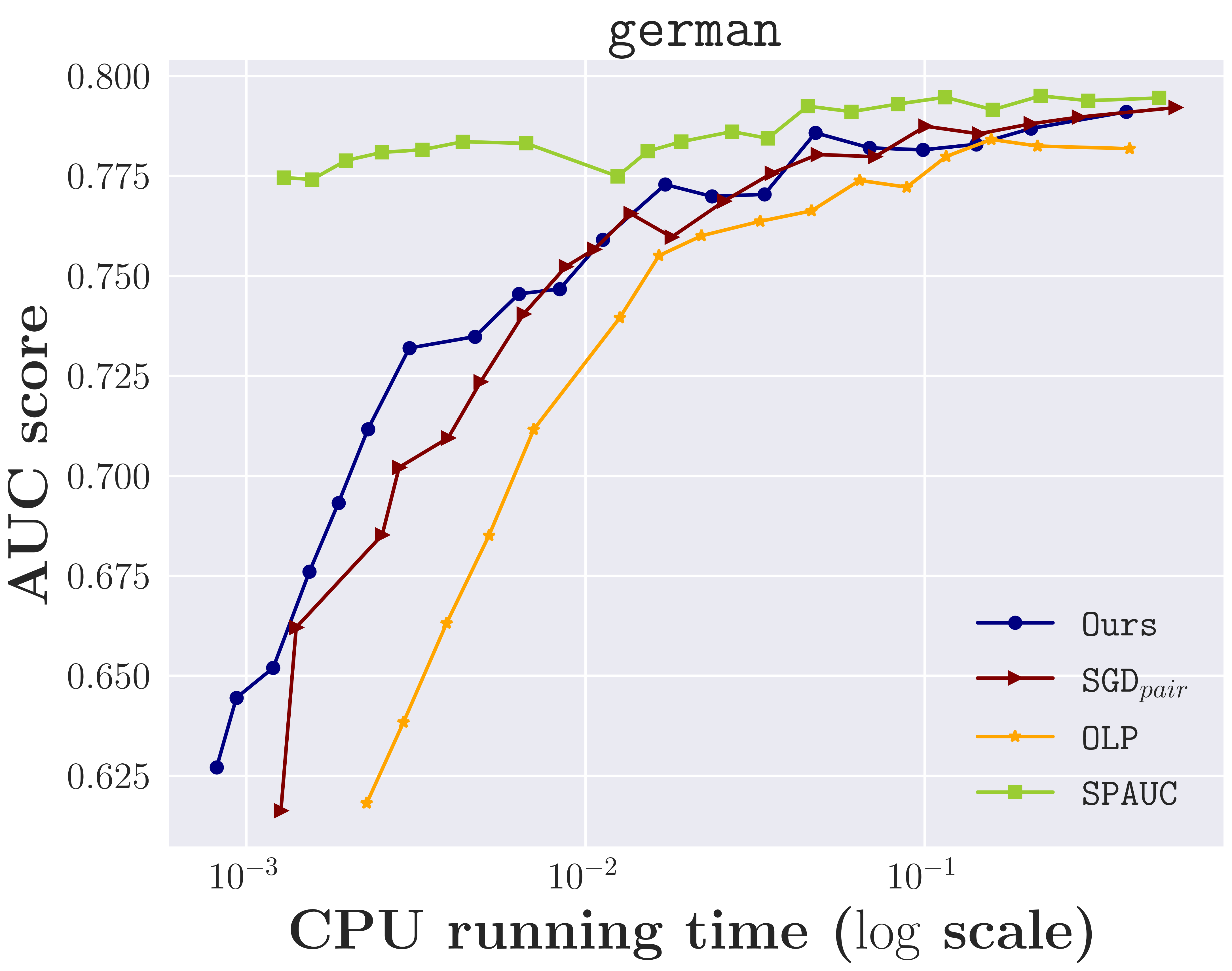}  
\end{subfigure}
\begin{subfigure}{.33\textwidth}
\centering
\includegraphics[width=.95\linewidth]{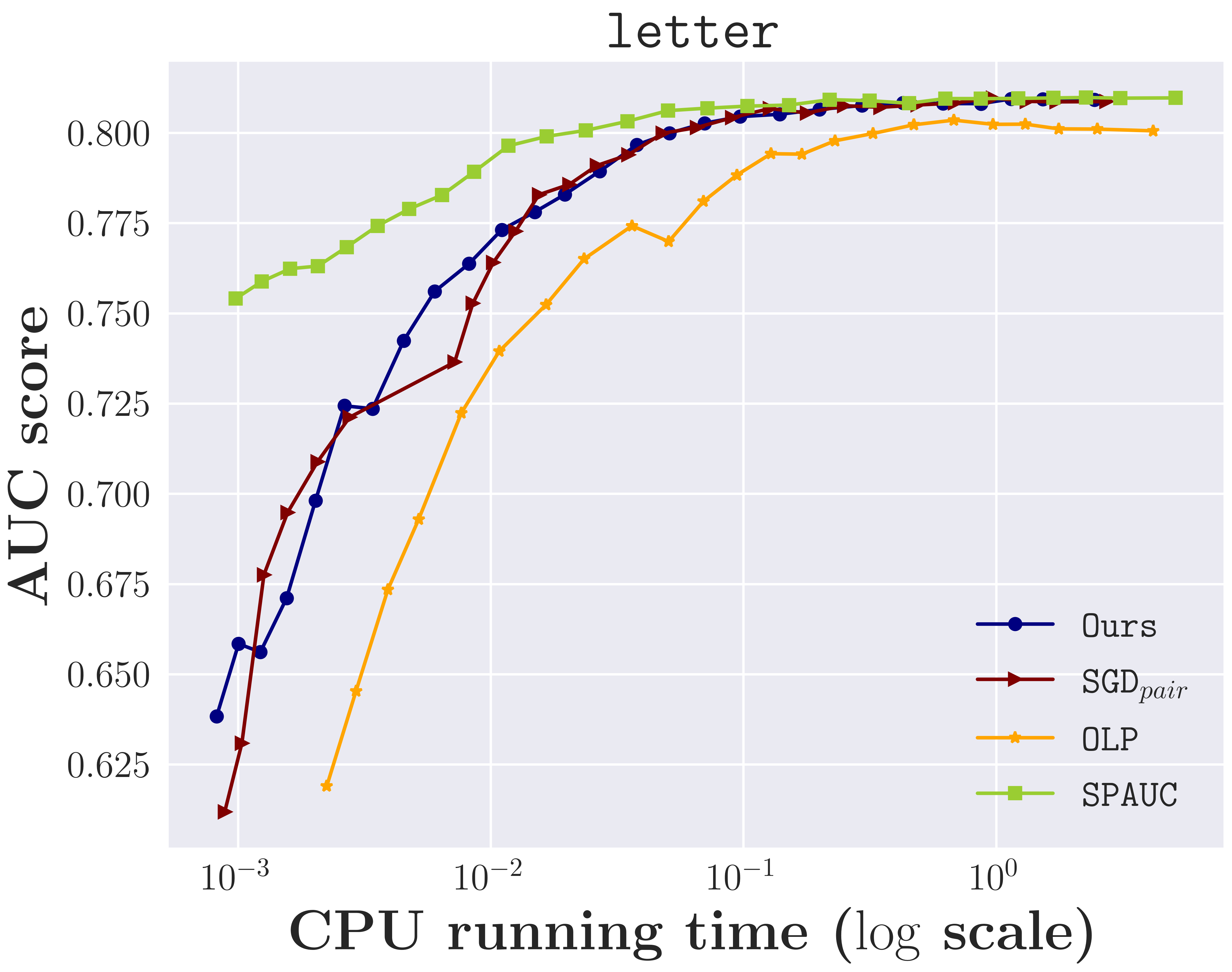}  
\end{subfigure}

\begin{subfigure}{.33\textwidth}
\centering
\includegraphics[width=.95\linewidth]{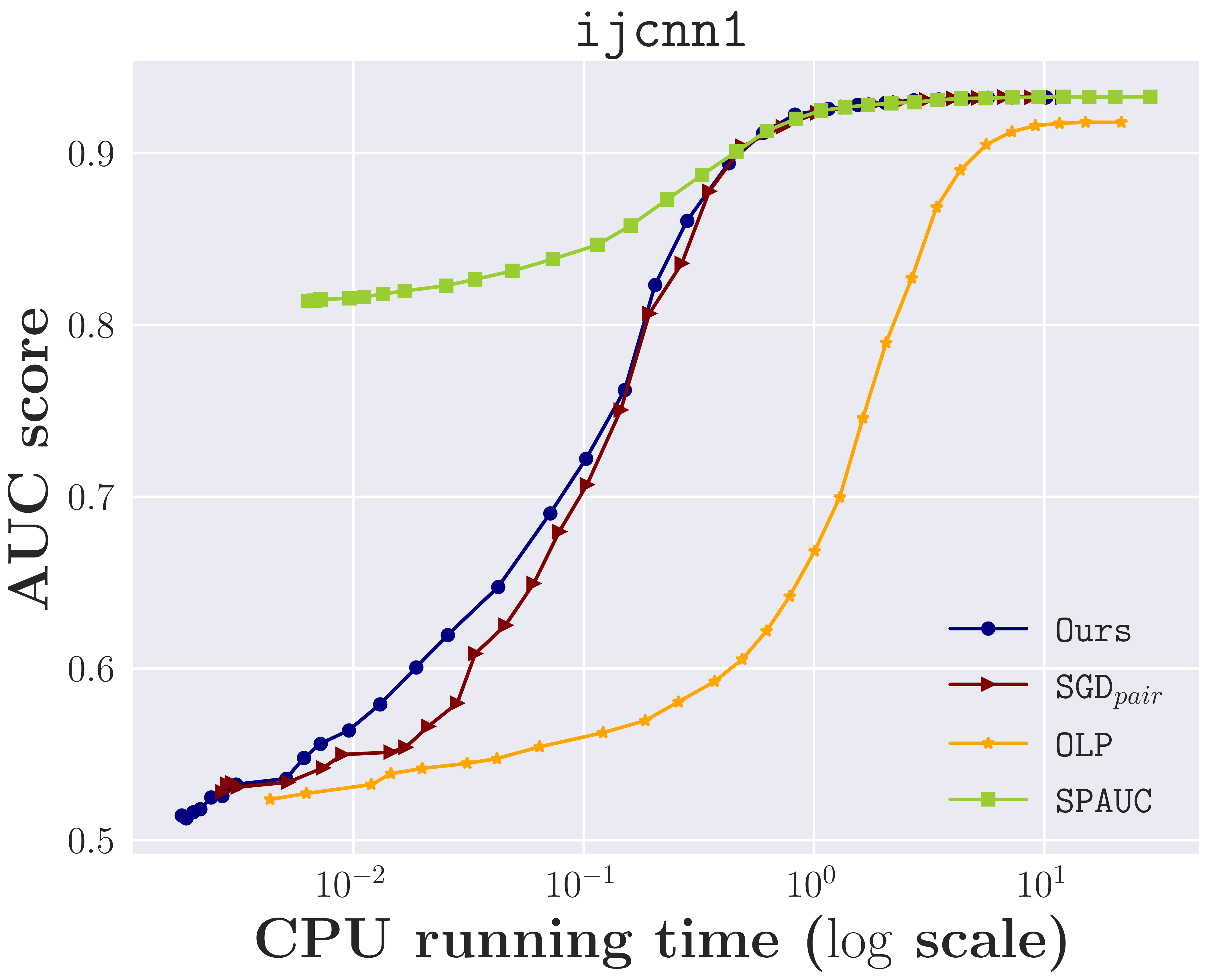}  
\end{subfigure}
\begin{subfigure}{.33\textwidth}
\centering
\includegraphics[width=.95\linewidth]{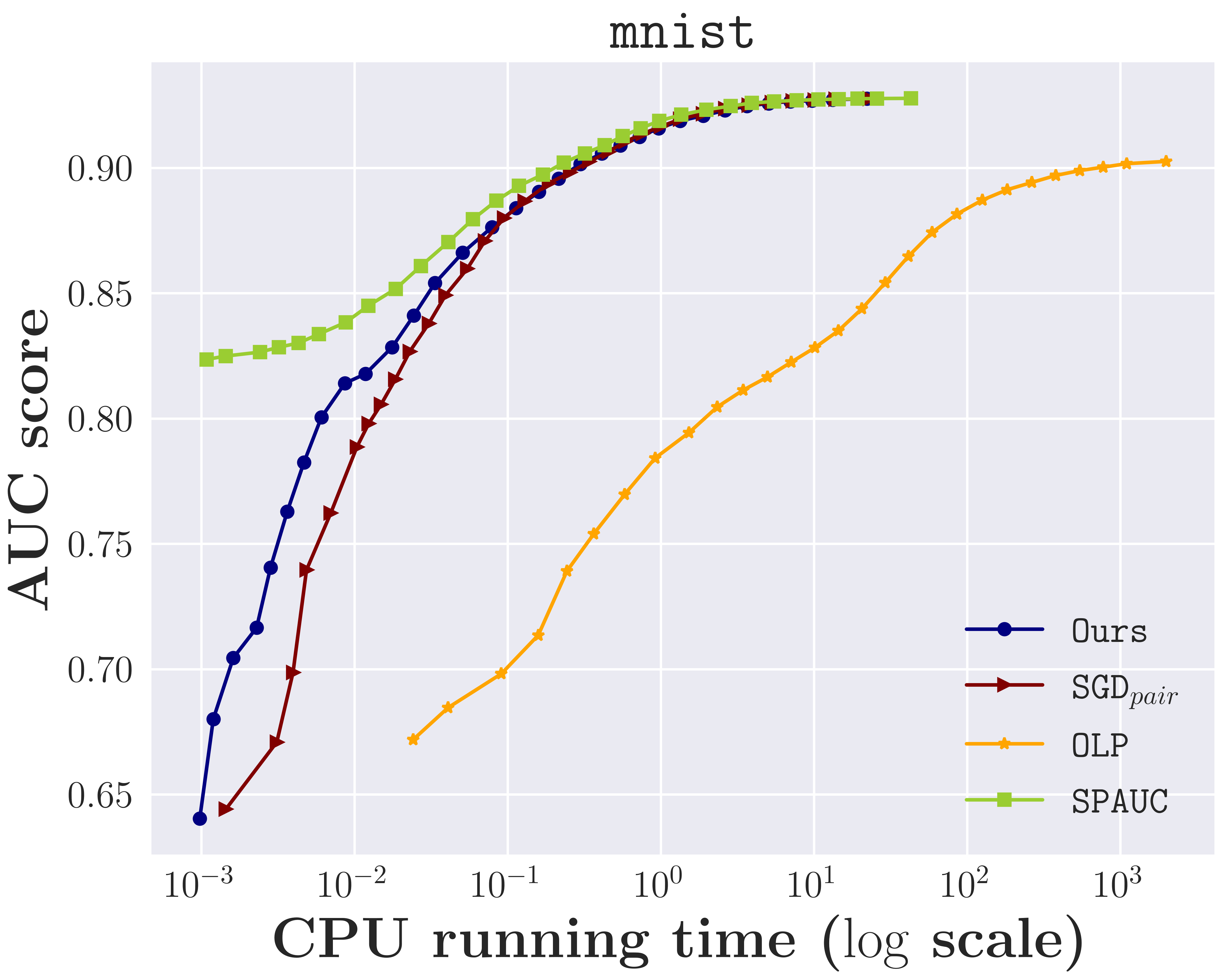}  
\end{subfigure}
\begin{subfigure}{.33\textwidth}
\centering
\includegraphics[width=.95\linewidth]{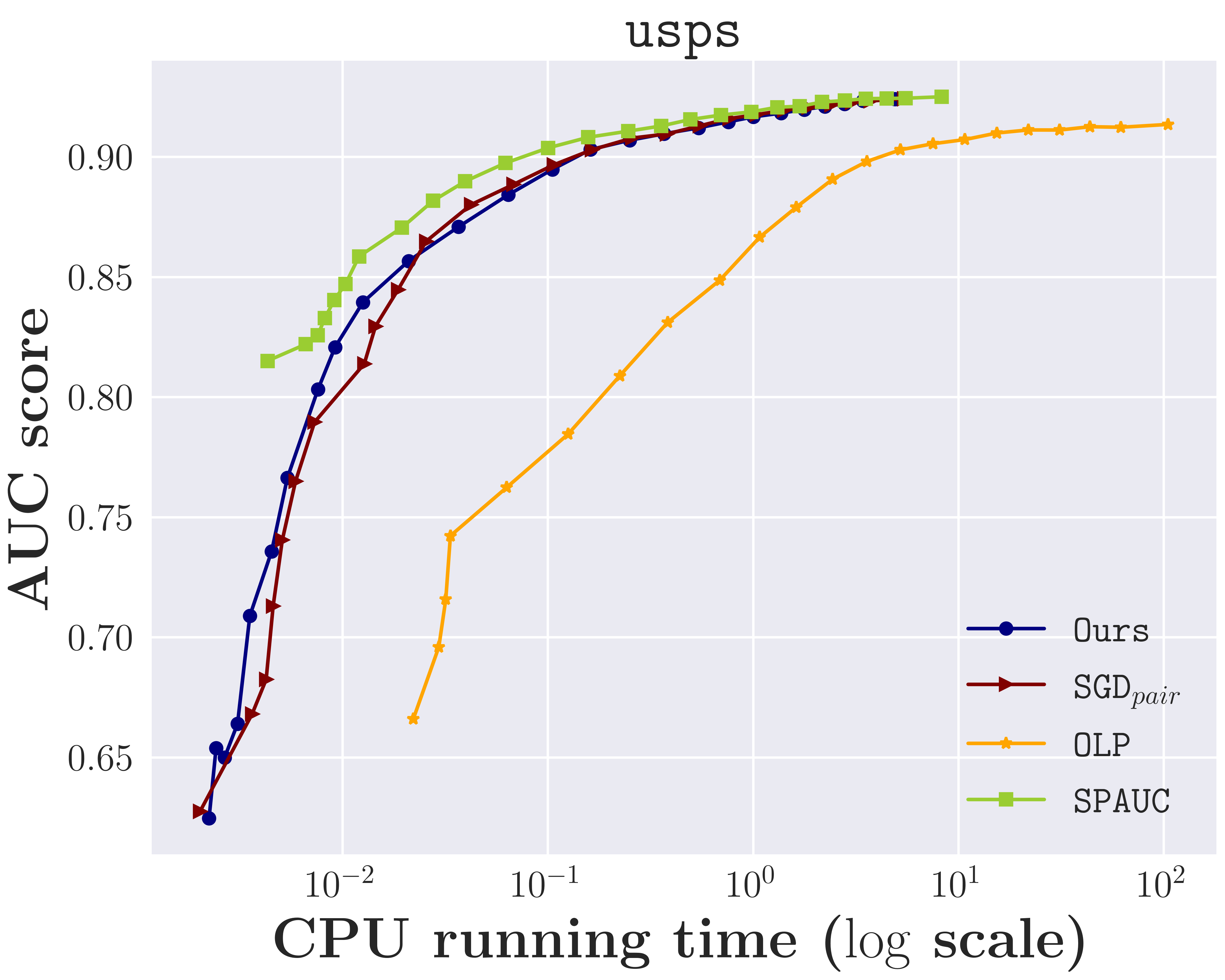}  
\end{subfigure}
\caption{AUC score against CPU running time for the square loss\label{fig:iter-square}}
\vspace*{-5mm}
\end{figure}

\begin{figure}[ht!]
\centering
\includegraphics[width=.33\linewidth]{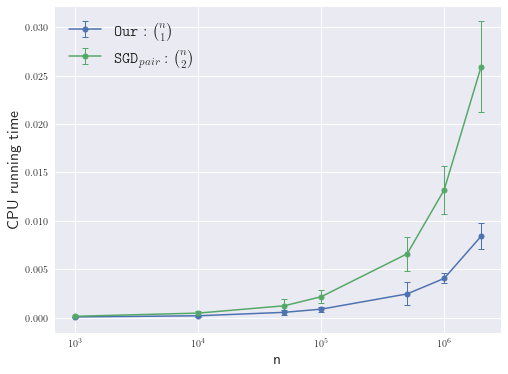} \caption{CPU running time of  different sampling schemes against the sample size $n$ \label{fig:sampling}}
\vspace*{-2mm}
\end{figure}

\begin{figure}[t]
\begin{subfigure}{.33\textwidth}
\centering
\includegraphics[width=.95\linewidth]{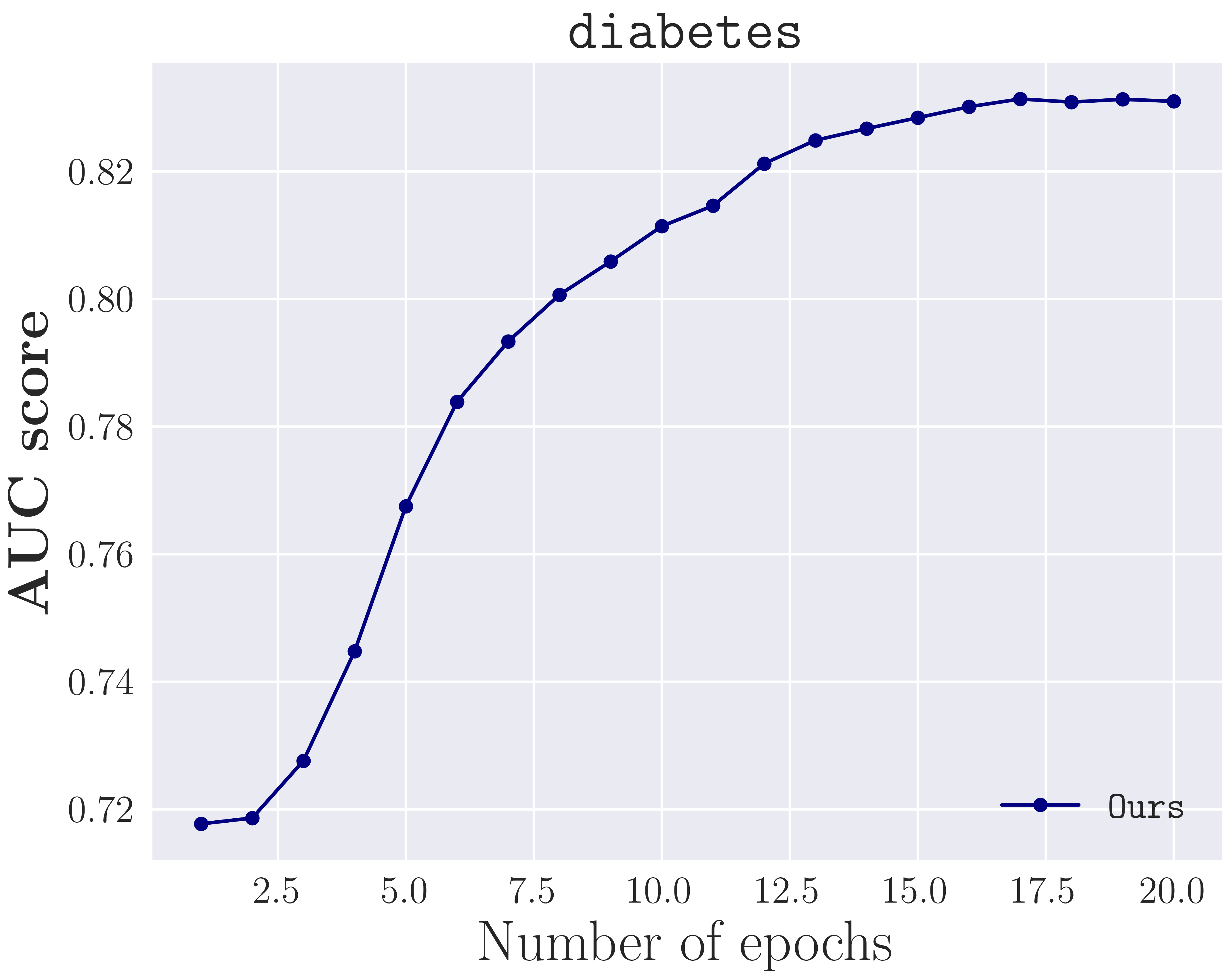}  
\end{subfigure}
\begin{subfigure}{.33\textwidth}
\centering
\includegraphics[width=.95\linewidth]{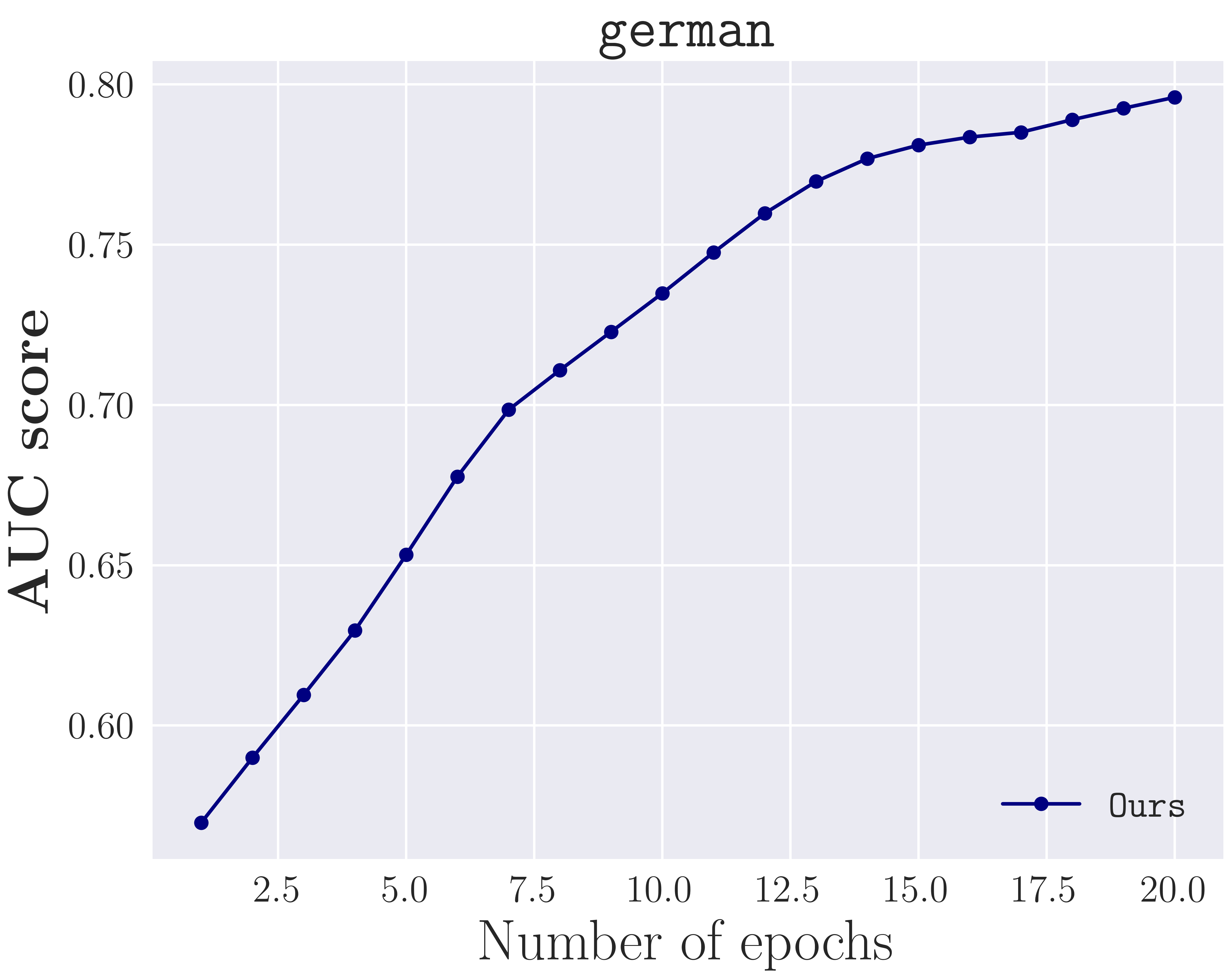}  
\end{subfigure}
\begin{subfigure}{.33\textwidth}
\centering
\includegraphics[width=.95\linewidth]{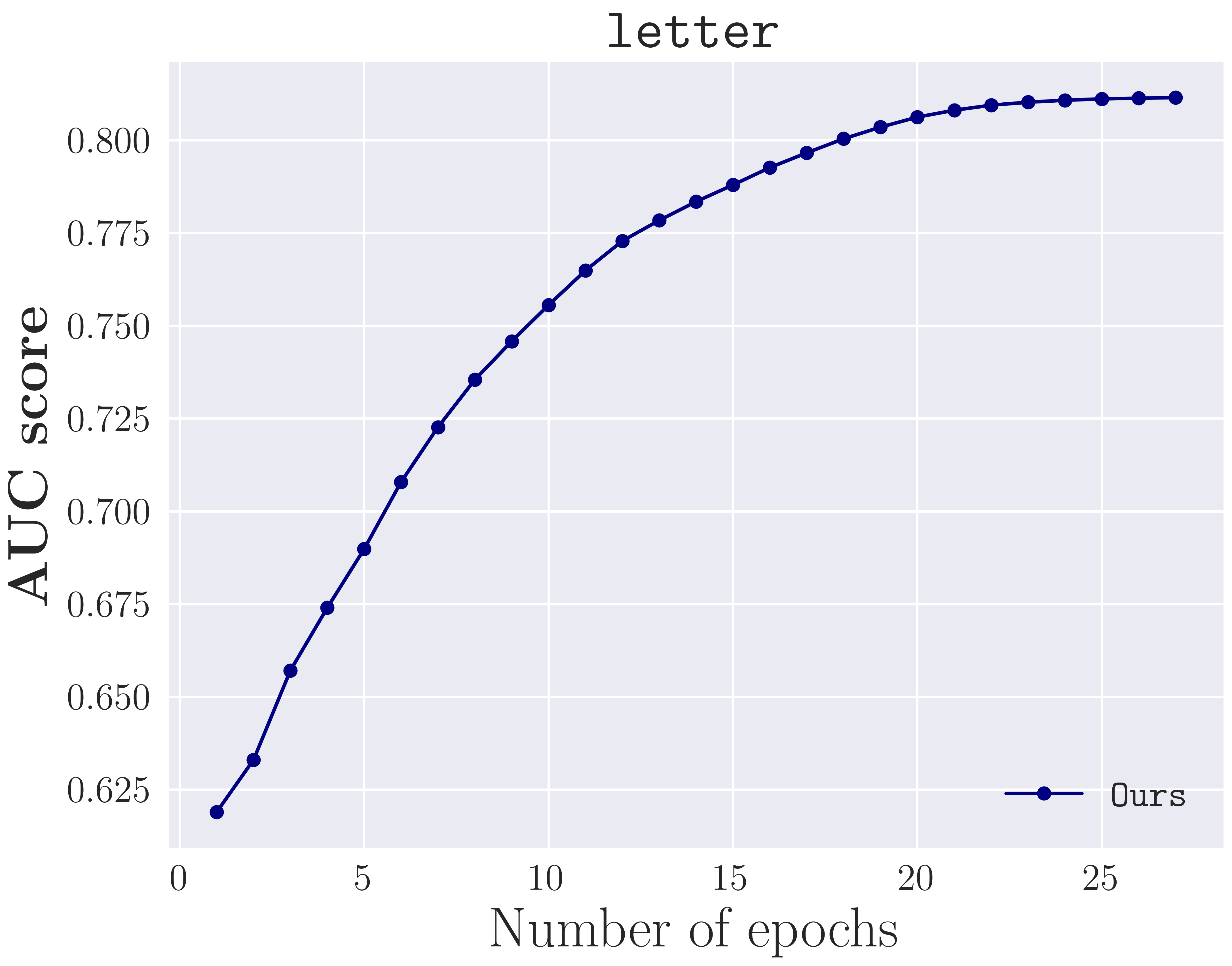}  
\end{subfigure}

\begin{subfigure}{.33\textwidth}
\centering
\includegraphics[width=.95\linewidth]{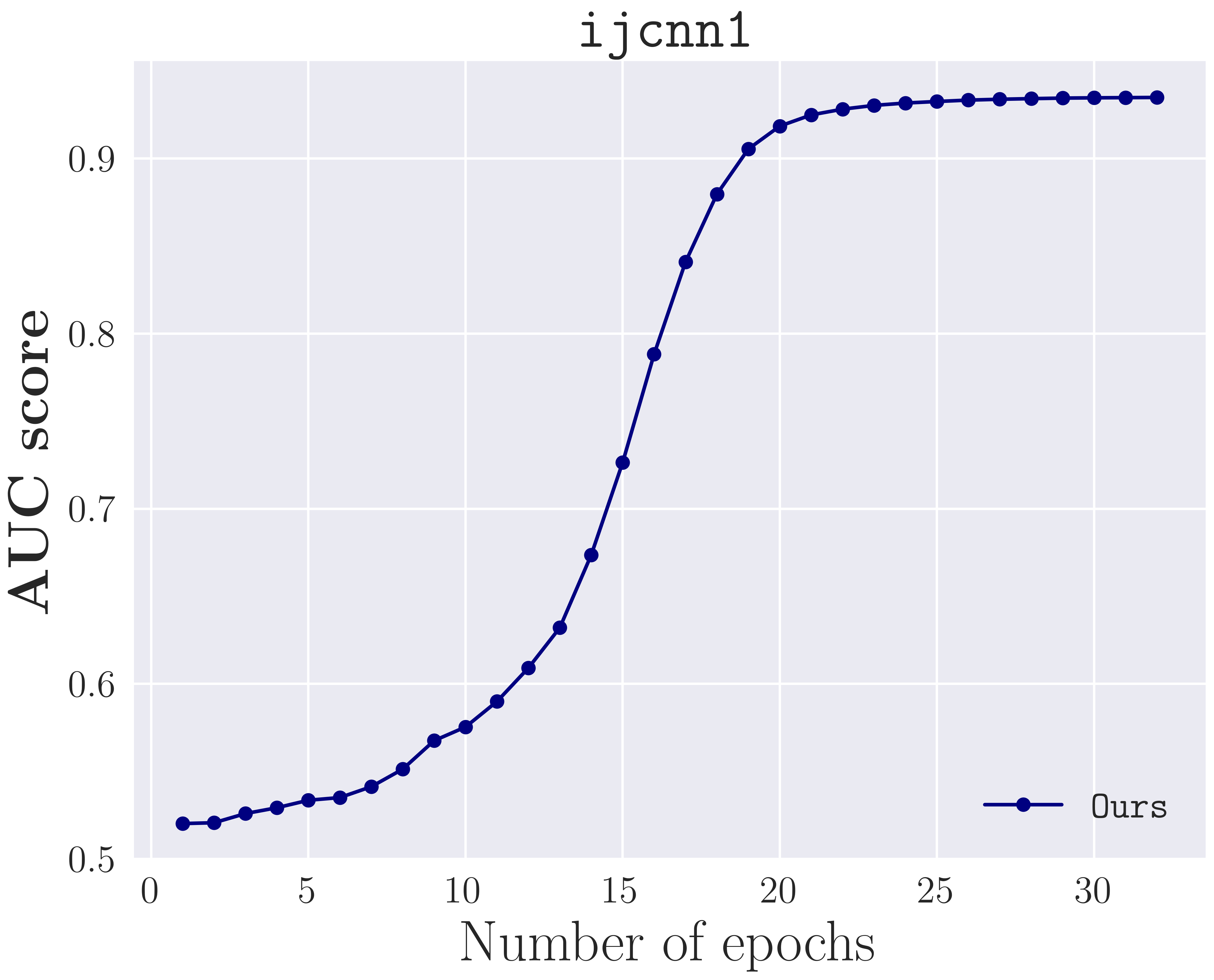}  
\end{subfigure}
\begin{subfigure}{.33\textwidth}
\centering
\includegraphics[width=.95\linewidth]{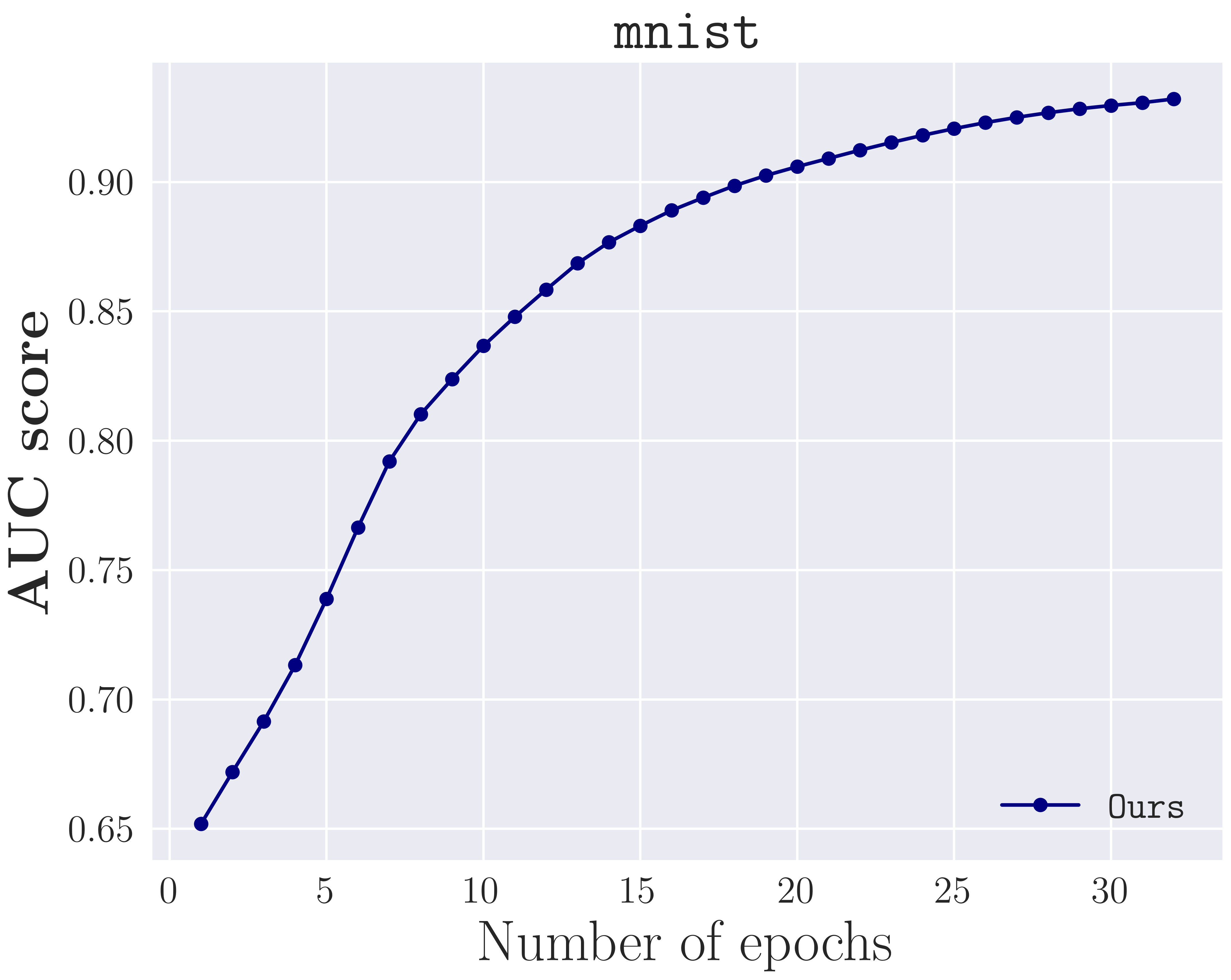}  
\end{subfigure}
\begin{subfigure}{.33\textwidth}
\centering
\includegraphics[width=.95\linewidth]{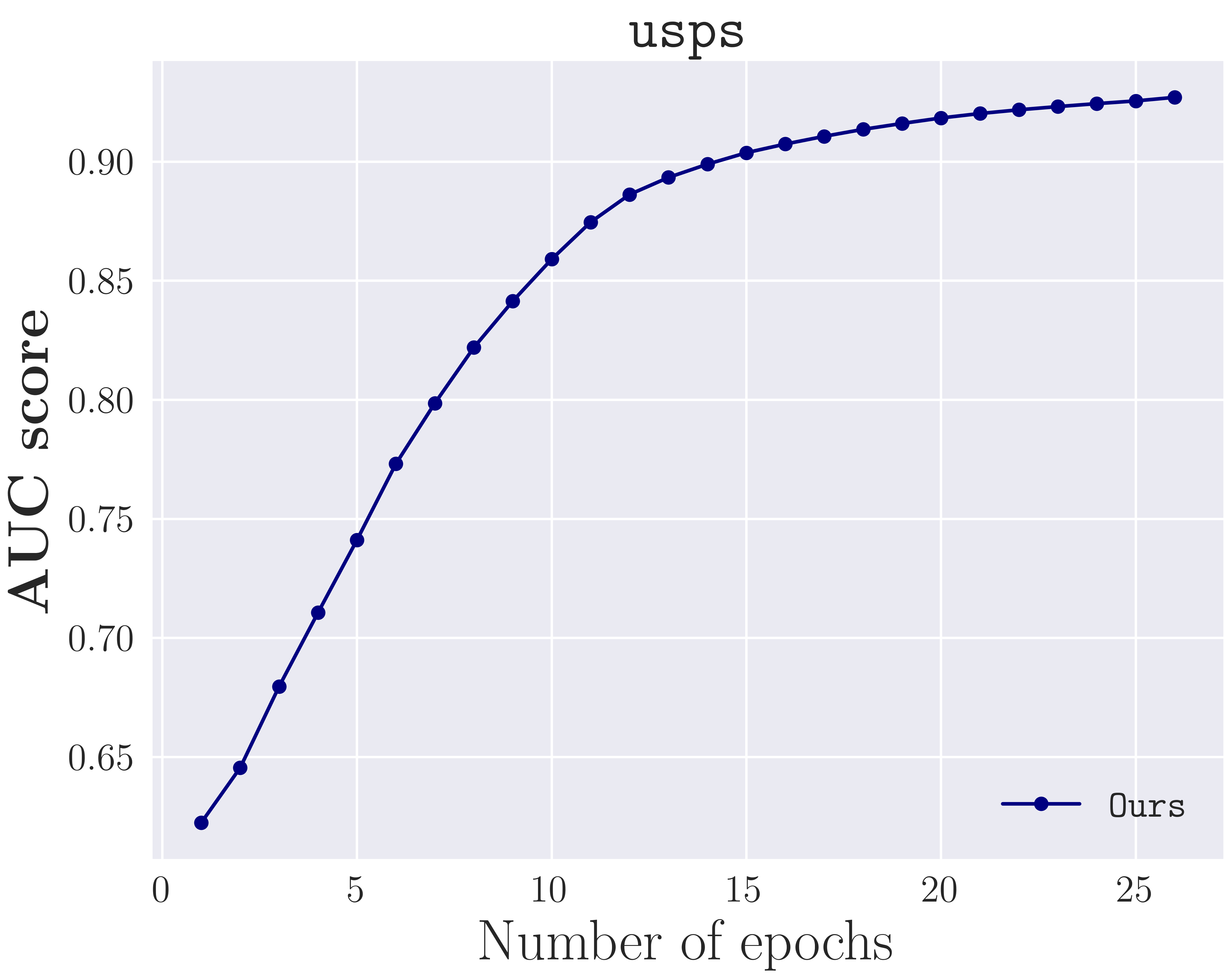}  
\end{subfigure}
\caption{Convergence of Algorithm \ref{alg:markov} for the generalized linear model\label{fig:iter-loglink}}
\vspace*{-5mm}
\end{figure}

Next, we  investigate our Algorithm \ref{alg:markov} in the non-convex setting. To this end,  we use the logistic link function $\text{logit}(t) = (1+\exp(-t))^{-1}$ and then the square loss surrogate function $\ell(t) = (1-t)^2.$   That is, the loss function for the problem of AUC maximization becomes $f(\wbf; (\xbf, y), (\xbf', y')) = (1 - \text{logit}(\wbf^\top(\xbf - \xbf'))^2\ibb_{[y=1\wedge y'=-1]}$. Although $f$ is non-convex, it was shown that it satisfies the PL condition \citep{foster2018uniform}. The results are reported in Figure \ref{fig:iter-loglink} which shows that Algorithm \ref{alg:markov} also converges very quickly in  this non-convex setting. 

Finally, we compare our differentially private algorithm for AUC maximization (i.e.  Algorithm \ref{alg:dp-iterative-localization-2}) with the  logistic loss  $\ell(t) = \log(1+\exp(-t))$  against the state-of-art algorithm \code{DPEGD} \citep{xue2021differentially}.    \code{DPEGD} used gradient descent and the localization technique to guarantee privacy.  Algorithm \ref{alg:markov} was used as non-private baseline, i.e. $\epsilon = 0$. Here,  $\delta = \frac{1}{n}$ as suggested in the previous work \citep{xue2021differentially}.  We consider the effect of different privacy budget $\epsilon$'s against the generalization ability. The implementation across all algorithms is based on fixed training size $256$. Average AUC scores over $25$ times repeated experiments are listed in Table \ref{tab:dp-logistic} and \ref{tab:dp-logistic-more} for the datasets of \texttt{diabetes} and \texttt{german}, respectively.  These results demonstrate Algorithm \ref{alg:dp-iterative-localization-2} achieves competitive  performance with \code{DPEGD} using full gradient descent. 


\begin{table*}[!ht]
\centering
\small
\caption{Average AUC $\pm$ standard deviation on \code{diabetes}. \code{Non-Private} result is $.813 \pm .016$.}
 \begin{tabular}{@{\hskip1pt}c@{\hskip1pt}|c|c|c|c|c|c}
\hline
 Algorithm & $\epsilon = 0.2$ & $\epsilon = 0.5$ & $\epsilon = 0.8$ & $\epsilon = 1.0$ & $\epsilon = 1.5$ & $\epsilon = 2.0$ \\\hline\hline 
\code{Our} & $.690 \pm .094$ & $.751 \pm .028$ &  $.771 \pm .016$ & $.783 \pm .024$ & $.784 \pm .018$ & $.789 \pm .018$  \\\hline
\code{DPEGD} \citep{xue2021differentially} & $.624 \pm .109$ &  $.727 \pm .055$ & $.768 \pm .027$ & $.796 \pm .011$ &  $.797 \pm .017$ &  $.792 \pm .016$ \\\hline
\end{tabular}
\label{tab:dp-logistic}
\vspace*{-3mm}
\end{table*}

\begin{table*}[!ht]
\centering
\small
\caption{Average AUC $\pm$ standard deviation on \code{german}. \code{Non-Private} result is $.763 \pm .016$.}
 \begin{tabular}{@{\hskip1pt}c@{\hskip1pt}|c|c|c|c|c|c}
\hline
 Algorithm & $\epsilon = 0.2$ & $\epsilon = 0.5$ & $\epsilon = 0.8$ & $\epsilon = 1.0$ & $\epsilon = 1.5$ & $\epsilon = 2.0$ \\\hline\hline 
\code{Our} &  $.614 \pm .035$ &  $.672 \pm .064$ & $.721 \pm .024$ &  $.725 \pm .032$ & $.747 \pm .019$ & $.749 \pm .021$ \\\hline
\code{DPEGD} \citep{xue2021differentially} & $.598 \pm .018$& $.703 \pm .039$ & $.723 \pm .029$ &  $.742 \pm .028$ &  $.753 \pm .017$ &  $.757 \pm .018$ \\\hline
\end{tabular}
\label{tab:dp-logistic-more}
\vspace*{-3mm}
\end{table*}

We also report the CPU running times of Algorithm \ref{alg:dp-iterative-localization-2} and $\code{DPEGD}.$ In this setting, we fix the privacy budget $\epsilon = 1$ and vary the training size $n$. The results are reported in Table \ref{tab:dp-logistic-time}. These results shows that Algorithm \ref{alg:dp-iterative-localization-2} can arrive competitive performance with $\code{DPEGD}$ with less CPU running time.

\begin{table*}[!ht]
\centering
\setlength{\tabcolsep}{4pt}
\small
\caption{Average AUC score and average CPU running time $\pm$ standard deviation.}
\label{tab:dp-logistic-time}
 \begin{tabular}{@{\hskip1pt}c@{\hskip1pt}|@{\hskip1pt}c@{\hskip1pt}|c|c|c|c|c|c}
\hline
\multirow{2}{*}{Algorithm} & & \multicolumn{3}{c|}{\code{diabetes}} & \multicolumn{3}{c}{\code{german}} \\\cline{3-8}
& & $n=100$ & $n=200$ & $n=300$ & $n=100$ & $n=200$ & $n=300$ \\\hline\hline
\multirow{2}{*}{\code{Our}} & AUC & $.709 \pm .051$ & $.788 \pm .019$ &  $.790 \pm .021$ &  $.681 \pm .029$ & $.692 \pm .032$ & $.734 \pm .022$  \\\cline{2-8}
& Time & $.046 \pm .010$ & $.096 \pm .019$ & $.135 \pm .026$ & $.377 \pm .097$ & $.637 \pm .136$ & $.767 \pm .151$  \\\hline
\multirow{2}{*}{\code{DPEGD} \citep{xue2021differentially}}  & AUC & $.705 \pm .070$ & $.772 \pm .017$ &  $.777 \pm .023$ &  $.687 \pm .033$ &  $.700 \pm .038$ &  $.755 \pm .019$ \\\cline{2-8}
& Time & $.421 \pm .067$ & $.885 \pm .158$ & $1.41 \pm .248$ & $1.00 \pm .185$ & $1.88 \pm .273$ & $2.57 \pm .421$  \\\hline
\end{tabular}
\end{table*}

\end{document}